\documentclass[year=26,pdfa]{fmcad}

\usepackage{graphicx} 

\title{Property-driven Causal Abstractions\\ for Markov Decision Processes}
\author{Jule Schmidt$^1$\orcid{0009-0004-2109-9975}, Maximilian Weininger$^1$\orcid{0000-0002-0163-2152}, Clemens Dubslaff$^2$\orcid{0000-0001-5718-8276}, David Parker$^3$\orcid{0000-0003-4137-8862}, Nils Jansen$^{1,4}$\orcid{0000-0003-1318-8973} \\
\{jule.schmidt, maximilian.weininger, n.jansen\}@rub.de, c.dubslaff@tue.nl, david.parker@cs.ox.ac.uk \\
\footnotesize{
$^1$Ruhr-University Bochum, Germany, $^2$Eindhoven University of Technology, The Netherlands,}
\\
\footnotesize{$^3$Oxford University, UK, $^4$Radboud University, Nijmegen, The Netherlands}}

\newcommand{\Nat}{\ensuremath\mathbb{N}\xspace}

\newcommand{\mdp}{\mathsf{M}}         
\newcommand{\imdp}{\mathcal{M}} 
\newcommand{\sg}{\mathcal{G}} 

\newcommand{\absmdp}{\hat{\mdp}} 
\newcommand{\absimdp}{\hat{\imdp}} 
\newcommand{\abssg}{\hat{\sg}} 

\newcommand{\states}{\mathcal{S}}      
\newcommand{\absstates}{\hat{\states}} 
\newcommand{\istates}{\states}
\newcommand{\absistates}{\absstates}
\newcommand{\state}{\mathit{s}}        
\newcommand{\absstate}{\hat{\state}}   
\newcommand{\action}{\mathit{a}}
\newcommand{\q}{q} 
\newcommand{\absq}{\hat{\q}} 
\newcommand{\statevar}{\mathit{X}} 
\newcommand{\statevarval}{\mathit{x}}  

\newcommand{\Max}{\vartriangle}
\newcommand{\Min}{\triangledown}

\newcommand{\statesMax}{\states_{\Max}}
\newcommand{\statesMin}{\states_{\Min}}
\newcommand{\absstatesMax}{\absstates_{\Max}}
\newcommand{\absstatesMin}{\absstates_{\Min}}

\newcommand{\actions}{\mathcal{A}}   
\newcommand{\absactions}{\hat {\actions}}

\newcommand{\transition}{\mathcal{P}}  
\newcommand{\abstransition}{\hat{\transition}} 
\newcommand{\uncertaintyset}{\mathcal{U}}
\newcommand{\itransitionlb}{\transition_{\mathit{L}}}
\newcommand{\absitransitionlb}{\abstransition_{\mathit{L}}}
\newcommand{\itransitionub}{\transition_{\mathit{U}}}
\newcommand{\absitransitionub}{\abstransition_{\mathit{U}}}

\newcommand{\distribution}{\mathit{D}} 

\newcommand{\fullmdp}{\mdp = (\states, \actions, \transition)}
\newcommand{\fullimdp}{\imdp = (\states, \actions, \itransitionlb, \itransitionub)}
\newcommand{\fullSG}{\sg = (\states,\statesMax,\statesMin, \actions, \transition)}

\newcommand{\Paths}{\mathrm{Paths}}

\newcommand{\Path}{\rho}

\newcommand{\targetset}{\textsf{T}}

\newcommand{\abstargetset}{\hat{\targetset}}

\newcommand{\reachprob}{\mathbb{P}}

\newcommand{\policy}{\pi}
\newcommand{\policies}{\Pi}

\newcommand{\intervalpolicy}{\policy_2}

\newcommand{\abspolicy}{\hat{\policy}}

\newcommand{\semantics}[1]{\llbracket#1\rrbracket}

\newcommand{\opt}{\mathrm{opt}}
\newcommand{\nopt}{\overline{\opt}}

\newcommand{\boolvar}{\ensuremath{F}\xspace}
\newcommand{\Boolvar}{\ensuremath{\mathsf{F}}\xspace}
\newcommand{\boolvarval}{\ensuremath{f}\xspace}
\newcommand{\partint}{\delta} 
\newcommand{\partintset}{\Delta} 
\newcommand{\totint}{\theta} 
\newcommand{\totintset}{\Theta} 

\newcommand{\Effect}{\mathsf{Effect}}
\newcommand{\Reach}{\mathsf{Valid}}

\newcommand{\Causes}{\mathsf{Causes}}
\newcommand{\featcause}{\gamma} 

\newcommand{\partition}{\mathit{P}}

\newcommand{\val}{\mathsf{V}}

\newcommand{\wval}{\mathsf{W}} 

\newcommand{\bval}{\mathsf{B}} 

\newcommand{\abs}[1]{\lvert #1 \rvert}

\newcommand{\eqdef}{\vcentcolon=}

\usepackage{xargs}
\usepackage[textwidth=15mm,
	backgroundcolor=white,
    tickmarkheight=2pt,
	textsize=tiny]{todonotes}
\usepackage{silence}
\newcommandx{\todomacro}[4][4=]{%
    \todo[bordercolor=#2, linecolor=#2, , backgroundcolor=#2!4, #4]{\textcolor{#2}{\textbf{#1:} #3}}%
}
\newcommandx{\js}[2][2=]{\todomacro{JS}{violet}{#1}[#2]}
\newcommandx{\mw}[2][2=]{\todomacro{MW}{blue}{#1}[#2]}
\newcommandx{\nj}[2][2=]{\todomacro{NJ}{red}{#1}[#2]}
\newcommandx{\cd}[2][2=]{\todomacro{CD}{cyan}{#1}[#2]}
\newcommandx{\dave}[2][2=]{\todomacro{DP}{teal}{#1}[#2]}

\newcommand{\para}[1]{\smallskip\noindent\emph{#1}}

\newcommand{\predpart}{\mathsf{PredPart}}

\newcommand{\cgpart}{\mathsf{CGPart}}

\newcommand{\xpos}{\mathit{x}}
\newcommand{\ypos}{\mathit{y}}
\newcommand{\passenger}{\mathit{passenger}}
\newcommand{\battery}{\mathit{battery}}
\usepackage{csquotes}
\usepackage{amsmath}
\usepackage{amsfonts}
\usepackage{amssymb,amsthm}
\usepackage{mathtools}
\usepackage{placeins}

\theoremstyle{plain}
\newtheorem{theorem}{Theorem}
\newtheorem{lemma}{Lemma}

\theoremstyle{definition}
\newtheorem{definition}{Definition}

\theoremstyle{plain}
\newtheorem{example}{Example}
\theoremstyle{remark}
\newtheorem{remark}{Remark}

\usepackage{thm-restate}

\usepackage{algpseudocode}
\usepackage{algorithm}

\usepackage{stmaryrd}

\usepackage{xcolor}
\usepackage{enumitem}

\usepackage{hyperref}
\usepackage[capitalize]{cleveref}
\usepackage{wrapfig}

\usepackage{tikz}
\usepackage{pgfplots}

\usetikzlibrary{shapes,positioning,arrows,arrows.meta,calc,automata,matrix,fit,backgrounds,fadings,colorbrewer,plotmarks}
\usepackage{tikzsymbols}

\pgfplotsset{compat=1.17}

\tikzstyle{loopaction}=[font=\small,outer sep=2pt]
\tikzstyle{actionnode}=[circle,draw=black,fill=black,minimum size=1mm,inner sep=0,outer sep=0]
\tikzstyle{actionedge}=[draw,-]
\tikzstyle{prob}=[font=\scriptsize,outer sep=1pt]
\tikzstyle{probedge}=[draw,->]
\tikzstyle{directedge}=[draw,->]

\tikzset{chainarrow/.tip={Stealth[length=3pt]}}
\tikzset{>=chainarrow}

\tikzset{
	action/.style={font=\small,outer sep=2pt,inner sep=0pt},
	state/.style={thick,align=center,minimum width=0.8cm,minimum height=0.8cm,inner sep=0.2em,rectangle,rounded corners,black,draw},
	triangle/.style={regular polygon, regular polygon sides=3},
	max vertex/.style={state,path picture={\draw[gray,rounded corners=0,fill,fill opacity=0.2,thick] (path picture bounding box.south west) -- (path picture bounding box.north) -- (path picture bounding box.south east);}},
	min vertex/.style={state,path picture={\draw[gray,rounded corners=0,fill,fill opacity=0.2,thick] (path picture bounding box.north west) -- (path picture bounding box.south) -- (path picture bounding box.north east);}},
	staterewardbox/.style={draw,fill=white,rounded corners=1,inner sep=2pt}
}

\usepackage{csvsimple}
\usepackage{booktabs}
\usepackage{subcaption}

\usepackage{fontawesome5}

\usepackage[table]{xcolor}

\usepackage{etoolbox} 
\newtoggle{arxiv}
\newcommand{\ifarxivelse}[2]{\iftoggle{arxiv}{#1}{#2}}
\toggletrue{arxiv} 

\renewcommand{\textcolor}[2]{#2}
\begin{document}
\pagestyle{plain}
\maketitle


\begin{abstract}
Markov Decision Processes (MDPs) are widely used as decision-making models, commonly specified over factored state spaces through state variables and their valuations.
The exponential blowup in the number of states renders many reasoning tasks in MDPs challenging.
Abstractions are promising techniques to reduce MDPs and thus mitigate scalability issues.
In this work, we introduce a notion of causality on factored MDPs and a novel property-driven causal abstraction technique that retains many characteristics of the original MDP model.
For this, we rely on causal relations over state variable predicates and identify those states that share the same reasons for fulfilling or violating a given abstraction property.
We theoretically and empirically compare various causal MDP abstractions using different model types such as MDPs, interval MDPs, or stochastic games.
Our evaluation demonstrates the potential of our approach: 
For several standard benchmarks, we obtain small abstractions that allow us to compute near-optimal policies for the original MDP.
Furthermore, our causal abstractions often generalize to related large-scale MDP models.
\end{abstract}


\section{Introduction}
\label{sec:1-introduction}
Markov Decision Processes (MDPs) are a common model for sequential decision-making under uncertainty.
A standard way to specify an MDP is by factoring its state space, where each state corresponds to a valuation of a vector of state variables~\cite{DBLP:conf/aaai/StrehlDL07}. 
These variables inherently carry meaning, capturing structural properties of the environment and providing a compact description of the system dynamics.
Probabilistic model checking tools like PRISM~\cite{DBLP:conf/cav/KwiatkowskaNP11} or Storm~\cite{hensel2022storm} have input languages that follow this principle and provide mature implementations for formally analyzing MDPs.

\emph{Abstractions.}
As the number of variables increases, the induced state space quickly grows beyond the reach of exact analysis, which motivates the use of abstractions.
Abstractions reduce the size of the original MDP model, simplifying the solving while preserving the information relevant to the property of interest.
Predicate abstraction~\cite{GraSai97} groups states according to the truth values of predicates over the state variables. 
However, choosing suitable predicates is difficult, domain-dependent, and may result in abstractions that are either still too large to be solved or too coarse to be informative. 

    \begin{figure}
        \centering
        \begin{tikzpicture}[
        box/.style={rectangle,draw=black,thick, minimum size=0.98cm},
    ]
	\draw[] (0,0) grid (3,3);

    \node (x) at (0, 3.25) {x};
    \node (x0) at (0.5, 3.25) {0};
    \node (x1) at (1.5, 3.25) {1};
    \node (x2) at (2.5, 3.25) {2};
    \node (y) at (-0.25, 3) {y};
    \node (y0) at (-0.25, 2.5) {0};
    \node (y1) at (-0.25, 1.5) {1};
    \node (2) at (-0.25, 0.5) {2};

	\node (car) [scale=0.9] at (0.5, 2.5) {\huge\faIcon{car}};
    \node (charge) at (2.5, 0.5) {\huge\faIcon{charging-station}};
    \node (parking) at (0.5, 0.5) {\huge\faIcon{flag-checkered}};
    \node (passenger) at (2.5, 2.5) {\huge\faIcon{female}};
    \node [scale=0.35] (battery) at (0.25, 2.85) {\huge\faIcon{battery-full}};
\end{tikzpicture}
        \caption{The electric taxi MDP model, used as our running example throughout the paper.}
        \label{fig:fuel_taxi_mdp}
    \end{figure}
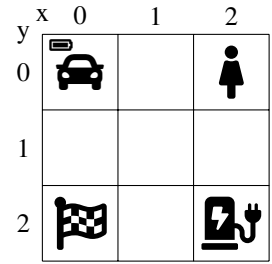
\begin{example}
    Throughout this paper, we use the electric taxi MDP~\cite{simao2021alwayssafe}, visualized in \cref{fig:fuel_taxi_mdp}, as our illustrative running example.
    In brief, the taxi must navigate a 2D grid-world, to bring a passenger to a destination, while ensuring that it does not run out of battery.
    Passing by a charging station allows the battery to be recharged. 
    Already for a $3\times3$ grid, this MDP consists of $93$ states.
    A grid of size $10\times 10$ has ca. $8200$ states, a  $100\times 100$  grid has ca. $10^7$, and even larger state spaces are to be expected in real-world scenarios. 
    Many states, however, carry similar information:
    For example, the agent's grid position can be identical across different battery levels and regardless of the passenger's status.
    Slightly different position valuations may also induce the same level of safety criticality.
    An abstraction only focused on safety can merge these states, leaving a smaller model for exact analysis.
\end{example}

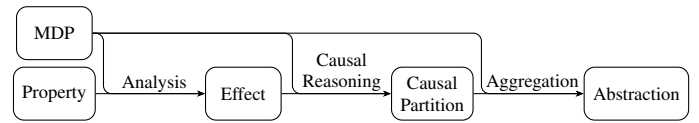
\begin{figure}[t]
    \centering
    \begin{tikzpicture}[
            box/.style={draw, rounded corners, 
            minimum width=1cm, 
            minimum height=0.7cm, 
            align=center},
            arrow/.style={->, thick},
            font=\scriptsize,
            yscale=0.8
        ]
        
    \node[box] (MDP) at (0,1)  {MDP};
    \node[box] (prop) at (0,0)  {Property};
    \node[box] (effect) at (2.5,0)  {Effect};
    \node[box] (cp) at (5,0)  {Causal\\Partition};
    \node[box] (abs) at (7.7,0)  {Abstraction};

    \draw[->,rounded corners=3pt] (MDP) to ($(MDP)+(0.65,0)$) to[out=-90,in=90] ($(MDP) + (0.65,-1)$) to node[above,yshift=-1mm] {Analysis} (effect);
    \draw[->] (prop) to (effect);

    \draw[->,rounded corners=3pt] (MDP) to ($(MDP)+(3.15,0)$) to ($(MDP) + (3.15,-1)$) to node[above,yshift=-0.9mm,align=center] {Causal\\Reasoning} (cp);
    \draw[->] (effect) to (cp);

    \draw[->,rounded corners=3pt] (MDP) to ($(MDP)+(5.65,0)$) to ($(MDP) + (5.65,-1)$) to node[above,yshift=-0.9mm,align=center] {Aggregation} (abs);
    \draw[->] (cp) to (abs);
    
    \end{tikzpicture}
    \caption{The property-driven causal abstraction workflow.}
    \label{fig:pipeline}
\end{figure}

\emph{Causality.}
The example illustrates that property-relevant information often depends on concrete valuations of specific state variables only.
To reveal such information, we employ approaches from causal inference~\cite{pearl2009causalitybook,halpern2016actualcausality}.
In our context of MDPs, such reasoning helps identify state variables and valuations that cause the satisfaction or violation of a property.

\textbf{Our approach: Property-driven causal abstractions.}
We provide formal concepts of causality for factored MDPs, establishing feature causality~\cite{dubslaff2024featurecausality} over state variable predicates.
We develop abstraction methods that (1) find causes for a given property of interest in a factored MDP, (2) use these causes to obtain state-space partitions, from which we (3) construct abstract models that provide a suitable tradeoff between model size and preservation of property-relevant information.
We now detail the steps of our approach as visualized in \cref{fig:pipeline}.

\emph{Effect sets and predicates.} 
In the first step, we analyze the input MDP regarding the property at hand, such as reaching certain unsafe states. 
Based on that, we build so-called effect sets that distinguish states with respect to (potentially multiple different) probabilities to satisfy (\enquote{good effects}) or violate (\enquote{bad effects}) the property, or more fine-grained notions such as nearly satisfying or violating.
Then, we employ a Boolean encoding of the factored state towards a predicate abstraction describing the membership of states to the various effect sets.

\emph{Causes.}
The second step introduces the novel concept of feature causality for factored MDPs.
Our notion of \emph{causes} takes into account specific valuations of state variables and not just the states, exploiting the inherent meaning of variables.
Using causal inference, we determine such causes as minimal sets of predicates sufficient to show an effect.

\emph{Causal partition.}
We then partition the state space based on causally relevant parts.
In particular, we consider three approaches towards such a partition: \textit{one-shot}, which groups states that satisfy the same causes; \textit{iterative}, which iteratively considers multiple effect sets; and \textit{causal graph}, a standard way to aggregate states that neglect specific state valuations.

\emph{Abstraction.}
Given a partition and the original MDP, we finally build an abstraction by aggregating states in the same subset of the partition into one abstract state.
To achieve a thorough evaluation of the approach, we compare three methods for abstracting an MDP's transition function: (1) using a weighted average, resulting in an MDP~\cite{simao2021alwayssafe}, (2) constructing an interval MDP where the intervals capture all possible transition probabilities~\cite{DBLP:journals/ai/GivanLD00,DBLP:conf/birthday/SuilenBB0025,DBLP:journals/sttt/BadingsSSJ23}, and (3) using a stochastic game where an opponent chooses original states from the partition~\cite{kattenbelt2010gamebasedabstraction}.

\emph{Results.}
Our paper introduces property-driven causal abstractions and demonstrates their potential: 
Our experimental evaluation shows that, on many models, causal abstractions reduce model size, while hardly compromising the quality of optimal policies.
Further, causes can generalize from small model variants to larger ones.
Our detailed analysis of different choices in analyzing the property, performing causal reasoning, and aggregating states reveals that usually, focusing on states that are close to satisfying the property leads to the best partitions, 
and that using SG-based abstractions leads to the best performance.
Overall, our contributions are the following:
 \begin{itemize}
     \item A new notion of causality exploiting the meaning of state variables in factored MDPs.
     \item Property-driven causal abstractions, including interval MDPs, and their theoretical analysis.
     \item An empirical evaluation of the tradeoff between abstraction and policy performance on standard benchmarks.
 \end{itemize}

\noindent\emph{Limitations.}
Currently, our approach requires expensive steps, namely analyzing the whole MDP to find effect sets and computing the causes.
However, our main goal is to achieve a fundamental understanding of feature causality and its potential for abstracting MDPs.
Then, given this knowledge, the next step is to develop efficient ways to approximate the effect set and its causes, for example by reinforcement learning or by solving small MDPs precisely to obtain causes and then using them to abstract larger model variants.
We remark that no automated abstraction method can be expected to work for every MDP, independent of property and structure. 
\textcolor{violet}{We view our approach as a step on the route towards scalability instead of a complete practical solution.}


\section{Related Work}
\emph{Causality} has been extensively studied in philosophy, social sciences, and artificial intelligence~\cite{Eells91,pearl2009causalitybook,Peters2017}.
Seminal work by Halpern and Pearl led to \emph{actual causality}~\cite{halpern2015causalitymodified} establishing cause-effect relationship-based structural causal models (SCMs).
\emph{Feature causality}~\cite{DubWeiBai22} is a form of actual causality that does not rely on learning SCMs~\cite{dubslaff2024featurecausality} and can serve as abductive explanations under constraints~\cite{GorRub22}.
Our approach relates most to work on causal reasoning in Markovian models, recently considered in formal verification and reinforcement learning (RL).
\emph{Probabilistic causes} were defined as state sets for which the probability of an effect increases~\cite{BaiDubFun21,baier2024foundationsprcausality,oura2025prcausalityuncertainmdp}. 
Other approaches rely on hyperproperties~\cite{abraham2018hyperpctl,DimFinkbeinerTorfah-probHyper-MDP-ATVA2020}, using counterfactual logic~\cite{Kazemi_Lally_Paoletti_2025}, or smallest prefix paths~\cite{ziemek2022probcausesmc}.
Due to their path dependency, they are not directly suitable for local causal state abstractions as we present here.
In the context of RL, causal information is mainly used to improve the performance -- cf. surveys on causal RL~\cite{zeng2023surveycrl,deng2023crlsurvey}.
Most approaches there rely on given SCMs and adapt existing model-free~\cite{bareinboim2015banditscausal,mendezmolina2020causalqlearning} and model-based methods~\cite{sun2022swartd,gasse2021crlobsintdata}.

\textit{Abstracting Markovian models.}
Various state-abstraction and refinement methods tackle the state-space-explosion problem in MDP-based analysis~\cite{li2006stateabstractionmdp}.
Probabilistic \emph{bisimulation}~\cite{LarSko91} is prevalent, and has been extended to MDPs with factored state spaces~\cite{GivDeaGre03} or using bisimulation metrics~\cite{FerPanPre04,RuaComPan15}.
PrIC3~\cite{BatJunKam20} generates ad-hoc local on-the-fly abstractions to find inductive invariants, but does not generate a-priori sufficient or necessary conditions for these.
Simão et al. considered abstractions in constrained factored MDPs during RL at the level of variables~\cite{simao2021alwayssafe}. 
Our abstraction method is more fine-grained, relying on the impact of variable valuations.

\textit{Causal abstractions.}
Lally et~al.~\cite{lally2026robustcounterfactualinferencemarkov} consider causal abstractions towards interval MDP with theoretical guarantees, however relying on SCMs compatible with an underlying MDP.
Zhang et~al.~\cite{zhang2020invariantcausalblockmdp} identify causal state features in so-called \emph{block MDPs} using invariant causal prediction to improve RL but rely on different notions of causality depending on conditional probabilities.
Another line of work learns causal graphs over state variables to derive abstractions~\cite{wang2022causaldynamicsabstraction,wang2024causalstateabstractions}. 
Clustering states in an MDP with similar causal importance for an RL objective has shown to improve learning~\cite{PraChoTap24}.
Similarly, we propose a property-driven abstraction method lumping explanatory equivalent states.

In contrast to the aforementioned abstraction approaches, we provide a foundational formal framework for property-driven causal abstraction in factored MDPs that does not rely on SCM learning, statistical data, nor is limited to the application of RL.

\section{Preliminaries}
We recall Markov decision processes (MDPs, \cref{sec:2-mdp}) and their extensions to interval MDPs and stochastic games (\cref{sec:2-imdp}) that we use for abstractions.
Finally, we provide the basic notions of feature causality (\cref{sec:2-causality}).

For bounds $\ell,u\in\Nat$, let us denote by $[\ell..u]$ the interval $\{n \in\Nat \mid \ell\leq n \leq u\}$.
Given a finite set $Y$, a \emph{partition} of $Y$ is a set $\partition = \{\partition_1, \ldots,\partition_k \subseteq Y\}$ of pairwise disjoint sets such that $\biguplus_{i\in[1,k]} \partition_i = Y$.
A \emph{probability distribution} over $Y$ is a function $\mathsf d \colon Y \to [0,1]$ where 
$\sum_{y\in Y} \mathsf d(y) = 1$. 
The set of all probability distributions over $Y$ is denoted by $\distribution(Y)$, and $Y^\omega$ is the set of all infinite sequences of elements in $Y$.

\subsection{Markov Decision Processes}\label{sec:2-mdp}

\begin{definition}[Markov decision process (MDP)]
    \label{def:mdp}
    An MDP~\cite{puterman} is a tuple $\fullmdp$ with finite sets of states $\states$ and actions $\actions$, 
    and a (partial) transition function $\transition \colon \states \times \actions \to \distribution(\states)$.
\end{definition}

We say that an MDP $\fullmdp$ is \emph{factored} 
if every state $\state\in\states$ is a vector of $n$ state valuations $(\statevarval_1, \ldots, \statevarval_n)\in \statevar_1 \times \ldots \times \statevar_n$ 
over state variables $\statevar_i$ associated with a finite domain of $m_i$ elements $\statevarval_{i,j}$ for $i\leq n$, $j<m_i$.
We write $\actions(\state)$ for the set of \emph{available actions} in state $\state\in\states$, i.e., those where $\transition$ is defined, and assume $\actions(\state) \neq \emptyset$ for all $\state\in\states$.
We write $\transition(\state,\action,\state')$ for $\transition(\state, \action)(\state')$. 

\begin{example}\label{ex:mdp}
    We model the taxi example from the introduction (see \cref{fig:fuel_taxi_mdp}) as the following MDP:
    The state space is made up of four state variables $\statevar_1 \times \statevar_2 \times \statevar_3 \times \statevar_4 \in [0..2]\times[0..2]\times[0..5]\times\{0,1\}$, with two variables for the position, one for the battery, and one for the passenger's status.
    The actions are $\actions = \{u,d,l,r,e\}$, where the first four move the taxi \textbf{u}p, \textbf{d}own, \textbf{l}eft, or \textbf{r}ight, and $e$ allows the passenger to \textbf{e}nter or \textbf{e}xit the taxi.
    Action $e$ is only available in states with $\battery>0$ at the pickup location ($x=0, y=0$).
    For a moving action, the transition function $\transition(s,a)$ assigns probability $0.9$ to the state in the chosen direction  
    and probability $0.1$ of getting stuck in traffic; in both cases, the battery decharges.
    States with $\battery=0$ are self-looping sink states, and the taxi resets $\battery$ to 5 by passing the charging station at $\xpos=2,\ypos=2$.
\end{example}

The \emph{semantics} of MDPs is defined as usual through policies and paths.
A \emph{policy} is a function $\policy \colon \states \to \actions$ assigning to each state $s$ an available action $a\in \actions(s)$.
We restrict to memoryless deterministic policies
, as these suffice for optimality under the objectives we consider~\cite{BK08}. 
A \emph{path} is an infinite sequence $\Path = \state_0 \action_0 \state_1 \action_1 \ldots \in (\states\times\actions)^\omega$ where $\transition(\state_i,\action_i,\state_{i+1}) >0$ for all $i\in\mathbb{N}_0$.
We write $\Path_i$ 
for the state $s_i$ of a path $\Path$.
$\policies$ is the set of all policies and $\Paths_\mdp$ the set of all paths.
Applying a policy $\policy$ to an MDP $\mdp$ induces for each state $\state$ a unique probability measure over paths $\reachprob^{\policy}_{\state, \mdp}$ \cite[Chapter 10.1]{BK08}. 

\para{Property, objective and value.} 
A \emph{reachability} property combines a set of target states $\targetset\subseteq \states$ with an optimization objective $\opt\in\{\min,\max\}$ that captures the optimal probability to reach $\targetset$, denoted $\lozenge \targetset \eqdef \{\Path\in\Paths_\mdp \mid \exists i\in\mathbb{N}_0\colon \Path_i\in \targetset\}$.
The \emph{value of a state}~$\state$ is $\val_{\mdp}(s) \eqdef \opt_{\policy\in\policies} \reachprob^{\policy}_{\state, \mdp}[\lozenge \targetset]$.

\begin{example}\label{ex:objective}
    An example property for the taxi model uses $\opt=\min$ and $\targetset=\{(\xpos, \ypos, \battery, \passenger) \mid \battery = 0\}$.
    Intuitively, the value of a state captures the minimum probability for it to run out of battery.
\end{example}

\subsection{Interval Markov Decision Processes and Stochastic Games}\label{sec:2-imdp}

\begin{definition}[Interval MDP (IMDP)]
    \label{def:imdp}
    An IMDP~\cite{DBLP:journals/ai/GivanLD00} is a tuple $\fullimdp$, where $\istates$ and $\actions$ are as for MDPs and
    $\itransitionlb,\itransitionub\colon\states\times\actions\times\states \to [0,1]$ are functions providing lower and upper bounds on transition probabilities.
\end{definition}

We write $\mdp\in\imdp$ to indicate that an MDP $\mdp$ is \emph{consistent} with an IMDP $\imdp$.
That is, they share state and action spaces, and for all states $\state,\state'$ and actions $\action$ we have 
$\itransitionlb(\state,\action,\state') \leq \transition(\state,\action,\state') \leq \itransitionub(\state,\action,\state')$
Note that if $\itransitionlb(\state,\action,\state')>\itransitionub(\state,\action,\state')$, there exists no consistent MDP.

\begin{example}
    Recall from \cref{ex:mdp} that the chances of being stuck in traffic are precisely known to be 0.1.
    In an IMDP, we can instead have, e.g., $\transition(\state_0, l, \state_1) = [0.8, 0.95]$ and $\transition(\state_0, l, \state_2) = [0.05, 0.2]$, adding uncertainty to the transition probability.
    In \cref{sec:4-title}, we will abstract IMDPs from MDPs by grouping states.
    Then, the bounds on the transition probabilities are determined by the extremal transition probabilities in the group of states.
\end{example}

\para{IMDP semantics and value.}
The semantics of an IMDP is the set of all consistent MDPs. 
We focus on consistent MDPs with best- and worst-case value:
For a given optimization direction $\opt\in\{\min,\max\}$, we write $\nopt$ for its complement, i.e., $\nopt \eqdef \min$ if $\opt=\max$ and $\nopt \eqdef \max$ otherwise.
Then, 
$\bval_{\imdp}(\state) \eqdef \opt_{\mdp\in\imdp} \val_{\mdp}(\state)$
and
$\wval_{\imdp}(\state) \eqdef \nopt_{\mdp\in\imdp} \val_{\mdp}(\state)$
are the best- and worst-case value of a state $\state$, respectively.

\begin{definition}[Stochastic game (SG)]
    \label{def:sg}
    An SG~\cite{Condon92} is a tuple $\fullSG$, where $\states$, $\actions$, and $\transition$ are as for MDPs and $\states = \statesMax \uplus \statesMin$, i.e., the state space is partitioned into states of the agent ($\Max$) and opponent ($\Min$).
\end{definition}

\para{SG semantics and value.} 
In SGs, we separately consider the sets of policies 
$\policies_\Max$ and $\policies_\Min$ of the agent and opponent, where $\policy_\Max \in \policies_\Max$ is a function $\statesMax \to \actions$ and analogously for the opponent.
Fixing a pair of policies $(\policy_\Max,\policy_\Min)$ induces a probability measure $\reachprob^{\policy_\Max,\policy_\Min}_{\state, \sg}$ as in MDPs.
The value of a state is $\val_{\sg}(\state) \eqdef \opt_{\policy_\Max\in\policies_\Max} \nopt_{\policy_\Min\in\policies_\Min} \reachprob^{\policy_\Max,\policy_\Min}_{\state, \sg}[\lozenge \targetset]$.
We additionally define an analogue of the best-case value of IMDPs, where the opponent states also play in favour of the agent: $\bval_{\sg}(\state) \eqdef \opt_{\policy_\Max\in\policies_\Max} \opt_{\policy_\Min\in\policies_\Min} \reachprob^{\policy_\Max,\policy_\Min}_{\state, \sg}[\lozenge \targetset]$.

\subsection{Feature Causality}\label{sec:2-causality}
Feature causality~\cite{dubslaff2024featurecausality} 
reasons about Boolean \emph{features} and their influence on properties.
For features $\boolvar$, let $\Boolvar=\{a,\overline{a} \mid a\in\boolvar\}$ denote the set of literals indicating that a feature $a$ is active ($a$) or inactive ($\overline{a}$). 
An \emph{assignment} $\partint\subseteq\Boolvar$ is a set of literals where no feature occurs twice, called \emph{total} if every feature occurs exactly once.
We denote by $\partintset(\boolvar)$ and $\totintset(\boolvar)$ the set of assignments and total assignments, respectively.
The \textit{semantics} $\semantics{\partint}$ of an assignment $\partint$ is the set of consistent total assignments, i.e., $\semantics{\partint} \eqdef \{\totint \in \totintset(\boolvar) \mid \partint\subseteq \totint\}$.
Let $\Reach \subseteq \totintset(\boolvar)$ be the set of \emph{valid} assignments satisfying domain constraints, and $\Effect \subseteq \Reach$ a set of \emph{effect} assignments. 
The goal is now to determine those assignments that are \emph{sufficient to guarantee an effect} among those that are valid. 

\begin{definition}[Cause~\cite{dubslaff2024featurecausality}]
    \label{def:featcause}
    Let $\Effect \subseteq \Reach\subseteq \totintset(\boolvar)$. 
    A \emph{cause} of $\Effect$ w.r.t. $\Reach$ is an assignment $\featcause \in \partintset(\boolvar)$, where
    \begin{enumerate}[label=\textup{\textbf{(C\arabic*)}},leftmargin=*,ref=\textup{\textbf{(C\arabic*)}}]
        \item\label{def:fc1} $\emptyset \neq \semantics{\featcause} \cap \Reach \subseteq \Effect$ (\textbf{sufficiency})
        \item\label{def:fc2} $\semantics{\delta} \cap \Reach \nsubseteq \Effect$ for all $\delta\subset\featcause$ (\textbf{subset minimality}).
    \end{enumerate}
\end{definition}
Inspired by actual causality~\cite{halpern2016actualcausality}, \ref{def:fc1} implements the sufficiency condition, ensuring that all total assignments from $\featcause$ are in $\Effect$, and \ref{def:fc2} ensures subset minimality.
A cause $\featcause$ thus can serve as explanation for any assignment $\totint\in\semantics{\featcause}\cap\Effect$.
$\Causes$ is the set of all causes of $\Effect$ w.r.t $\Reach$.

\para{Cause-effect covers.}
A \emph{cause-effect cover} is a set of causes $C\subseteq\Causes$ that explain the whole set $\Effect$, i.e., $\Effect\subseteq \bigcup_{\gamma \in C} \semantics{\gamma}$.
We call $C$ a \emph{minimal cover} if there is no cause-effect cover $C'$ where $\abs{C'}<\abs{C}$~\cite[Section 4.2]{dubslaff2024featurecausality}.
Computing an exact minimal cover is expensive,
therefore heuristics based on \emph{most general causes}~\cite[Definition 4]{dubslaff2024featurecausality} are used to establish small cause-effect covers.


\section{Causal Partitions for Factored MDPs}\label{sec:3-title}

In this section, we establish several approaches to derive partitions of MDP state spaces using property-driven causal predicates.
We thereby define the concept of feature causality for MDPs. 
We fix a factored MDP $\fullmdp$ and introduce predicates over its state variables $\statevar_i$ as a basis for predicate abstractions of its state space $\states$.
We assume that state variables $\statevar_i$ are either \emph{ordinal} or \emph{categorical}.
Ordinal variables are interpretable with a meaningful order, like a coordinate of a grid position.
Let $O\subseteq[1..n]$ be the index set of all ordinal state variables in $\mdp$.
Categorical state variables represent distinct categories like the color of a signal.

\subsection{Predicate Abstraction for Factored MDPs}\label{sec:3-MDP2partialassignment}
We define a set of Boolean features $\boolvar_\mdp$\textcolor{violet}{, which includes one Boolean for each state variable valuation except the smallest one in ordinal variables}. 
\textcolor{violet}{Then, } 
every state $\state\in\states$ uniquely corresponds to a total assignment $\totint_\state\in\totintset(\boolvar_\mdp)$.
Hence, the semantics of an assignment over the features $\boolvar_\mdp$ corresponds to a set of states in $\mdp$.
\textcolor{violet}{We refer to the Boolean features as $\boolvarval_{i,k}$ where $i$ is the index \textcolor{violet}{of the state variable $\statevar_i$} and $k \in [0..\abs{\statevar_i}]$ \textcolor{violet}{its $k$-th valuation.}
}
For an ordinal state variable $\statevar_i$ we assume the order $\statevarval_{i,j}\leq\statevarval_{i,k}$ iff $j\leq k$. 
We now describe sets of states through assignments over $\boolvar_\mdp$ modulo a theory over ordinal and categorical predicates.
\[
\boolvar_\mdp \enskip=\enskip \bigcup_{i\in O} \big\{\,\boolvarval_{i,k} \mid k\in [0..m_i]\,\big\}\setminus\big\{\,\boolvarval_{i,0} \mid i\in O\, \big\}.
\]
\textcolor{violet}{
Let $\state = (\mathfrak{s}_0, \ldots, \mathfrak{s}_n)$ be a concrete state in $\mdp$. 
The interpretation of state $\state$, $\totint_\state$, distinguishes ordinal and categorical features: }
If $\statevar_i$ is ordinal, then we define
\textcolor{violet}{$\boolvarval_{i,k} =  1$ if $\mathfrak{s}_i \geq \statevarval_{i,k}$ and $\boolvarval_{i, k}=0$ otherwise. That means that }
$\boolvarval_{i,k}$ is interpreted as a predicate $\statevar_i \geq \statevarval_{i,k}$. 
Note that binary variables $\statevar_i$ with domain $\{0,1\}$ are covered by our definition of ordinal features.
In this case, we simplify notation and write $\statevar_i$ instead of $\statevar_i \geq 1$.
If $\statevar_i$ is categorical, then \textcolor{violet}{
$\boolvarval_{i,k}=1$ iff $\mathfrak{s}_i = \statevarval_{i,k}$,} i.e., $\boolvarval_{i,k}$ is interpreted as a predicate $\statevar_i = \statevarval_{i,k}$.
Given an assignment $\partint\in\partintset(\boolvar_\mdp)$ and a state $\state$, we write $\partint(\state) = 1$ if $\totint_\state$ 
\textcolor{violet}{is in the semantics of $\partint$}, and otherwise $\partint(\state) = 0$.

\begin{example}
    \textcolor{violet}{
    In our taxi MDP running example, all state variables are ordinal. 
    We focus here on the battery: if $\battery = 3$, that implies $\battery \geq 2$, $\battery \geq 1$, and $\battery \geq 0$. 
    In our definition of taxi, the domain of this variable is $\battery \in [5]$.
    Therefore, we get five boolean variables: $\boolvarval_{\battery, k}, k \in [1..5]$. 
    For states in which $\battery=3$, these are assigned as follows: $\boolvarval_{\battery, k} = 1, k \in \{1, 2, 3\}$ and $\boolvarval_{\battery, k} = 0, k \in \{4, 5\}$.
    }
\end{example}

\subsection{One-shot Causal Partition} 

We propose a partition technique that is conceptually simple yet powerful.
For an MDP $\mdp$ and a set of predicate assignments $C \subseteq \partintset(\boolvar_\mdp)$, our partition groups states that are in the semantics of the same subset of assignments.
We refer to the partition algorithm as $\predpart$, as it 
is well-defined independent of the method used to generate $C$.
As we use causal predicates, and the input is exactly one set of~predicates~$C$, we call partitions generated using $\predpart$ \emph{one-shot causal partitions}, abbreviated to \texttt{C-1} in the evaluation. We define: 
\[
    \predpart(\mdp,C) = 
    \{
        \{ \state' \in \states \mid \forall \partint \in C \colon \partint (\state') = \partint(\state) \} \mid \state \in \states
    \}
\]

\begin{example}
    \cref{fig:taxi-partition} shows a predicate partition of the taxi MDP, in three dimensions ($x$, $y$, and $\battery$) as the passenger is causally irrelevant.
    Circles represent states in the original MDP, and colors indicate partition subsets.
    For subset 0, the predicate is $(\battery\geq1) \land (x \geq 2) \land (y \geq 2)$, which covers all states where $x$ and $y$ are $2$, and the battery level is $\geq 1$.

    \begin{figure}
        \centering
        \includegraphics[width=.5\linewidth]{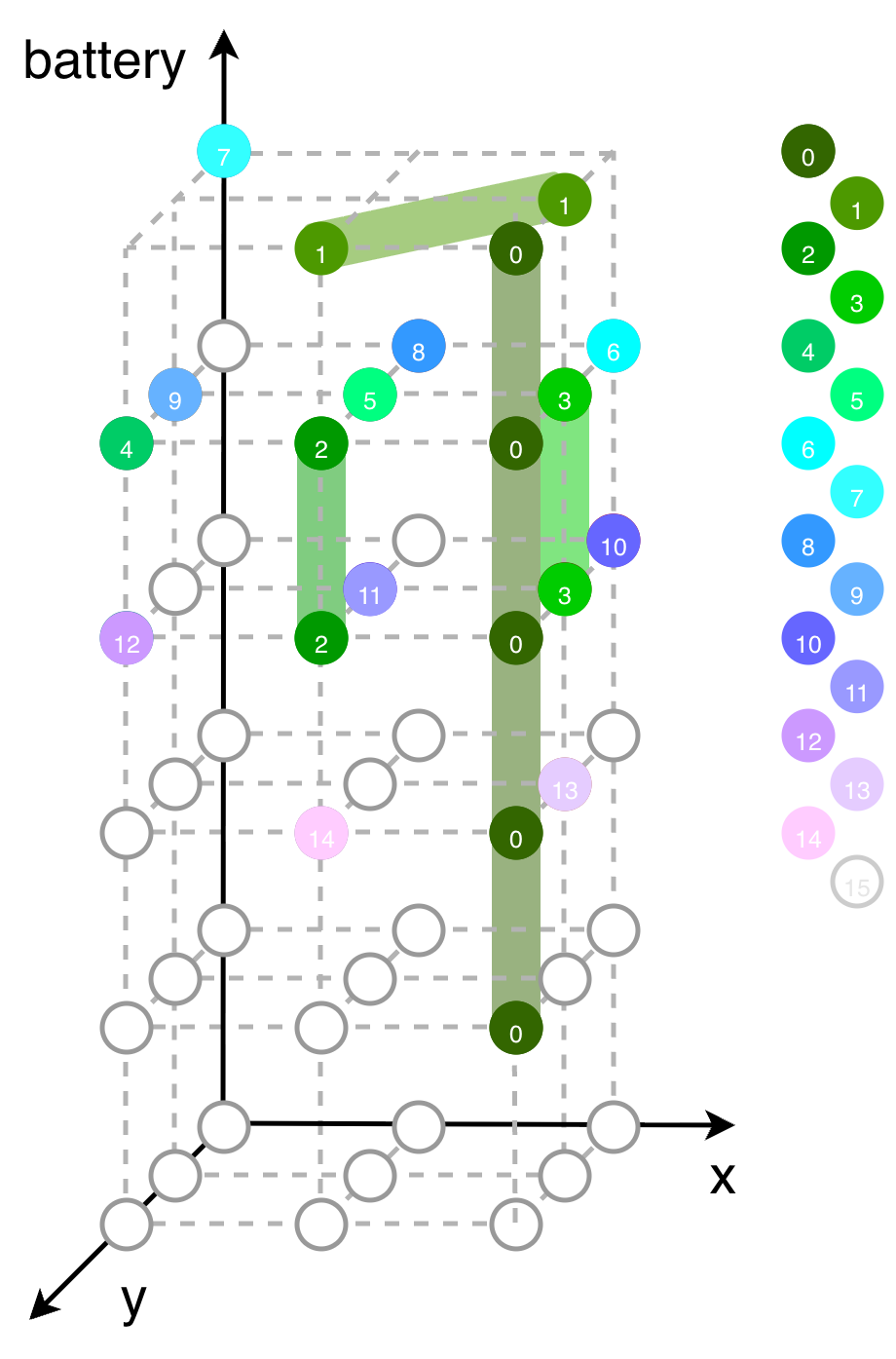}
        \caption{Predicate partition of the taxi running example.}
        \label{fig:taxi-partition}
    \end{figure}
\end{example}

When \textcolor{violet}{we use the partition} as input for aggregations (\cref{sec:4-title}), 
\textcolor{violet}{its size }determines the size of the abstract model.
There are at most $2^{\abs{C}}$ groups in a $\predpart$ partition, if the semantics of every combination of assignments are not empty. 
If $C$ is a minimal cause-effect cover, the size of the partition is minimal and still fully covers $\Effect$.
That means when $\predpart$ is used with a minimal cause-effect cover, it is optimal in size and thus serves as a theoretical milestone of what causal partitions can achieve in terms of size.
It is, however, expensive to compute, since it requires solving the model for the $\Effect$ set and computing a minimal cover of $\Effect$.
We formalize the suggested minimality characteristic in \cref{thm:3-min-cov}.

\begin{remark}
    \label{thm:3-min-cov}
    Given an MDP $\mdp$ and sets $\Reach$ and $\Effect$,
    let $C$ be a minimal cover of $\Effect$ w.r.t. $\Reach$.
    Then, $\predpart(\mdp, C)$ is the coarsest (i.e., lowest cardinality) causal partition induced by $\Reach$ and $\Effect$.
    This follows directly from the definition of minimal covers.
\end{remark}

\para{Choosing assignment sets.}
The $\Reach$ and $\Effect$ sets determine the state space partition and are 
built according to the given property and objective, such as the minimal reachability probability.
The property induces a value for each state, and we use different \emph{thresholds} on those values to partition states into the $\Reach$ and $\Effect$ sets.
We consider percentile and relative thresholds.
For percentile thresholds, $\varepsilon$ fixes the percentage of states to include. 
For example, if $\varepsilon=0.1$, and the 10\% of states with lowest value have $\val(\state) \leq 0.25$ and the highest 10\% have $\val(\state) \geq 0.93$, then we use those values as thresholds $\tau_0$ and $\tau_1$.
Relative thresholds depend not only on $\varepsilon$, but also the value of a state of interest $\state_0$, e.g.\ an initial state.
Then, the relative threshold is $\tau = \val(\state_0) \cdot (1 \pm \varepsilon)$.
For example, if $\val(\state_0)=0.5$ and $\varepsilon=0.1$, we get lower and upper relative thresholds $\tau_{0}=0.45$ and $\tau_{1}=0.55$.

Next to the thresholds, we also consider different ways to use them to construct $\Reach$ and $\Effect$.
For $\Reach$, we can either use all reachable states as $\Reach = \{\totint_\state \in \totintset(\boolvar_\mdp) \mid \state\in\states: \state \text{ is reachable}\}$ ($\Reach$ 0 in \cref{fig:3-valid_effect}), or exclude states with values larger ($\Reach$ 2) or smaller ($\Reach$ 1) than $\tau$.
Further, we formalize the $\Effect$ set using the predicates in $\boolvar_\mdp$: Given an interval $I\subseteq [0,1]$, $\Effect(I)=\{\totint_\state \in \totintset(\boolvar_\mdp)\mid \val(\state)\in I\}$.
\cref{fig:3-valid_effect} shows different effect sets for intervals under two thresholds $\tau_0$ and $\tau_1$, e.g., $\Effect$ 4 is provided by $I_4=[0,\tau_1)$.

\begin{example}
    For an MDP with a reachability property and maximizing objective, $\Reach$ 1 in \cref{fig:3-valid_effect} excludes the \enquote{worst}, and $\Reach$ 2 the \enquote{best} states.
    $\Effect$ 0 and 4 capture \enquote{bad} states. 
    The meaning of $\Effect$ 2 depends on $\Reach$.
    \textcolor{violet}{
    For the specific example in \cref{fig:taxi-partition}, we consider a \enquote{good} $\Effect$ set with $\Reach$ 0 and $\Effect$ 3.
    }
\end{example}

\begin{figure}
    \centering
    \includegraphics[width=\linewidth]{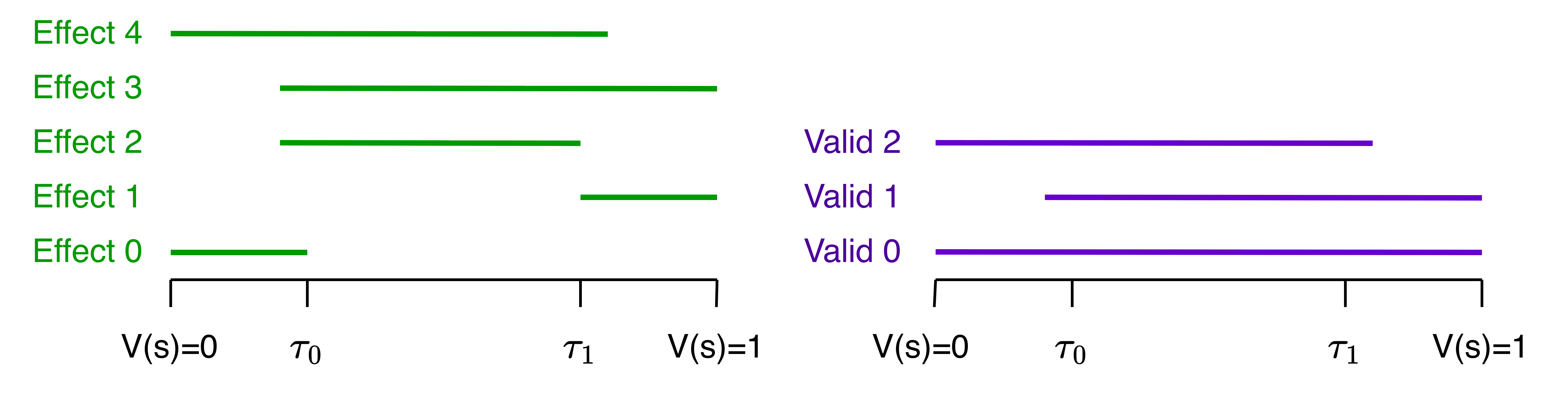}
    \caption{Different settings of $\Effect$ and $\Reach$ for $\predpart$.}
    \label{fig:3-valid_effect}
\end{figure}

\subsection{Iterative Causal Partition}

Abstractions built from $\predpart$ tend to be small and to fixate on states with extreme values, motivating more elaborate approaches.
We propose to iteratively consider finer slices of the state space as $\Effect$, compute causes for each, and build a partition from the causal slices, as shown in \cref{fig:3-iterative-partition}. 
\textcolor{violet}{This is done by solving multiple reachability queries, where the previous iteration's $\Effect$ set becomes the next target set. }
Formally, for $k$ iterations, with a threshold $\tau$ and starting from states of value 1, the initial sets are $\Effect_0 = \{\state \in \states \mid \val(\state) = 1\}$ and $\Reach_0 = \states$.
Then, for the following iterations:

\vspace{-12pt}
\begin{align*}
    \Effect_i &= \{\state \in \states \mid \reachprob_s(\lozenge \Effect_{i-1}) > \tau\}, i \in [1, k] \\
    \Reach_i &= \Reach_{i-1} \backslash \Effect_{i-1}
\end{align*}

\begin{figure}
    \centering
	\scalebox{0.8}{

\begin{tikzpicture}
  \def\W{8}    
  \def\H{4}    
  \def\cx{7.3} 
  \def\cy{0.7} 

  \fill[green!20] (0,0) rectangle (\W,\H);

  \begin{scope}
    \clip (0,0) rectangle (\W,\H);
    \fill[green!50] (\cx,\cy) circle (7.2);
    \fill[yellow!60] (\cx,\cy) circle (5);
    \fill[orange!50] (\cx,\cy) circle (3);
    \fill[red!45] (\cx,\cy) circle (1);

    \draw[thick] (\cx,\cy) circle (7.2);
    \draw[thick] (\cx,\cy) circle (5);
    \draw[thick] (\cx,\cy) circle (3);
    \draw[thick] (\cx,\cy) circle (1);
  \end{scope}
  
  \draw[thick] (0,0) rectangle (\W,\H);

  \node at (1,-0.5)  {$E_k$};
  \node at (2,-0.5)  {$\ldots$};
  \node at (3,-0.5)  {$E_2$};
  \node at (5,-0.5)  {$E_1$};
  \node at (7,-0.5)  {$E_0$};

  \fill (0.5, 0.5) circle (2pt);
  \draw[-{Stealth[length=6pt]}, thick]
    (-.4, 0.5) -- (0.5, 0.5);

  \fill (4.5, 3) circle (2pt);
  \fill (6, 2.5) circle (2pt);
  \draw[-{Stealth[length=6pt]}, thick]
    (4.5, 3) to [bend left = 40] (6, 2.5);
  \draw[-{Stealth[length=6pt]}, thick]
    (6, 2.5) to [bend left = 40] (7.3, 1.25);

  \fill (5.5, 2) circle (2pt);
  \draw[-{Stealth[length=6pt]}, thick]
    (5.5, 2) to [bend left = 10] (7, 0.75);
    
   \fill (3.5, 2) circle (2pt);
  \draw[-{Stealth[length=6pt]}, thick]
    (3.5, 2) to [bend right = 10] (5, 1.75);
    
    \fill (5.25, 0.6) circle (2pt);
  \draw[-{Stealth[length=6pt]}, thick]
    (5.25, 0.6) to [bend right = 10] (6.7, 0.6);
    
   \fill (3.5, 1) circle (2pt);
  \draw[-{Stealth[length=6pt]}, thick]
    (3.5, 1) to [bend right = 10] (5.5, 1);

\end{tikzpicture}}
    \caption{Visualization of the intuition for $\mathrm{IterPredPart}$.}
    \label{fig:3-iterative-partition}
\end{figure}
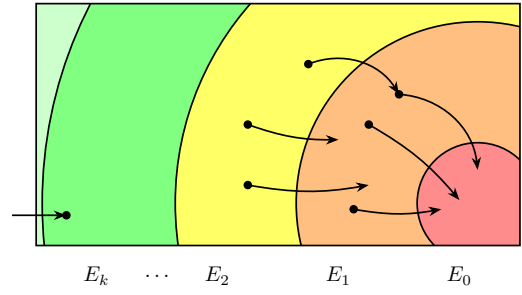

From the iterative $\Reach$ and $\Effect$ sets and the MDP, we obtain a list of sets of causes.
We use this list in \cref{alg:iterative_partition} to construct the iterative partition, called \texttt{C-IT}.
For each slice, it generates a partial partition like $\predpart$, retaining states outside $\Effect$ for future iterations.
The resulting partition is more fine-grained than $\predpart$.

\begin{algorithm}
    \begin{algorithmic}
        \Procedure{Get \texttt{C-IT}}{MDP $\mdp$, cause list $\mathrm{CS}$}
            \State $\partition = \emptyset$
            \State $\states' = \states$

            \ForAll{$C$ in $\mathrm{CS}$}
                \State $\partition = \partition \cup \{ \{\state' \in \states' \mid \forall \delta \in C \colon \delta(\state') = \delta(\state) \land \exists \delta \in C \colon \delta(\state') = \mathrm{true} \} \mid \state \in \states'\}$
                \State $\states' = \states' \backslash \{\state \in \states' \mid \exists \delta \in C \colon \delta(\state) = \mathrm{true}\}$
            \EndFor

            \If{$\abs{\states'} > 0$}
                $\partition = \partition \cup \{\states'\}$
            \EndIf
            
            \Return $\partition$
     \EndProcedure
    \end{algorithmic}
    \caption{Deriving the iterative causal partition.}
    \label{alg:iterative_partition}
\end{algorithm}

\para{Choosing iterative assignment sets.}
There are several ways to define iterations over $\Reach$ and $\Effect$ sets.
One option is to fix a threshold as for $\predpart$, and re-use it for future iterations on new $\Effect$ sets.
Alternatively, we predefine a desired number of $k$ iterations, adapt thresholds dynamically, and recursively split by the median value of remaining states.

\subsection{Causal Graph Partition}
Traditionally, causality captures dependencies between features, which are visualized in causal graphs as edges between causally dependent features.
These dependencies are usually assumed to be given or estimated from data using statistical causal discovery methods.
Here, we can omit those estimations since we have access to the full model.
Intuitively, a causal graph partition ignores features 
that do not cause any part of the effect.
Using this information, we build a causal graph and identify the indices of irrelevant features $I_\text{irr}\subseteq [1..n]$.
Then, $\cgpart$ groups states that only differ in irrelevant features:
\[
    \cgpart(\mdp,I_\text{irr}) = 
    \{
        \{ \state' \in \states \mid \forall i\in [1..n]\setminus I_\text{irr} \colon \state'_i = \state_i\} \mid \state \in \states
    \}
\]

Given the MDP $\mdp$ and $\Effect$, $\cgpart$ is always at least as large as a corresponding $\predpart$, formalized as follows.

\begin{theorem}\label{thm:cgLeqFC}
    Given an MDP $\mdp$ with sets $\Reach$ and $\Effect$, 
    a cause-effect cover $C$ of $\Effect$ w.r.t. $\Reach$, and $I_\text{irr}$ the irrelevant features given the causal graph for $\mdp$ and $\Effect$.
    Then, $\forall \states_{CG}\in\cgpart(\mdp,I_\text{irr}) \colon \exists \states_P\in\predpart(\mdp,C) \colon \states_{CG}\subseteq \states_P$.
    Thus, $\abs{\cgpart(\mdp,I_\text{irr})} \geq \abs{\predpart(\mdp,C)}$.
\end{theorem}
\begin{proof}[Proof sketch]
    The first claim follows from \cref{def:featcause}: An irrelevant feature cannot appear in any cause.
    The second claim follows from the first. 
    \ifarxivelse{\cref{app:causality-proofs}}{\cite[Appendix A]{techreport}} provides the full proof.
\end{proof}

\begin{remark}[Inputs to the partition algorithms]
    \textcolor{violet}{
    Under the same hyperparameters, all three partitions are based on, or start from, the same set of input causes.
    Specifically, the set $C$ of predicate assignments used for the one-shot partition is the set of causes that is obtained automatically using one $\Effect$ set. 
    The same set $C$ is used to derive the irrelevant features for the causal-graph partition.
    Further, $C$ is the first iteration step for threshold-based iterative partitions.
    }
\end{remark}

\section{From Partition to Abstraction}\label{sec:4-title}

\begin{figure*}[t]
    \centering
    \tikzset{
    >=Stealth,
    state/.style={rectangle, rounded corners=7pt, draw, minimum width=18pt, minimum height=14pt, inner sep=2pt},
    action/.style={circle, fill=black, inner sep=1.2pt},
    adv/.style={rectangle, draw, fill=black!10, minimum size=6pt, inner sep=0pt},
    every label/.style={font=\small},
    edge/.style={->, thick},
    line/.style={thick},
}

\begin{subfigure}[b]{0.28\linewidth}
\centering
\begin{tikzpicture}[yscale=0.7]
  \node[state] (p)  at (1.8, 0.4) {$p$};
  \node[action] (pa) at (1.8,-0.5) {};
  \node[state] (q)  at (1.0,-1.5) {$q$};
  \node[state] (r)  at (2.6,-1.5) {$r$};
  \node[action] (qb) at (1.0,-2.7) {};
  \node[action] (rb) at (2.6,-2.7) {};
  \node[state] (s)  at (0.0,-3.9) {$s$};
  \node[state] (t)  at (1.8,-3.9) {$t$};
  \node[state] (u)  at (3.6,-3.9) {$u$};

  \draw[line] (p) -- (pa) node[midway,right] {$a$};
  \draw[edge] (pa) -- (q) node[midway,above,sloped,font=\scriptsize] {$0.5$};
  \draw[edge] (pa) -- (r) node[midway,above,sloped,font=\scriptsize] {$0.5$};
  \draw[line] (q) -- (qb) node[midway,left] {$b$};
  \draw[line] (r) -- (rb) node[midway,right] {$b$};
  \draw[edge] (qb) -- (s) node[midway,above,sloped,font=\scriptsize] {$0.1$};
  \draw[edge] (qb) -- (t) node[midway,above,sloped,font=\scriptsize] {$0.9$};
  \draw[edge] (rb) -- (t) node[midway,above,sloped,font=\scriptsize] {$0.2$};
  \draw[edge] (rb) -- (u) node[midway,above,sloped,font=\scriptsize] {$0.8$};
\end{tikzpicture}
\caption{MDP}
\label{fig:abstraction-example-mdp}
\end{subfigure}\hfill
\begin{subfigure}[b]{0.22\linewidth}
\centering
\begin{tikzpicture}[yscale=0.7]
  \node[state] (p)  at (1.0, 0.4) {$\{p\}$};
  \node[action] (pa) at (1.0,-0.5) {};
  \node[state,minimum width=28pt] (qr) at (1.0,-1.5) {$\{q,r\}$};
  \node[action] (qrb) at (1.0,-2.7) {};
  \node[state] (s)  at (0.0,-3.9) {$\{s\}$};
  \node[state,minimum width=28pt] (tu) at (2.0,-3.9) {$\{t,u\}$};

  \draw[line] (p) -- (pa) node[midway,right] {$a$};
  \draw[edge] (pa) -- (qr) node[midway,right,font=\scriptsize] {$1$};
  \draw[line] (qr) -- (qrb) node[midway,right] {$b$};
  \draw[edge] (qrb) -- (s) node[midway,above,sloped,font=\scriptsize] {$0.05$};
  \draw[edge] (qrb) -- (tu) node[midway,above,sloped,font=\scriptsize] {$0.95$};
\end{tikzpicture}
\caption{WA}
\label{fig:abstraction-example-wa}
\end{subfigure}\hfill
\begin{subfigure}[b]{0.22\linewidth}
\centering
\begin{tikzpicture}[yscale=0.7]
  \node[state] (p)  at (1.0, 0.4) {$\{p\}$};
  \node[action] (pa) at (1.0,-0.5) {};
  \node[state,minimum width=28pt] (qr) at (1.0,-1.5) {$\{q,r\}$};
  \node[action] (qrb) at (1.0,-2.7) {};
  \node[state] (s)  at (0.0,-3.9) {$\{s\}$};
  \node[state,minimum width=28pt] (tu) at (2.0,-3.9) {$\{t,u\}$};

  \draw[line] (p) -- (pa) node[midway,right] {$a$};
  \draw[edge] (pa) -- (qr) node[midway,right,font=\scriptsize] {$[1,1]$};
  \draw[line] (qr) -- (qrb) node[midway,right] {$b$};
  \draw[edge] (qrb) -- (s) node[midway,above,sloped,font=\scriptsize] {$[0,0.1]$};
  \draw[edge] (qrb) -- (tu) node[midway,above,sloped,font=\scriptsize] {$[0.9,1]$};
\end{tikzpicture}
\caption{IMDP}
\label{fig:abstraction-example-imdp}
\end{subfigure}\hfill
\begin{subfigure}[b]{0.22\linewidth}
\centering
\begin{tikzpicture}[yscale=0.7]
  \node[state] (p)  at (1.0, 0.4) {$\{p\}$};
  \node[adv] (adv0)  at (1.0, -0.2) {};
  \node[action] (pa) at (1.0,-0.7) {};
  \node[state,minimum width=28pt] (qr) at (1.0,-1.5) {$\{q,r\}$};
  \node[adv] (adv1) at (0.5,-2.3) {};
  \node[adv] (adv2) at (1.5,-2.3) {};
  \node[action] (b1) at (0.3,-3.1) {};
  \node[action] (b2) at (1.7,-3.1) {};
  \node[state] (s)  at (0.0,-3.9) {$\{s\}$};
  \node[state,minimum width=28pt] (tu) at (2.0,-3.9) {$\{t,u\}$};

  \draw[line] (p) -- (adv0);
  \draw[line] (adv0) -- (pa) node[midway,right] {$a$};
  \draw[edge] (pa) -- (qr) node[midway,right,font=\scriptsize] {$1$};
  \draw[line] (qr) -- (adv1);
  \draw[line] (qr) -- (adv2);
  \draw[line] (adv1) -- (b1) node[midway,left,font=\small] {$b$};
  \draw[line] (adv2) -- (b2) node[midway,right,font=\small] {$b$};
  \draw[edge] (b1) -- (s) node[midway,above,sloped,font=\scriptsize] {$0.1$};
  \draw[edge] (b1) -- (tu) node[midway,above,sloped,font=\scriptsize] {$0.9$};
  \draw[edge] (b2) -- (tu) node[midway,right,font=\scriptsize] {$1$};
\end{tikzpicture}
\caption{SG}
\label{fig:abstraction-example-sg}
\end{subfigure}
    \caption{Example MDP and three possible abstractions after grouping states $\{q,r\}$ and $\{t,u\}$:
    (a) Original MDP; (b) weighted average (WA) MDP abstraction; (c) IMDP abstraction; (d) SG abstraction.}
    \label{fig:4-abstraction-example}
\end{figure*}
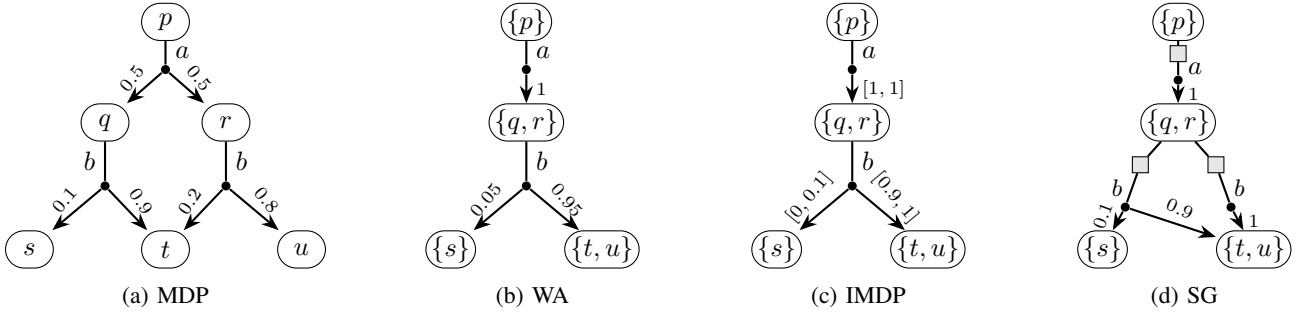

Given an MDP and a state space partition, we construct an abstract model via an \emph{aggregation method}.
This method maps each subset of the partition to a single abstract state and induces an abstract transition function that approximates the original.
\cref{sec:4-defs} provides three such methods: the weighted average method (WA) averages transition probabilities of aggregated states (\cref{def:wa_abstraction}); the IMDP method represents all possible transition probabilities as an interval (\cref{def:imdp_abstraction}); and the SG method lets an opponent player choose an original state from an abstract one (\cref{def:sg_abstraction}).
\cref{sec:4-discussion} discusses (dis-)advantages of these methods.

\subsection{Aggregation Methods}\label{sec:4-defs}

\begin{example}\label{ex:4-abstractions}

    We start with an example to provide the intuition for the three aggregation methods and showcase their differences.
    Consider the MDP in \cref{fig:abstraction-example-mdp} and a partition that groups $\{q,r\}$ and $\{t,u\}$, but leaves all other states as singletons.
    In all three abstractions, this partition forms the abstract state space,
    and successor distributions are aggregated by summing across concrete states in an abstract state.

    The \emph{WA abstraction} (\cref{fig:abstraction-example-wa}) obtains transition probabilities from an abstract state by \emph{averaging} the probabilities of the concrete states.
    E.g., $\abstransition(\{q,r\}, b)(\{s\}) = \frac 1 2 \cdot 0.1 + \frac 1 2 \cdot 0 = 0.05$.

    The \emph{IMDP abstraction} (\cref{fig:abstraction-example-imdp}) obtains transition probabilities via the minimum and maximum of the probabilities of its concrete states.
    When taking action $b$ in $\{q,r\}$, the probability for $\{s\}$ is in $[0,0.1]$
    and for $\{t,u\}$ is in $[0.9,0.2+0.8] = [0.9,1]$.
    As we prove in \cref{sec:4-discussion}, the best- and worst-case value of the IMDP always contains the original MDP's value.

    The \emph{SG abstraction} (\cref{fig:abstraction-example-sg}) frames the abstract model as a 2-player game. 
    Abstract states belong to the opponent and offer a choice of multiple successor states (gray squares).
    Each of these states corresponds to a set of original states that have equivalent outgoing distributions for each action. 
    Here, states $q$ and $r$ behave differently,
    resulting in two choices in $\{q,r\}$.
\end{example}

\para{Action- and target-consistent partition.}
Before formally defining the aggregation methods, we provide two notions that simplify the definitions.
\emph{Action-consistency} requires that all states $\state$ in an abstract state $\absstate$ share the same set of available actions.
This set then constitutes the set of available abstract actions $\absactions(\absstate)$, i.e., $\actions(\state) = \absactions(\absstate)$.
This simplifies transferring policies between abstraction and original model.
The assumption is easily satisfied by refining any given partition $\partition$ to be action-consistent by splitting groups of states into new sets that agree on their available actions.
Formally, for every subset $\absstate\in\partition$ and original state $\state\in\absstate$, the refined partition contains the set $\{\state' \in \absstate \mid \actions(\state') = \actions(\state)\}$.
Related work often assumes that all actions are available in all states~\cite{DBLP:conf/hybrid/JacksonLFL21,simao2021alwayssafe}.

Similarly, we require \emph{target-consistency}, i.e., no subset of the partition contains both target and non-target states. 
While it is possible to soundly define the abstractions without this requirement, target-consistency simplifies the definitions (an abstract state $\absstate$ is a target iff $\absstate\subseteq\targetset$). 
Target-consistency is implicitly ensured in~\cite{simao2021alwayssafe} and explicitly required in~\cite{kattenbelt2010gamebasedabstraction,DBLP:conf/hybrid/JacksonLFL21}.

\begin{definition}[Weighted Average (WA) Abstraction] 
    \label{def:wa_abstraction}
    For 
    a $\partition$ and MDP $\fullmdp$, the WA abstraction MDP is $\absmdp = (\absstates,\absactions,\abstransition)$ with:
    $\absstates = \partition$, 
    $\absactions=\actions$ and for all $\state\in\absstate$, $\actions(\absstate)=\actions(\state)$.
    For $\absstate,\absstate'\in\absstates$ and $\action\in\absactions(\absstate)$, the transition function is $\abstransition(\absstate, \action,\absstate') = \sum_{\state \in \absstate} \sum_{\state' \in \absstate'} \frac{1}{\abs{\absstate}}\transition(\state, \action, \state')$. 
\end{definition}

\begin{definition}[IMDP Abstraction]
    \label{def:imdp_abstraction}
    For $\partition$ and MDP $\fullmdp$, the IMDP abstraction is $\absimdp= (\absstates,\absactions,\itransitionlb,\itransitionub)$ with $\absstates$, and $\absactions$ as in \cref{def:wa_abstraction}.
    For $\absstate,\absstate'\in\absstates$ and $\action\in\absactions(\absstate)$, the lower and upper bounds on the transition probabilities are
    $\absitransitionlb(\absstate, a, \absstate') = \min_{\state \in \absstate} \sum_{\state' \in \absstate '} \transition(\state, a, \state')$ and 
    $\absitransitionub(\absstate, a, \absstate') = \max_{\state \in \absstate} \sum_{\state' \in \absstate '} \transition(\state, a, \state')$. 
\end{definition}

While the definitions of WA and IMDP abstraction follow naturally from \cref{ex:4-abstractions}, formally capturing the idea of the SG abstraction is more involved.
We introduce the following notation:
$\transition_\partition(\state) = \{(\action,\mathsf d) \mid \action \in \actions(\state), \mathsf d \in \distribution(\partition)\colon \mathsf d(\absstate) = \sum_{\state'\in\absstate} \transition(\state,\action,\state')\}$, i.e., $\transition_\partition(\state)$ is the set of all action-distribution pairs available in $\state$, where $\mathsf d$ is the transition function lifted to the partition. 
For example, in \cref{fig:abstraction-example-sg}, we have 
$\transition_\partition(\{q,r\}) = \{(b, [\{s\}\mapsto 0.1, \{t,u\}\mapsto 0.9]), (b,[\{t,u\}\mapsto1]\}$.
The subsets of the partition correspond to opponent states, where first, the opponent chooses an original state, and then the agent chooses an action (and a distribution over successor states). 
Thus, the agent states merge all original states that have the same action-distribution pairs.

\begin{definition}[SG Abstraction]\label{def:sg_abstraction}
    For $\partition$ and MDP $\fullmdp$, the abstract SG is $\abssg = (\absstates,\absstatesMax,\absstatesMin, \absactions, \abstransition)$ with
    $\absstatesMin = \partition$ and 
    $\absstatesMax = \{ \transition_\partition(\state) \mid \state \in \states \}$.
    $\absactions = \actions \cup \absstatesMax$, where for $\absstate\in\absstatesMin$, we have $\actions(\absstate) = \{ \transition_\partition(\state) \mid \state\in\absstate \}$, and picking the action surely leads to the selected successor, i.e., $\abstransition(\absstate, Y, Y) = 1$ for $Y\in\absstatesMax$.
    For $Y\in\absstatesMax$, recall that $Y$ is a set of action-distribution pairs $(\action,\mathsf d)$.
    Then, $\actions(Y) = \{\action \mid (\action,\mathsf d) \in Y\}$, and for $\absstate\in\absstatesMin$ we have $\abstransition(Y,\action,\absstate) = \mathsf d(\absstate)$. 
\end{definition}

\para{Definition Origins.}
The WA abstraction originates from~\cite{simao2021alwayssafe}.
We developed the IMDP abstraction independently, but it is effectively equivalent to transferring~\cite[Eqs. (5) and (6)]{DBLP:conf/hybrid/JacksonLFL21} to the discrete setting.
The SG abstraction adapts~\cite{kattenbelt2010gamebasedabstraction} to our notation, in particular adding action identifiers. 
The definitions do not require the given partition to be causal.

\subsection{Theoretical Comparison of Aggregation Methods}\label{sec:4-discussion}

We compare the aggregation methods with respect to  three measures of quality: 
(i) \emph{size}, (ii) \emph{value bounds}, i.e., how the value(s) of the abstraction relate to the original MDP's value, and (iii) \emph{policy performance}, i.e., how a worst-case optimal policy in the abstraction performs in the original MDP.

\para{Abstraction size.}
Effectively, all abstractions are of the same size with an abstract state for every subset of the partition. 
The SG abstraction contains the additional agent states, but they need not be stored explicitly, as there is a strict alternation between states of agent and opponent~\cite[Section 3]{kattenbelt2010gamebasedabstraction}.

\para{Value bounds.}
The WA abstraction does not provide any precision indication and the abstract MDP's value can be arbitrarily smaller or larger than the original MDP's. 
In contrast, the IMDP and SG abstraction yield an interval with a best- and worst-case value that is guaranteed to contain the original MDP's value.
The width of the interval is then a measure of their precision.
We summarize these insights in \cref{thm:4-abstraction-comparison}, where $\preceq$ indicates $\leq$ for $\max$ and $\geq$ for $\min$ objectives.
\begin{restatable}{theorem}{abstractioncomparison}\label{thm:4-abstraction-comparison}
    For MDP $\mdp$ and partition $\partition$, let $\absmdp$, $\absimdp$, and $\abssg$ be the abstract MDP, IMDP, and SG according to \cref{def:wa_abstraction,def:imdp_abstraction,def:sg_abstraction}.
    For every abstract state $\absstate\in\partition$ and original state $s\in\absstate$, the following hold:
    \begin{enumerate}
        \item \emph{WA no guarantee:}        
        It is possible that $\val_{\mdp}(\state) \preceq \val_{\absmdp}(\absstate)$ or that $\val_{\absmdp}(\absstate) \preceq \val_{\mdp}(\state)$.
        \label{thm:valueWA}
        \item \emph{IMDP guarantee:}
        $\wval_{\absimdp}(\absstate) \preceq \val_{\mdp}(\state) \preceq \bval_{\absimdp}(\absstate)$.
        \label{thm:valueIMDP}
        \item \emph{SG guarantee:}
        $\val_{\abssg}(\absstate) \preceq \val_{\mdp}(\state) \preceq \bval_{\abssg}(\absstate)$.
        \label{thm:valueSG}
    \end{enumerate}
\end{restatable}
\begin{proof}
    The full proof is in \ifarxivelse{\cref{app:abstraction-proofs}}{\cite[Appendix B]{techreport}}.
    There, for the WA abstraction, we provide \textcolor{violet}{concrete counterexamples where reachability probabilities in abstract states are strictly smaller or larger than in the original MDP. }
    For the IMDP abstraction, we \textcolor{violet}{first show that there exist valid probability distributions giving upper and lower bounds on the values of the original MDP that are consistent with the IMDP. We then} use an inductive argument \textcolor{violet}{to prove the relevant inequalities.} 
    For the SG abstraction, the proof from~\cite[Theorem 1]{kattenbelt2010gamebasedabstraction} still applies to our definition.
\end{proof}

From a practical perspective, the value problem for all three models used as abstractions
can be solved efficiently with dynamic programming:
\emph{value iteration}~\cite{chatterjee2008valueiteration,Condon92} for MDPs or SGs;
and \emph{robust value iteration}~\cite{iyengar2005robust} for IMDPs.
Despite the differences in the underlying models,
the efficiency and scalability of these solution methods are relatively similar.

\para{Policy performance.}
An abstract policy can be applied in the original MDP as follows:
For a policy $\abspolicy$ in the abstract MDP or IMDP, we obtain a policy $\policy$ for the original MDP by setting $\policy(\state) = \abspolicy(\absstate)$ for every original state $\state$, where $\absstate$ is the unique state with $\state\in\absstate$.
For a policy in the abstract SG, every agent state corresponds to a set $\transition_\partition(\state)$. 
Thus, we can transfer an abstract agent-policy in the SG to the MDP by selecting the action chosen in the unique abstract agent state corresponding to the original $\state$.
As the abstractions differ widely in how they aggregate information, the policy performance on the original MDP can also differ depending on the MDP's structure and the given partition, as empirically demonstrated in \cref{sec:5-title}.

\section{Experimental Evaluation}\label{sec:5-title}

Our evaluation focusses on the following questions:
\begin{enumerate}[label=\textbf{(RQ\arabic*)}, leftmargin=*]
    \item How does the property affect performance?
    \item How do different causal partitions compare?
    \item How do different aggregation methods compare?
    \item Do the causal partitions generalize to larger models?
\end{enumerate}
\subsection{Setup}

\begin{figure*}[]
    \centering
    \includegraphics[width=0.3\linewidth]{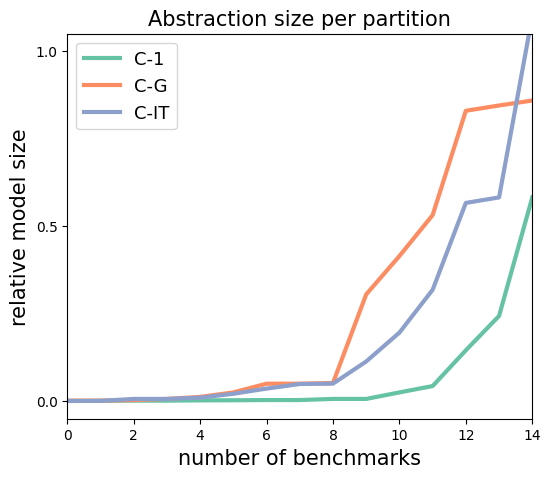} 
    \includegraphics[width=0.3\linewidth]{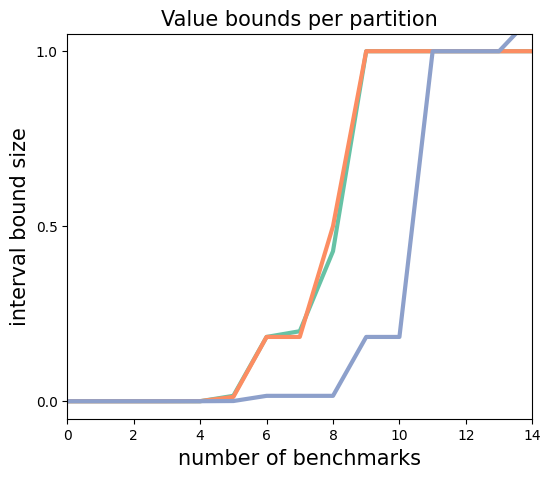}
    \includegraphics[width=0.3\linewidth]{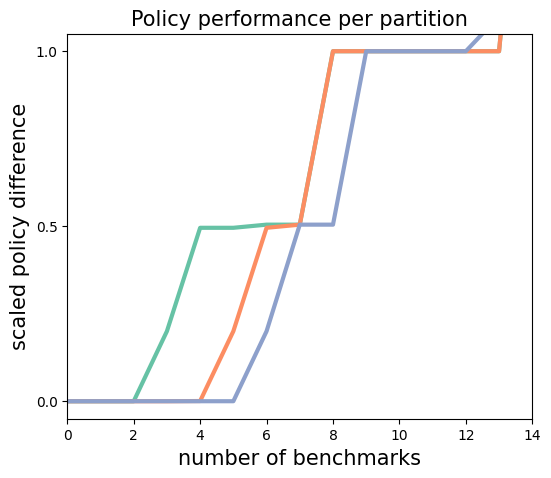}
    \caption{Performance of different causal partitions across benchmarks.}
    \label{fig:partition-performance-modified}
\end{figure*}

\para{Implementation.}
Our implementation builds on the model checkers \texttt{PRISM}~\cite{DBLP:conf/cav/KwiatkowskaNP11} and \texttt{Storm}~\cite{hensel2022storm}.
Specifically, we use version 1.11.1 of stormpy, the Python bindings of \texttt{Storm} for the general framework, and adapt \texttt{PRISM} for computing values of SGs and IMDPs.
The overall experimental pipeline is run on a 2023 MacBook Pro with Apple M3 Pro Chip and 18GiB RAM.
Cause computations were carried out with an end-to-end symbolic implementation~\cite{LieDub26} based on binary decision diagrams, and run on a workstation with Debian 12 Bookworm with AMD Ryzen Threadripper PRO 5965WX 24-Cores CPU and 8x64GiB Samsung DDR4-3200 RAM.
The most expensive step of feature cause calculation uses a timeout of up to 12 hours, see \ifarxivelse{\cref{app:timeout}}{\cite[Appendix C]{techreport}} for details.
\textcolor{violet}{Our implementation is available on Zenodo \cite{artifact}}.

\para{Benchmarks.}
We use the MDPs collected in \cite{practitionersJournalPreprint} that are available in PRISM format and have unbounded reachability or safety properties.
We exclude models where $\reachprob_{\min}{=}1$ or $\reachprob_{\max}{=}0$, as well as highly synthetic verification benchmarks with low inherent meaning in their state features. 
We configure the benchmark parameters to state spaces in the order of $10^4$ (small), $10^5$ (medium), and $10^6$ (large).
A full list is in \ifarxivelse{\cref{tab:app_benchmarks}}{\cite[Appendix C]{techreport}}.

\para{Algorithms.}
We consider all variants of $\Effect$ and $\Reach$ sets introduced in \cref{sec:3-title}, as well as the three partition types, namely causal graph (\texttt{C-G}), one-shot (\texttt{C-1}), and iterative (\texttt{C-IT}).
We combine these with the three aggregation methods \texttt{WA}, \texttt{IMDP}, and \texttt{SG}, resulting in names such as \texttt{C-IT-SG}.

\begin{remark}
    \textcolor{violet}
    {Our empirical evaluation focuses on reachability properties with minimizing and maximizing objectives. 
    The former is also called \emph{safety}, with the intuition that the probability of reaching a safety-critical state is minimized.
    Further, our definitions immediately includes reach-avoid properties, which we also analyse empirically. 
    Extending our work to reward-based properties would require only minor changes to definitions and proof statements, but does not affect the derivation of causes. 
    We exclude them, however, since checking reward-based properties on non-graph-preserving interval models is currently not supported by PRISM or Storm.}
\end{remark}

\para{Metrics.}
We consider three metrics: (1) the size of the abstraction relative to the original model size, (2) the difference between the worst and best case value for the IMDP and SG aggregation method, and (3) the relative difference between the optimal policy in the MDP and the worst case abstract policy, scaled by the difference between minimal and maximal reachability probability \ifarxivelse{(see \cref{app:exp-metrics})}{(see \cite[Appendix C]{techreport}}.
All metrics range from 0 to 1, and small values are preferable.

\begin{figure*}
    \centering
    \includegraphics[width=0.9\linewidth]{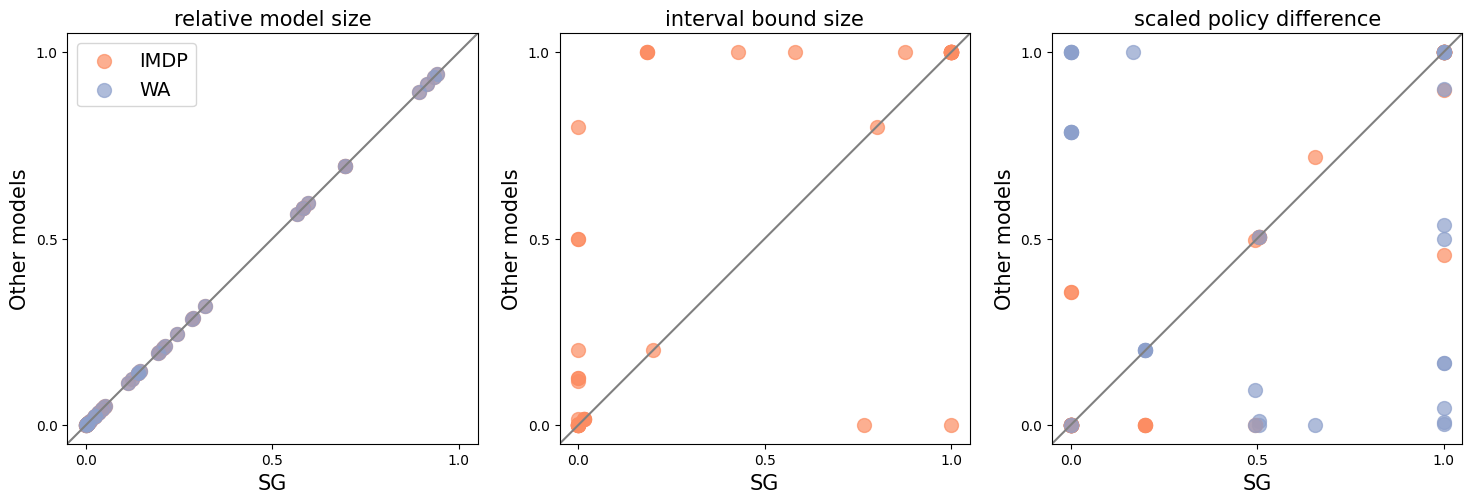}
    \caption{Comparing the performance of different abstraction methods.}
    \label{fig:abstraction-method-all}
\end{figure*}

\subsection*{\textbf{(RQ1)} How does the property affect performance?}
The $\Effect$ set can vary depending on interpretation (see \cref{fig:3-valid_effect}) and value threshold. 
We group the interpretations into those that capture \enquote{good} and \enquote{bad} effects. 
For maximizing properties, 3 and 1 capture \enquote{good} effects whereas 2 and 4 capture \enquote{bad} effects.
Moreover, for the \texttt{C-IT} partition, we can pre-determine a desired number of iterations. 
We analyse configurations using the \texttt{SG} method on small models.

\para{For the \texttt{C-1} partition, results are similar across hyperparameters.} 
The detailed analysis in \ifarxivelse{\cref{app:experiment-plots}, \cref{tab:oneshot-thresholds}}{\cite[Appendix C]{techreport}}, shows that the relative threshold with $\epsilon=0.1$ achieves the best policy performance across benchmarks.
We deem this the most important metric for the quality of our abstraction.
For $\Effect$ choices, interpretations that find causes for large parts of the state space (3 and 4 in \cref{fig:3-valid_effect}) lead to the best abstractions.

\para{For the \texttt{C-IT} partition, we receive better abstractions with fixed numbers of iterations.}
Like with \texttt{C-1}, we find that results are very similar across thresholds, and also when using 6 or 10 fixed iterations.
\ifarxivelse{\cref{fig:iterative-hyperparameters} and \cref{fig:iterative-best-parameters} in \cref{app:experiment-plots} include}{\cite[Appendix C]{techreport} includes} detailed analysis.
Between fixed thresholds or numbers of iterations, we observe that fixing the number of iterations leads to slightly larger abstractions, but with highly improved interval bounds and policy performance. 

\para{We use our findings to fix hyperparameters for the subsequent experiments.}
\textcolor{violet}{Overall we find that while some hyper-parameter values have little influence, such as the concrete threshold values, others matter greatly. In particular, we find that the semantics of the $\Effect$ set are decisive, and we observe the best performance with $\Effect$ sets covering a large part of the \enquote{good} states. Based on these findings, f}or the \texttt{C-1} partition, we continue with relative $\epsilon=0.1$ and use the $\Effect$ sets numbered 3 and 4 in \cref{fig:3-valid_effect}.
For the \texttt{C-IT} partition, 
\textcolor{violet}{we again observe that concrete threshold values have little impact, but fixing a larger number of iteration yields better performance. Therefore, }
we iterate 6 times, since that leads to the best ratio of performance and computational effort.
In both cases, we observe that \enquote{good} effect sets achieve better abstractions than \enquote{bad}
effects and include visualizations in the \ifarxivelse{\cref{app:exp_details} in \cref{fig:hyperparameters-good-bad}}{\cite[Appendix C]{techreport}}.
To still cover both interpretations, we continue with the hyperparameters on both types of effects.

\begin{figure*}
    \centering
    \includegraphics[width=0.9\linewidth]{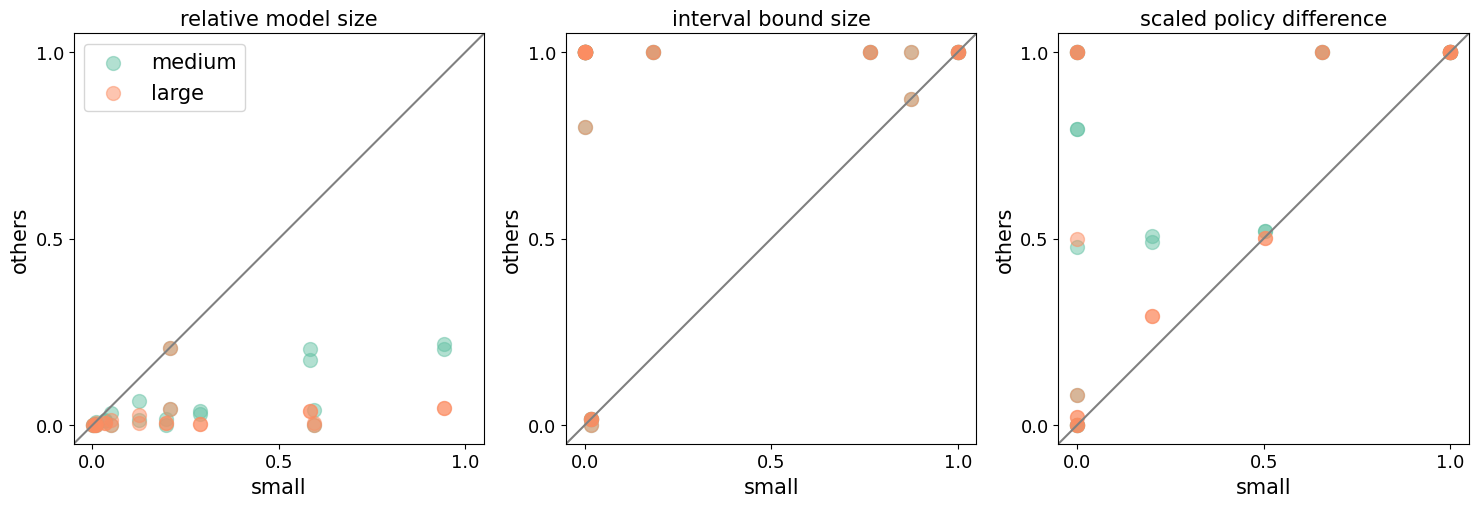}
    \caption{Generalizing causes from the small models to larger model sizes.}
    \label{fig:generalization-small-causes}
\end{figure*}

\subsection*{\textbf{(RQ2)} How do different causal partitions compare?}

We compare the \texttt{C-G}, \texttt{C-IT}, and \texttt{C-1} partitions with the previously fixed hyperparameters using the \texttt{SG} abstraction. 
The \texttt{C-G} partition uses the same causes as \texttt{C-1}.

\para{The \texttt{C-G} partition does not modify the model in 38\% of our benchmarks.}
In these cases, no state variables can be completely removed since all are causally relevant to the property in some states and the abstract and original model are the same. 
We exclude those benchmarks from further evaluation but include details in \ifarxivelse{\cref{app:exp_details}, \cref{fig:partition-performance-all}}{\cite[Appendix C]{techreport}}.

\para{\texttt{C-IT} partitions produce better abstractions than \texttt{C-1} and \texttt{C-G}.}
In terms of model size, the \texttt{C-1} partition produces the smallest abstract models on most benchmarks, retaining less than 20\% of the original states on 13 out of 14 benchmarks.
\texttt{C-G} returns the largest sized abstractions on most benchmarks.
However, all three partitions are able to significantly reduce model sizes across benchmarks.
On the qualitative metrics however, \texttt{C-IT}-based abstractions achieve the smallest value bounds and best policy performance, with small bounds on 10 out of 14 benchmarks and good policies on 5 out of 14 benchmarks.
These results are visualized in \cref{fig:partition-performance-modified}.
\textcolor{violet}{We find that \texttt{C-G}, which is usually employed by related works, is consistently outperformed. \texttt{C-IT} performs best, as it captures the most fine-grained interplay of causes. 
It is, however, expensive to derive.
\texttt{C-1}, with lower computational cost and often reasonable performance marks a middle ground.}

\subsection*{\textbf{(RQ3)} How do different aggregation methods compare?} 

We compare the performance of the aggregation methods in \cref{fig:abstraction-method-all}, with more details in \ifarxivelse{\cref{app:exp_details} (\cref{fig:5-abstractions-quantile,fig:5-abstractions-good-bad})}{\cite[Appendix C]{techreport}}.
Our empirical results confirm \cref{thm:4-abstraction-comparison}, i.e., the interval for the abstract IMDP and SG always contain the original value.

\para{\texttt{SG} abstractions perform best, followed by \texttt{IMDP} and \texttt{WA}.}
Since all abstractions are built from the same partition, abstract model size is always the same.
The \texttt{SG} abstractions mostly achieve smaller interval value bounds than the \texttt{IMDP} ones, often even achieving an interval size of 0.
The policy performance however differs more strongly between the abstraction methods.
\texttt{WA} abstractions often achieve a policy difference close to 0, but can also result in arbitrarily wrong policies without any guarantees.
Between \texttt{IMDP}s and \texttt{SG}s, the results are often similar.
As in the previous experiments, we observe that all methods perform better on effects that are interpreted as \enquote{good} than those that are \enquote{bad}\ifarxivelse{ (see \cref{app:experiment-plots})}{}.
\textcolor{violet}{Our findings here show a trade-off between performance and guarantees versus computational cost. 
Particularly, the \texttt{SG} aggregations offers guarantees and the best performance, as it tightly preserves uncertainty. 
This however comes at increased computational cost, compared to \texttt{WA}. 
The latter still often performs well and is easy to solve, but can be arbitrarily wrong, as the averaging may neglect specific behaviors.}

\subsection*{\textbf{(RQ4)} Do the causal partitions generalize to larger models?}

Using causality in MDPs and RL is often motivated by the potential of causal relations to generalize to scaled-up models. 
We test the potential of our causes to generalize to medium and large sized models \ifarxivelse{(see \cref{tab:app_benchmarks})}.
We apply causes retrieved on small models to larger settings, and report the results in \cref{fig:generalization-small-causes}.
Additionally, \ifarxivelse{\cref{app:exp_details}}{\cite[Appendix C]{techreport}} compares with medium-sized causes on the medium-sized models.

\para{The causes generalize well in terms of size but at reduced performance.}
The small causes decrease the relative size of larger models to less than 20\%, even in cases where the small abstraction maintained more than 60\% of the original size.
However, the policy difference often becomes worse, sometimes degrading to the worst possible policy.
Still, on several models, the abstract policy attains a policy difference close to 0 on the large models, indicating optimal performance.

\para{Abstract policies perform well on large models.}
The policy performance for large models is often closer to the original model performance than for medium abstractions, even if both are derived from small causes.
This is interesting, as obtaining causes for larger models is computationally more expensive.
Our findings demonstrate the potential of our approach and provide an avenue to efficiently apply property-driven causal abstractions even on large models, provided they are parameterized to allow for small variants.

\para{Causes fitted to the model size increase performance, but are computationally expensive.}
As we show in \ifarxivelse{\cref{app:exp_details} in \cref{fig:generalization-small-medium-causes} and \cref{fig:generalization-medium-models}}{\cite[Appendix C]{techreport}}, those abstractions built from medium-sized causes usually receive tighter value bounds and better policy performance.
The calculation of most general causes however becomes more computationally expensive on larger models, with several experimental configurations timing out even after 12 hours, as documented in \ifarxivelse{\cref{tab:app-fc-timeouts}}{\cite[Appendix C]{techreport}}. 
\textcolor{violet}{Our straight-forward approach for generalizing causes between models shows our framework's potential to scale to models of larger sizes without further increase in computational demands.}

\subsection{General findings and practical implications}

\textcolor{violet}{In our empirical evaluation, we observe an overall mixed performance across benchmarks, and across combinations of partitions and aggregation methods. 
We now provide further insights on our key takeaways from this evaluation.}

\textcolor{violet}{Concerning benchmarks, we find that our abstractions are effective when whole state variables or large parts of their domain are irrelevant to the property of interest.
In these cases, our abstractions become very small in size while maintaining good policy performance.
On the other hand, handcrafted and highly artificial models where state variables do not carry much structural meaning, and where the outcome is solely dependent on a single state, do not produce effective abstractions. 
These benchmarks can show extreme differences in values between states, leading to a singular cause, and an abstraction consisting of only two abstract states.}

\textcolor{violet}{
Regarding the practical implications to draw from our experiments, in particular, we find a trade-off between computational expenses and the informativeness of partitions and therein the preciseness of abstractions. Specifically, we find that a combination of C-IT-SG is the most reliable yet most expensive choice. C-1-WA on the other hand offers a lightweight heuristic without guarantees on abstraction quality, that nonetheless often works well in practice. However, combinations of C-1 or C-IT with the SG aggregation method is preferable when guarantees matter.}

\section{Conclusion}
We formalize feature causality in MDPs to derive several types of property-driven predicate-based causal abstractions.
We combine our causal state space partitions with three aggregation methods and analyze their theoretical (dis-)advantages.
Empirical evaluation shows that our proposed methods can significantly reduce model size while maintaining good policy performance on several standard benchmarks.
Further, the derived feature causes show potential to generalize to larger model sizes.
A relevant limitation is that our approach requires to exactly analyze the full model to obtain the causal information.
However, our aim is to gain a fundamental understanding of how causality can be applied to obtain abstractions that are capable of retaining relevant parts of the model.
Future work aims to tackle this limitation by, among others, approximating the causal predicates from data or learned models.

\section*{Acknowledgments}

This work was funded by the ERC Starting Grant 101077178 (DEUCE) and supported by the DFG project TRR 248 (see https://perspicuous-computing.science, project ID 389792660), and the NWO Veni grant VI.Veni.222.431.

\bibliographystyle{ieeetr}
\bibliography{references}

@article{LieDub26,
	author = {Liem, Edward and Dubslaff, Clemens},
	howpublished = {Technical Report},
	journal = {CoRR},
	title = {Implicit Computation of Filtered Prime Implicants},
	volume = {abs/2607.26787},
	year = {2026}}

@article{techreport,
	author = {Schmidt, Jule and Weininger, Maximilian and Dubslaff, Clemens and Parker, Dave and Jansen, Nils},
	journal = {CoRR},
	title = {Property-driven Causal Abstractions for Markov Decision Processes},
	volume = {abs/2607.26787},
	year = {2026}}

@inproceedings{lally2026robustcounterfactualinferencemarkov,
	author = {Lally, Jessica and Kazemi, Milad and Paoletti, Nicola},
	booktitle = {AAMAS},
	note = {to appear},
	title = {Robust Counterfactual Inference in {M}arkov Decision Processes},
	year = {2026}}

@inproceedings{PraChoTap24,
	author = {Pranger, Stefan and Chockler, Hana and Tappler, Martin and K\"{o}nighofer, Bettina},
	booktitle = {Advances in Neural Information Processing Systems},
	doi = {10.52202/079017-0881},
	editor = {Globerson, A. and Mackey, L. and Belgrave, D. and Fan, A. and Paquet, U. and Tomczak, J. and Zhang, C.},
	pages = {28103--28126},
	publisher = {Curran Associates, Inc.},
	title = {Test Where Decisions Matter: Importance-driven Testing for Deep Reinforcement Learning},
	url = {https://proceedings.neurips.cc/paper_files/paper/2024/file/317ccced29ed464df181c781cb436180-Paper-Conference.pdf},
	volume = {37},
	year = {2024}}

@inproceedings{GraSai97,
	author = {Graf, Susanne and Sa{\"{\i}}di, Hassen},
	booktitle = {{CAV}},
	pages = {72--83},
	publisher = {Springer},
	series = {Lecture Notes in Computer Science},
	title = {Construction of Abstract State Graphs with {PVS}},
	volume = {1254},
	year = {1997}}

@article{Condon92,
	author = {Condon, Anne},
	doi = {10.1016/0890-5401(92)90048-K},
	journal = {Inf. Comput.},
	number = {2},
	pages = {203--224},
	title = {The Complexity of Stochastic Games},
	url = {https://doi.org/10.1016/0890-5401(92)90048-K},
	volume = {96},
	year = {1992}}

@inproceedings{DBLP:conf/hybrid/JacksonLFL21,
	author = {Jackson, John and Laurenti, Luca and Frew, Eric W. and Lahijanian, Morteza},
	booktitle = {{HSCC}},
	pages = {6:1--6:11},
	publisher = {{ACM}},
	title = {Strategy synthesis for partially-known switched stochastic systems},
	url = {https://doi.org/10.1145/3447928.3456649},
	year = {2021}}

@inproceedings{BatJunKam20,
	author = {Batz, Kevin and Junges, Sebastian and Kaminski, Benjamin Lucien and Katoen, Joost{-}Pieter and Matheja, Christoph and Schr{\"{o}}er, Philipp},
	booktitle = {{CAV} {(2)}},
	pages = {512--538},
	publisher = {Springer},
	series = {Lecture Notes in Computer Science},
	title = {PrIC3: Property Directed Reachability for MDPs},
	volume = {12225},
	year = {2020}}

@article{LarSko91,
	author = {Larsen, Kim G. and Skou, Arne},
	doi = {https://doi.org/10.1016/0890-5401(91)90030-6},
	issn = {0890-5401},
	journal = {Information and Computation},
	number = {1},
	pages = {1-28},
	title = {Bisimulation through probabilistic testing},
	url = {https://www.sciencedirect.com/science/article/pii/0890540191900306},
	volume = {94},
	year = {1991}}

@inproceedings{FerPanPre04,
	author = {Ferns, Norm and Panangaden, Prakash and Precup, Doina},
	booktitle = {{UAI}},
	pages = {162--169},
	publisher = {{AUAI} Press},
	title = {Metrics for Finite Markov Decision Processes},
	year = {2004}}

@article{GivDeaGre03,
	author = {Givan, Robert and Dean, Thomas and Greig, Matthew},
	doi = {https://doi.org/10.1016/S0004-3702(02)00376-4},
	issn = {0004-3702},
	journal = {Artificial Intelligence},
	note = {Planning with Uncertainty and Incomplete Information},
	number = {1},
	pages = {163-223},
	title = {Equivalence notions and model minimization in Markov decision processes},
	url = {https://www.sciencedirect.com/science/article/pii/S0004370202003764},
	volume = {147},
	year = {2003}}

@inproceedings{RuaComPan15,
	author = {Ruan, Sherry Shanshan and Comanici, Gheorghe and Panangaden, Prakash and Precup, Doina},
	booktitle = {{AAAI}},
	pages = {3578--3584},
	publisher = {{AAAI} Press},
	title = {Representation Discovery for MDPs Using Bisimulation Metrics},
	year = {2015}}

@inproceedings{DimFinkbeinerTorfah-probHyper-MDP-ATVA2020,
	author = {Dimitrova, Rayna and Finkbeiner, Bernd and Torfah, Hazem},
	booktitle = {{ATVA}},
	pages = {484--500},
	publisher = {Springer},
	series = {Lecture Notes in Computer Science},
	title = {Probabilistic Hyperproperties of Markov Decision Processes},
	volume = {12302},
	year = {2020}}

@inproceedings{abraham2018hyperpctl,
	author = {{\'{A}}brah{\'{a}}m, Erika and Bonakdarpour, Borzoo},
	booktitle = {{QEST}},
	pages = {20--35},
	publisher = {Springer},
	series = {Lecture Notes in Computer Science},
	title = {HyperPCTL: {A} Temporal Logic for Probabilistic Hyperproperties},
	volume = {11024},
	year = {2018}}

@book{Eells91,
	author = {Eells, Ellery},
	collection = {Cambridge Studies in Probability, Induction and Decision Theory},
	place = {Cambridge},
	publisher = {Cambridge University Press},
	series = {Cambridge Studies in Probability, Induction and Decision Theory},
	title = {{Probabilistic Causality}},
	year = {1991}}

@book{Peters2017,
	address = {Cambridge, MA, USA},
	author = {Peters, Jonas and Janzing, Dominik and Sch\"olkopf, Bernhard},
	publisher = {MIT Press},
	title = {Elements of Causal Inference: Foundations and Learning Algorithms},
	year = {2017}}

@inproceedings{DubWeiBai22,
	author = {Dubslaff, Clemens and Weis, Kallistos and Baier, Christel and Apel, Sven},
	booktitle = {{ICSE}},
	pages = {325--337},
	publisher = {{ACM}},
	title = {Causality in Configurable Software Systems},
	year = {2022}}

@article{GorRub22,
	author = {Gorji, Niku and Rubin, Sasha},
	booktitle = {{AAAI}},
	pages = {5660--5667},
	publisher = {{AAAI} Press},
	title = {Sufficient Reasons for Classifier Decisions in the Presence of Domain Constraints},
	year = {2022}}

@inproceedings{BaiDubFun21,
	author = {Baier, Christel and Dubslaff, Clemens and Funke, Florian and Jantsch, Simon and Majumdar, Rupak and Piribauer, Jakob and Ziemek, Robin},
	booktitle = {{ICALP}},
	pages = {1:1--1:20},
	publisher = {Schloss Dagstuhl - Leibniz-Zentrum f{\"{u}}r Informatik},
	series = {LIPIcs},
	title = {From Verification to Causality-Based Explications (Invited Talk)},
	volume = {198},
	year = {2021}}

@article{Kazemi_Lally_Paoletti_2025,
	author = {Kazemi, Milad and Lally, Jessica and Paoletti, Nicola},
	doi = {10.1017/cbp.2025.2},
	journal = {Research Directions: Cyber-Physical Systems},
	pages = {e3},
	title = {Causal temporal reasoning for Markov decision processes},
	volume = {3},
	year = {2025}}

@book{puterman,
	author = {Puterman, Martin L.},
	doi = {10.1002/9780470316887},
	isbn = {978-0-47161977-2},
	publisher = {Wiley},
	series = {Wiley Series in Probability and Statistics},
	title = {{M}arkov Decision Processes: Discrete Stochastic Dynamic Programming},
	year = {1994}}

@book{BK08,
	author = {Baier, Christel and Katoen, Joost-Pieter},
	isbn = {978-0-262-02649-9},
	publisher = {{MIT} Press},
	title = {Principles of model checking},
	url = {https://mitpress.mit.edu/books/principles-model-checking},
	year = {2008}}

@article{DBLP:journals/ai/GivanLD00,
	author = {Givan, Robert and Leach, Sonia M. and Dean, Thomas L.},
	doi = {10.1016/S0004-3702(00)00047-3},
	journal = {Artif. Intell.},
	number = {1-2},
	pages = {71--109},
	title = {Bounded-parameter {M}arkov decision processes},
	url = {https://doi.org/10.1016/S0004-3702(00)00047-3},
	volume = {122},
	year = {2000}}

@book{halpern2016actualcausality,
	author = {Halpern, Joseph Y.},
	doi = {10.7551/mitpress/10809.001.0001},
	eprint = {https://direct.mit.edu/book-pdf/2262849/book\_9780262336611.pdf},
	isbn = {9780262336611},
	month = {08},
	publisher = {The MIT Press},
	title = {Actual Causality},
	url = {https://doi.org/10.7551/mitpress/10809.001.0001},
	year = {2016}}

@inproceedings{halpern2015causalitymodified,
	author = {Halpern, Joseph Y.},
	booktitle = {{IJCAI}},
	pages = {3022--3033},
	publisher = {{AAAI} Press},
	title = {A Modification of the Halpern-Pearl Definition of Causality},
	year = {2015}}

@article{dubslaff2024featurecausality,
	author = {Dubslaff, Clemens and Weis, Kallistos and Baier, Christel and Apel, Sven},
	journal = {J. Syst. Softw.},
	pages = {111915},
	title = {Feature causality},
	volume = {209},
	year = {2024}}

@inproceedings{li2006stateabstractionmdp,
	author = {Li, Lihong and Walsh, Thomas J. and Littman, Michael L.},
	booktitle = {AI{\&}M},
	title = {Towards a Unified Theory of State Abstraction for {MDP}s},
	year = {2006}}

@article{zeng2023surveycrl,
	author = {Zeng, Yan and Cai, Ruichu and Sun, Fuchun and Huang, Libo and Hao, Zhifeng},
	journal = {CoRR},
	title = {A Survey on Causal Reinforcement Learning},
	volume = {abs/2302.05209},
	year = {2023}}

@inproceedings{wang2024causalstateabstractions,
	author = {Wang, Zizhao and Wang, Caroline and Xiao, Xuesu and Zhu, Yuke and Stone, Peter},
	booktitle = {{AAAI}},
	pages = {15778--15786},
	publisher = {{AAAI} Press},
	title = {Building Minimal and Reusable Causal State Abstractions for Reinforcement Learning},
	year = {2024}}

@misc{deng2023crlsurvey,
	archiveprefix = {arXiv},
	author = {Deng, Zhihong and Jiang, Jing and Long, Guodong and Zhang, Chengqi},
	eprint = {2307.01452},
	primaryclass = {cs.LG},
	title = {Causal Reinforcement Learning: A Survey},
	url = {https://arxiv.org/abs/2307.01452},
	year = {2023}}

@inproceedings{simao2021alwayssafe,
	author = {Sim{\~{a}}o, Thiago D. and Jansen, Nils and Spaan, Matthijs T. J.},
	booktitle = {{AAMAS}},
	pages = {1226--1235},
	publisher = {{ACM}},
	title = {AlwaysSafe: Reinforcement Learning without Safety Constraint Violations during Training},
	year = {2021}}

@inproceedings{wang2022causaldynamicsabstraction,
	author = {Wang, Zizhao and Xiao, Xuesu and Xu, Zifan and Zhu, Yuke and Stone, Peter},
	booktitle = {{ICML}},
	pages = {23151--23180},
	publisher = {{PMLR}},
	series = {Proceedings of Machine Learning Research},
	title = {Causal Dynamics Learning for Task-Independent State Abstraction},
	volume = {162},
	year = {2022}}

@article{kattenbelt2010gamebasedabstraction,
	author = {Kattenbelt, Mark and Kwiatkowska, Marta Z. and Norman, Gethin and Parker, David},
	journal = {Formal Methods Syst. Des.},
	number = {3},
	pages = {246--280},
	title = {A game-based abstraction-refinement framework for Markov decision processes},
	volume = {36},
	year = {2010}}

@inproceedings{chatterjee2008valueiteration,
	author = {Chatterjee, Krishnendu and Henzinger, Thomas A.},
	booktitle = {25 Years of Model Checking},
	pages = {107--138},
	publisher = {Springer},
	series = {Lecture Notes in Computer Science},
	title = {Value Iteration},
	volume = {5000},
	year = {2008}}

@article{hensel2022storm,
	author = {Hensel, Christian and Junges, Sebastian and Katoen, Joost{-}Pieter and Quatmann, Tim and Volk, Matthias},
	journal = {Int. J. Softw. Tools Technol. Transf.},
	number = {4},
	pages = {589--610},
	title = {The probabilistic model checker Storm},
	volume = {24},
	year = {2022}}

@inproceedings{zhang2020invariantcausalblockmdp,
	author = {Zhang, Amy and Lyle, Clare and Sodhani, Shagun and Filos, Angelos and Kwiatkowska, Marta and Pineau, Joelle and Gal, Yarin and Precup, Doina},
	booktitle = {{ICML}},
	pages = {11214--11224},
	publisher = {{PMLR}},
	series = {Proceedings of Machine Learning Research},
	title = {Invariant Causal Prediction for Block MDPs},
	volume = {119},
	year = {2020}}

@article{ziemek2022probcausesmc,
	author = {Ziemek, Robin and Piribauer, Jakob and Funke, Florian and Jantsch, Simon and Baier, Christel},
	journal = {Innov. Syst. Softw. Eng.},
	number = {3},
	pages = {347--367},
	title = {Probabilistic causes in Markov chains},
	volume = {18},
	year = {2022}}

@article{sun2022swartd,
	author = {Sun, Hao},
	journal = {CoRR},
	title = {Toward Causal-Aware {RL:} State-Wise Action-Refined Temporal Difference},
	volume = {abs/2201.00354},
	year = {2022}}

@inproceedings{bareinboim2015banditscausal,
	author = {Bareinboim, Elias and Forney, Andrew and Pearl, Judea},
	booktitle = {{NIPS}},
	pages = {1342--1350},
	title = {Bandits with Unobserved Confounders: {A} Causal Approach},
	year = {2015}}

@article{gasse2021crlobsintdata,
	author = {Gasse, Maxime and Grasset, Damien and Gaudron, Guillaume and Oudeyer, Pierre{-}Yves},
	journal = {CoRR},
	title = {Causal Reinforcement Learning using Observational and Interventional Data},
	volume = {abs/2106.14421},
	year = {2021}}

@inproceedings{oura2025prcausalityuncertainmdp,
	author = {Oura, Ryohei and Ito, Yuji},
	booktitle = {{UAI}},
	pages = {3300--3321},
	publisher = {{PMLR}},
	series = {Proceedings of Machine Learning Research},
	title = {Probability-Raising Causality for Uncertain Parametric Markov Decision Processes with {PAC} Guarantees},
	volume = {286},
	year = {2025}}

@article{mendezmolina2020causalqlearning,
	author = {M{\'{e}}ndez{-}Molina, Arqu{\'{\i}}mides and Feliciano{-}Avelino, Ivan and Morales, Eduardo F. and Sucar, Luis Enrique},
	journal = {Res. Comput. Sci.},
	number = {3},
	pages = {95--104},
	title = {Causal Based Q-Learning},
	volume = {149},
	year = {2020}}

@book{pearl2009causalitybook,
	author = {Pearl, Judea},
	edition = {2},
	place = {Cambridge},
	publisher = {Cambridge University Press},
	title = {Causality},
	year = {2009}}

@article{baier2024foundationsprcausality,
	author = {Baier, Christel and Piribauer, Jakob and Ziemek, Robin},
	journal = {Log. Methods Comput. Sci.},
	number = {1},
	title = {Foundations of probability-raising causality in Markov decision processes},
	volume = {20},
	year = {2024}}

@inproceedings{DBLP:conf/cav/KwiatkowskaNP11,
	author = {Kwiatkowska, Marta Z. and Norman, Gethin and Parker, David},
	booktitle = {{CAV}},
	doi = {10.1007/978-3-642-22110-1\_47},
	pages = {585--591},
	publisher = {Springer},
	series = {Lecture Notes in Computer Science},
	title = {{PRISM} 4.0: Verification of Probabilistic Real-Time Systems},
	url = {https://doi.org/10.1007/978-3-642-22110-1\_47},
	volume = {6806},
	year = {2011}}

@article{practitionersJournalPreprint,
	author = {Hartmanns, Arnd and Junges, Sebastian and Quatmann, Tim and Weininger, Maximilian},
	doi = {10.1007/s10009-026-00848-y},
	journal = {Int J Softw Tools Technol Transfer},
	title = {The Revised Practitioner's Guide to {MDP} Model Checking Algorithms},
	year = {2026}}

@inproceedings{DBLP:conf/aaai/StrehlDL07,
	author = {Strehl, Alexander L. and Diuk, Carlos and Littman, Michael L.},
	booktitle = {{AAAI}},
	pages = {645--650},
	publisher = {{AAAI} Press},
	title = {Efficient Structure Learning in Factored-State MDPs},
	year = {2007}}

@article{iyengar2005robust,
	author = {Iyengar, Garud N},
	journal = {Mathematics of Operations Research},
	pages = {257--280},
	publisher = {INFORMS},
	title = {Robust dynamic programming},
	volume = {30(2)},
	year = {2005}}

@inproceedings{DBLP:conf/birthday/SuilenBB0025,
	author = {Suilen, Marnix and Badings, Thom and Bovy, Eline M. and Parker, David and Jansen, Nils},
	booktitle = {Principles of Verification {(3)}},
	pages = {126--154},
	publisher = {Springer},
	series = {Lecture Notes in Computer Science},
	title = {Robust Markov Decision Processes: {A} Place Where {AI} and Formal Methods Meet},
	volume = {15262},
	year = {2024}}

@article{DBLP:journals/sttt/BadingsSSJ23,
	author = {Badings, Thom and Sim{\~{a}}o, Thiago D. and Suilen, Marnix and Jansen, Nils},
	journal = {Int. J. Softw. Tools Technol. Transf.},
	number = {3},
	pages = {375--391},
	title = {Decision-making under uncertainty: beyond probabilities},
	volume = {25},
	year = {2023}}

@software{artifact,
  author       = {Schmidt, Jule and
                  Weininger, Maximilian and
                  Dubslaff, Clemens and
                  Parker, David and
                  Jansen, Nils},
  title        = {Property-Driven Causal Abstractions for {M}arkov
                   Decision Processes ({FMCAD} 2026 Artifact)
                  },
  year         = 2026,
  publisher    = {Zenodo},
  doi          = {10.5281/zenodo.21825827},
  url          = {https://doi.org/10.5281/zenodo.21825827},
  note         = {https://doi.org/10.5281/zenodo.21825827}
}

\clearpage
\appendices
\crefalias{section}{appendix} 
\crefalias{subsection}{appendix} 
\crefalias{subsubsection}{appendix}
\ifarxivelse{\section{Proofs for \cref{sec:3-title}}
\label{app:causality-proofs}

\subsection{Proof of \cref{thm:cgLeqFC}}

Intuitively, \cref{thm:cgLeqFC} claims that for a given MDP $\mdp$ with $\Reach$, $\Effect$, and a cause-effect cover $C$, the cardinality of $\predpart(\mdp, C)$ is always smaller than or equal to the cardinality of $\cgpart(\mdp, I_{\text{irr}})$.

\begin{proof}
    The second claim follows from the first.

    Formally, the set of irrelevant features is
    \[
        I_{\text{irr}} = \{i \colon \nexists \featcause \in C \colon \boolvarval_{i, k} \subset \gamma, k \in [0..\abs{\statevar_i}]\mid i \in [1..n]\}
    \]

    \begin{enumerate}
        \item Let $i \in I_{\text{irr}}$ be the index of an irrelevant feature. 
        Then, by definition, $\boolvarval_{i,k}\not\subset \featcause, k \in [1..\abs{\statevar_i}, \featcause \in C]$.
        Otherwise, it would violate \cref{def:featcause}, \cref{def:fc2}.
        Let $\states_{CG} \in \cgpart(\mdp, I_{\text{irr}})$.
        By definition, $\forall \state, \state' \in \states_{CG}$, the states only differ in irrelevant features, i.e., $\forall i \in [1..n] \backslash I_{\text{irr}}: \statevarval'_i = \statevarval_i$.
        Hence, $\featcause(\state) = \featcause(\state')$ forall $\featcause \in C$, i.e., $\state, \state'$ are in the same $\states_P \in \predpart(\mdp, C)$.
        Therefore, the first claim holds.
        \item If for all sets $\states_{CG}$ there is a set $\states_P$ such that $\states_{CG} = \states_P$, then $\abs{\cgpart(\mdp, I_{\text{irr}})} = \abs{\predpart(\mdp, C)}$.
        If some $\states_P\in\predpart(\mdp,C)$ is a superset of multiple $\states_{CG}\in\cgpart(\mdp,I_\text{irr})$, then $\abs{\cgpart(\mdp, I_{\text{irr}})} > \abs{\predpart(\mdp, C)}$.
    \end{enumerate}

\end{proof}

\section{Proofs for \cref{sec:4-title}}
\label{app:abstraction-proofs}

\subsection{Proof of \cref{thm:4-abstraction-comparison} - \cref{thm:valueWA}}

We show that the \texttt{WA} abstraction can lead to arbitrarily wrong values using the counter-examples in \cref{fig:app-ce-wa}.
In the MDP $\mdp$, the maximum probability to reach $s_2$ is $0.6$, by first playing $a$ and then $b$.
Grouping together the states $\{s_0,s_1\}$, we obtain the abstract MDP $\absmdp_1$. 
There, the maximum probability to reach $s_2$ is $1/2$ (by playing $b$), strictly smaller than in the original MDP.
Grouping together the states $\{s_0,s_1,s_3\}$, we obtain the abstract MDP $\absmdp_2$. 
There, the maximum probability to eventually reach $s_2$ is 1 (by forever playing $b$), strictly larger than in the original MDP.

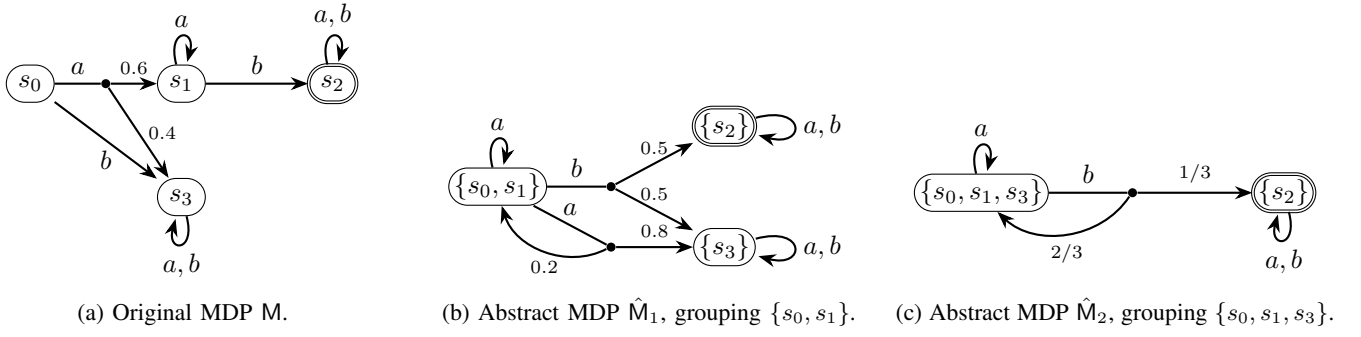
\begin{figure*}
\tikzset{
    >=Stealth,
    state/.style={rectangle, rounded corners=7pt, draw, minimum width=18pt, minimum height=14pt, inner sep=2pt},
    action/.style={circle, fill=black, inner sep=1.2pt},
    adv/.style={rectangle, draw, fill=black!10, minimum size=6pt, inner sep=0pt},
    every label/.style={font=\small},
    edge/.style={->, thick},
    line/.style={thick},
}
    \centering
\begin{subfigure}[b]{0.32\linewidth}
\centering
\begin{tikzpicture}
    \node[state] (s0) at (0,0)  {$s_0$};
    \node[state] (s1) at (2,0)  {$s_1$};
    \node[state,double] (s2) at (4,0)  {$s_2$};
    \node[state] (s3) at (2,-1.5) {$s_3$};

    \node[action] (s0a) at (1,0) {};
    \draw[line] (s0) -- (s0a) node[midway, above] {$a$};
    \path[edge] (s0a) edge node[pos=0.5, right,  font=\scriptsize] {$0.4$} (s3);
    \draw[edge] (s0a) -- (s1)  node[midway, above, font=\scriptsize] {$0.6$};

    \draw[edge] (s0) -- (s3) node[midway, below] {$b$};

    \draw[edge] (s1) -- (s2)  node[midway, above] {$b$};

    \path[edge] (s1) edge[loop above] node[above] {$a$}   (s1);
    \path[edge] (s2) edge[loop above] node[above] {$a,b$} (s2);
    \path[edge] (s3) edge[loop below] node[below] {$a,b$} (s3);
\end{tikzpicture}
\caption{Original MDP $\mdp$.}
\end{subfigure}
\hfill
\begin{subfigure}[b]{0.32\linewidth}
\centering
\begin{tikzpicture}
    \node[state] (s01) at (0,0)   {$\{s_0,s_1\}$};
    \node[state,double] (s2)  at (3,0.8) {$\{s_2\}$};
    \node[state] (s3)  at (3,-0.8){$\{s_3\}$};

    \node[action] (a_b) at (1.5, 0) {};
    \draw[line] (s01) -- (a_b) node[midway, above] {$b$};
    \draw[edge] (a_b) -- (s2)  node[midway, above, font=\scriptsize] {$0.5$};
    \draw[edge] (a_b) -- (s3)  node[midway, above, font=\scriptsize] {$0.5$};

    \node[action] (a_a) at (1.5, -0.8) {};
    \draw[line] (s01) -- (a_a) node[midway, above] {$a$};
    \draw[edge] (a_a) -- (s3)  node[midway, above, font=\scriptsize] {$0.8$};
    \path[edge] (a_a) edge[bend left=50]
        node[pos=0.5, below, font=\scriptsize] {$0.2$} (s01);

    \path[edge] (s2) edge[loop right] node[right] {$a,b$} (s2);
    \path[edge] (s3) edge[loop right] node[right] {$a,b$} (s3);
    \path[edge] (s01) edge[loop above] node[above] {$a$}   (s01);
\end{tikzpicture}
\caption{Abstract MDP $\absmdp_1$, grouping $\{s_0, s_1\}$.}
\end{subfigure}
\hfill
\begin{subfigure}[b]{0.32\linewidth}
\centering
\begin{tikzpicture}
    \node[state] (s014) at (0,0)    {$\{s_0,s_1,s_3\}$};
    \node[state,double] (s3)   at (4,0) {$\{s_2\}$};

    \path[edge] (s014) edge[loop above] node[above] {$a$}   (s014);

    \node[action] (a_ab) at (2,0) {};
    \draw[line] (s014) -- (a_ab) node[midway, above] {$b$};
    \draw[edge] (a_ab) -- (s3)   node[midway, above, font=\scriptsize] {$1/3$};
    \path[edge] (a_ab) edge[bend left=50]
        node[pos=0.5, below, font=\scriptsize] {$2/3$} (s014);

    \path[edge] (s3) edge[loop below] node[below] {$a,b$} (s3);
\end{tikzpicture}
\caption{Abstract MDP $\absmdp_2$, grouping $\{s_0, s_1, s_3\}$.}
\end{subfigure}
    \caption{MDPs and their WA abstractions to show \cref{thm:valueWA} of \cref{thm:4-abstraction-comparison}.}
    \label{fig:app-ce-wa}
\end{figure*}

\subsection{Proof of \cref{thm:4-abstraction-comparison} - \cref{thm:valueIMDP}}
\label{sec:appendix_imdp_proof}

\para{Notes on notation.}
For a given IMDP $\absimdp$, we write $\intervalpolicy$ to refer to the choice of probability distribution, or choice of consistent MDP, from the set of MDPs consistent with $\absimdp$.
Further, we write $\abstransition^{\intervalpolicy}$ for the probability distribution induced by an interval policy $\intervalpolicy$.
We call $\uncertaintyset(\absstate, a)$ the uncertainty set, i.e., the set of possible probability distributions given an abstract state $\absstate$ and action $a$.
The uncertainty set of the IMDP contains all the MDPs consistent with the IMDP.

For ease of notation, we introduce the following shorthands:
\begin{itemize}
    \item $\transition(q') = \transition(\state, a, q')$
    \item $\abstransition^{\intervalpolicy}(\absq) = \abstransition^{\intervalpolicy} (\absstate, a, \absq)$
    \item $\absitransitionlb(\absq_i) = \absitransitionlb(\absstate, a, \absq_i)$
    \item $\absitransitionub(\absq_i) = \absitransitionub(\absstate, a, \absq_i)$
\end{itemize}

\para{Proof outline.}
We need to show four inequalities. The proof is done as follows.
\begin{enumerate}
    \item We show that there exist interval policies $\intervalpolicy$ that induce probability distributions giving upper and lower bounds on the values of the original MDP in \cref{thm:exist_interval_policy}.
    \item Then, we show that these interval distributions belong to MDPs consistent with the IMDP in \cref{thm:interval_policy_valid}.
    \item Finally, we use these findings together with \cref{thm:induction-1} and \cref{thm:induction-2} to prove the inequalities of \cref{thm:valueIMDP} by induction.
\end{enumerate}

The proof is analogous for all four inequalities. We include the full proof for $\opt=\max, \wval(\absstate) \leq \val(\state)$ and point out the differences where applicable for the three remaining cases.

\para{Necessary lemmas and their proofs.}

\begin{lemma}
    \label{thm:exist_interval_policy}
    There exist interval policies $\intervalpolicy$ such that the induced probability distributions give valid upper and lower bounds on the values of the original MDP.
    In all cases, we show the bounds $\forall \absstate \in \absstates, \forall a \in \actions$ where $a = \policy(s) = \abspolicy(\absstate) \forall \state \in \states$, as described in \cref{sec:4-discussion}.

    \para{Lower bound.}
    For $\opt=\max, \wval(\absstate) \leq \val(\state)$ and for $\opt=\min, \bval(\absstate) \leq \val(\state)$.
    Let $f \colon \absstates \to [0,1]$. Then:
    \[
        \min_{\abstransition^{\intervalpolicy} \in \uncertaintyset(\absstate, a)} \sum_{\absq \in \absstates} \abstransition^{\intervalpolicy}(\absq) \cdot f(\absq) \leq \min_{\state \in \absstate} \sum_{\absq \in \absstates} \sum_{\q' \in \absq} \transition(q') \cdot f(\absq).
    \]
    I.e., the value of the probability distribution $\abstransition^{\intervalpolicy}$ with the smallest value in $\absimdp$ is smaller or equal to the value of the state $\state \in \absstate$ with smallest value in $\mdp$.

    \para{Upper bound.}
    For $\opt=\max, \bval(\absstate) \geq \val(\state)$ and for $\opt=\min, \wval(\absstate) \geq \val(\state)$.
    Let $f \colon \absstates \to [0, 1]$. Then:
    \[
        \max_{\abstransition^{\intervalpolicy} \in \uncertaintyset(\absstate, a)} \sum_{\absq \in \absstates} \abstransition^{\intervalpolicy}(\absq) \cdot f(\absq) \geq \max_{\state \in \absstate} \sum_{\absq \in \absstates} \sum_{\q' \in \absq} \transition(q') \cdot f(\absq).
    \]
    I.e., the value of the probability distribution $\abstransition^{\intervalpolicy}$ with the largest value in $\absimdp$ is larger or equal to the value of the state $\state \in \absstate$ with largest value in $\mdp$.
    
\end{lemma}

\begin{proof}
    Let $\absstate \in \absstates$ an abstract state in $\absimdp$. 
    Let $n = \abs{\absstates}$.
    We define an ordering of abstract successor states of $\absstate$.
    For \emph{lower bounds}, let $\absq_1, \ldots, \absq_n$ be an ordering over abstract successor states of $\absstate$ where $\forall i \in [1..n) \colon f(\absq_i) \leq f(\absq_{i+1})$.
    I.e., the successors are ordered by value in ascending order.
    Analogously for \emph{upper bounds}, let $\absq_1, \ldots \absq_n$ be an ordering over the abstract successors of $\absstate$ where $\forall i \in [1, n) \colon f(\absq_i) \geq f(\absq_{i+1})$, i.e., ordered by value in descending order.

    In both cases, $j$ is the index of the state in the ordering such that $\sum_{i=1}^j \absitransitionub(\absq_i) > 1 - \text{Base}$, with $\text{Base} = \sum_{\absq \in \absstates}\absitransitionlb(\absq)$.

    Then, $\abstransition^{\intervalpolicy}$ is a fixed probability distribution, such that for all ordered successor states $\absq_i$,

    \[
        \abstransition^{\intervalpolicy}(\absq_i) = 
        \begin{cases}
            \absitransitionub(\absq_i) & i < j \\
            \absitransitionlb ( \absq_i) & i > j \\
            1 - \sum_{i=1}^{j-1} \absitransitionub( \absq_i ) - \sum_{i=j+1}^{n} \absitransitionlb(\absq_i) & i = j
        \end{cases}.
    \]

    \cref{thm:interval_policy_valid} shows that $\abstransition^{\intervalpolicy}$ is a valid distribution.

    We now show the claim of \cref{thm:exist_interval_policy}.
    We include the full proof for the lower bound $\opt=\max, \wval(\absstate) \leq \val(\state)$, which is identical to $\opt=\min, \bval(\absstate) \leq \val(\state)$.
    For upper bounds, we replace every $\min_{\state \in \absstate}$ with $\max_{\state \in \absstate}$ and every $\leq$ with $\geq$.

    \begin{align*}
        &\sum_{\absq \in \absistates} \abstransition^{\intervalpolicy}(\absstate, a, \absq ) \cdot f(\absq) \\
        &= \sum_{i=1}^{j-1} \absitransitionub(\absq_i) \cdot f(\absq_i) + \sum_{i=j+1}^n \absitransitionlb(\absq_i) \cdot f(\absq_i) \\
        &+ (1 - \sum_{i=1}^{j-1} \absitransitionub(\absq_i) - \sum_{i=j+1}^n \absitransitionlb(\absq_i)) \cdot f(\absq_j) \\
        &= \sum_{i=1}^{j-1} (\min_{\state \in \absstate} \sum_{\q' \in \absq} \transition(\q') +  \absitransitionub(\absq_i) - \min_{\state \in \absstate} \sum_{\q' \in \absq} \transition(\q')) \cdot f(\absq_i)\\
        &+ \sum_{i=j+1}^n \absitransitionlb(\absq_i) \cdot f(\absq_i) 
        + (1 - \sum_{i=1}^{j-1} \absitransitionub(\absq_i) - \sum_{i=j+1}^n \absitransitionlb(\absq_i)) \cdot f(\absq_j) \\
        &\leq \sum_{i=1}^{j-1} \min_{\state \in \absstate} \sum_{\q' \in \absq} \transition(\q') \cdot f(\absq_i) +  \sum_{i=1}^{j-1}(\absitransitionub(\absq_i) \\
        &- \min_{\state \in \absstate} \sum_{\q' \in \absq} \transition(\q')) \cdot f(\absq_j) 
        + \sum_{i=j+1}^n \absitransitionlb(\absq_i) \cdot f(\absq_i) \\
        &+ (1 - \sum_{i=1}^{j-1} \absitransitionub(\absq_i) - \sum_{i=j+1}^n \absitransitionlb(\absq_i)) \cdot f(\absq_j) \tag{Note: $f(\absq_i) \leq f(\absq_j) \forall i < j$} \\
        \end{align*}
    \begin{align*}
        &= \sum_{i=1}^{j-1} \min_{\state \in \absstate} \sum_{\q' \in \absq_i} \transition(\q') \cdot f(\absq_i) + \sum_{i=j+1}^n \absitransitionlb(\absq_i) \cdot f(\absq_i) \\
        &+ (1 - \sum_{i=1}^{j-1} \absitransitionub(\absq_i) - \sum_{i=j+1}^n \absitransitionlb(\absq_i) + \sum_{i=1}^{j-1} \absitransitionub(\absq_i) \\
        &- \sum_{i=1}^{j-1} \min_{\state \in \absstate} \sum_{\q' \in \absq_i} \transition(\q')) \cdot f(\absq_j) \\
        &\leq \sum_{i=1}^{j-1} \min_{\state \in \absstate} \sum_{\q' \in \absq_i} \transition(\q') \cdot f(\absq_i) + \sum_{i=j+1}^n \absitransitionlb(\absq_ia) \cdot f(\absq_i) \\
        &+ (\min_{\state \in \absstate} \sum_{\q' \in \absq_j} \transition(\q') + \sum_{i=j+1}^n (\absitransitionlb ( \absq_i) - \min_{\state \in \absstate} \sum_{\q' \in \absq_i} \transition(\q'))) \cdot f(\absq_j) \tag{\cref{thm:lem1-rem1}} \\
        &\leq \sum_{i=1}^{j-1} \min_{\state \in \absstate} \sum_{\q' \in \absq_i} \transition(\q') \cdot f(\absq_i) + \sum_{i=j+1}^n \absitransitionlb(\absq_i) \cdot f(\absq_i) \\
        &+ \min_{\state \in \absstate} \sum_{\q' \in \absq_j} \transition(\q') \cdot f(\absq_j) \\
        &+ \sum_{i=j+1}^n (\min_{\state \in \absstate} \sum_{\q' \in \absq_i} \transition(\q') - \absitransitionlb(\absq_i)) \cdot f(\absq_i) \tag{\cref{thm:lem1-rem2}} \\
        &= \sum_{i=1}^j \min_{\state \in \absstate} \sum_{\q' \in \absq_i} \transition(\q') \cdot f(\absq_i) + \sum_{i=j+1}^n (\absitransitionlb(\absq_i) \\
        &+ \min_{\state \in \absstate} \sum_{\q' \in \absq_i} \transition(\q') - \absitransitionlb(\absq_i)) \cdot f(\absq_i) \\
        &= \sum_{i=1}^{j} \min_{\state \in \absstate} \sum_{\q' \in \absq_i} \transition(\q') \cdot f(\absq_i)  
        + \sum_{i=j+1}^{n} \min_{\state \in \absstate} \sum_{\q' \in \absq_i} \transition(\q') \cdot f(\absq_i) \\
        &= \sum_{\absq \in \absistates} \min_{\state \in \absstate} \sum_{\q' \in \absq} \transition(\q') \cdot f(\absq) \\
        &\leq \min_{\state \in \absstate} \sum_{\absq \in \absistates} \sum_{\q' \in \absq} \transition(\state, a, \q') \cdot f(\absq)
    \end{align*}
    Since we can show the inequality and since $\min_{\abstransition^{\intervalpolicy} \in \uncertaintyset} \leq \abstransition^{\intervalpolicy}$ because $\abstransition^{\intervalpolicy} \in \uncertaintyset$, the claim holds for lower bounds.
\end{proof}

\begin{remark}
    \label{thm:lem1-rem1}
    For lower bounds, we need to show that

    \begin{align*}
        &1- \sum_{i=1}^{j-1} \absitransitionub(\absq_i) - \sum_{i=j+1}^n \absitransitionlb(\absq_i) 
        \\&+ \sum_{i=1}^{j-1}(\absitransitionub(\absq_i) - \min_{\state \in \absstate} \sum_{\q' \in \absq_i} \transition(\q'))\\
        &\leq \min_{\state \in \absstate} \sum_{\q' \absq_j}\transition(\q') + \sum_{i=j+1}^n ( \absitransitionlb(\absq_i) - \min_{\state \in \states} \sum_{\q' \in \absq_i} \transition(\q'))
    \end{align*}

    This is done as follows:
    \begin{align*}
        &1- \sum_{i=1}^{j-1} \absitransitionub(\absq_i) - \sum_{i=j+1}^n \absitransitionlb(\absq_i) 
        \\&+ \sum_{i=1}^{j-1}(\absitransitionub(\absq_i) - \min_{\state \in \absstate} \sum_{\q' \in \absq_i} \transition(\q'))\\
        &\leq \min_{\state \in \absstate} \sum_{\q' \absq_j}\transition(\q') + \sum_{i=j+1}^n ( \absitransitionlb(\absq_i) - \min_{\state \in \states} \sum_{\q' \in \absq_i} \transition(\q')) \\
        &\iff 1 - \sum_{i=j+1}^n \absitransitionlb(\absq_i) - \sum_{i=1}^{j-1} \min_{\state \in \absstate_i} \sum_{\q' \in \absq_i} \transition(\q') 
        \\&\leq \min_{\state \in \absstate} \sum_{\q' \absq_j}\transition(\q') + \sum_{i=j+1}^n ( \absitransitionlb(\absq_i) - \min_{\state \in \states} \sum_{\q' \in \absq_i} \transition(\q')) \\
        &\iff 1 \leq \min_{\state \in \absstate} \sum_{\q' \absq_j}\transition(\q') + 2 \cdot \sum_{i=j+1}^n \absitransitionlb(\absq_i) 
        \\&- \sum_{i=j+1}^n \min_{\state \in \states} \sum_{\q' \in \absq_i} \transition(\q')) + \sum_{i=1}^{j-1} \min_{\state \in \absstate} \sum_{\q' \in \absq_i} \transition(\q') \\
        &\iff 1 + 2 \cdot \sum_{i=j+1}^n \min_{\state \in \states} \sum_{\q' \in \absq_i} \transition(\q'))
        \\&\leq 2 \cdot \sum_{i=j+1}^n \absitransitionlb(\absq_i) + \sum_{i=1}^{n} \min_{\state \in \absstate} \sum_{\q' \in \absq_i} \transition(\q') \\
        &\iff 1 + 2 \cdot \sum_{i=j+1}^n \min_{\state \in \states} \sum_{\q' \in \absq_i} \transition(\q'))
        \\&\leq 2 \cdot \sum_{i=j+1}^n \absitransitionlb(\absq_i) + 1 \\
    \end{align*}
    \begin{align*}
        &\iff \sum_{i=j+1}^n \min_{\state \in \states} \sum_{\q' \in \absq_i} \transition(\q'))
        \leq  \sum_{i=j+1}^n \absitransitionlb(\absq_i) \tag{\cref{def:imdp_abstraction}}\\
    \end{align*}
\end{remark}

\begin{remark}
    \label{thm:lem1-rem2}
    For lower bounds, we show that:
    \begin{align*}
        &\sum_{i=j+1}^n( \absitransitionlb(\absq_i) - \min_{\state \in \absstate} \sum_{\q' \in \absq_i} \transition(\q')) \cdot f(\absq_j)
        \\&\leq \sum_{i=j+1}^n( \min_{\state \in \absstate} \sum_{\q' \in \absq_i} \transition(\q') - \absitransitionlb(\absq_i)) \cdot f(\absq_i) \\
    \end{align*}

    That is done by solving the following inequality:

    \begin{align*}
        &\sum_{i=j+1}^n( \absitransitionlb(\absq_i) - \min_{\state \in \absstate} \sum_{\q' \in \absq_i} \transition(\q')) \cdot f(\absq_j)
        \\&\leq \sum_{i=j+1}^n( \min_{\state \in \absstate} \sum_{\q' \in \absq_i} \transition(\q') - \absitransitionlb(\absq_i)) \cdot f(\absq_i) \\
        &\iff \sum_{i=j+1}^n \absitransitionlb(\absq_i) \cdot f(\absq_j) - \sum_{i=j+1}^n \min_{\state \in \absstate} \sum_{\q' \in \absq_i} \transition(\q') \cdot f(\absq_j) \\
        &\leq \sum_{i=j+1}^n \min_{\state \in \absstate} \sum_{\q' \in \absq_i} \transition(\q') \cdot f(\absq_i) - \sum_{i=j+1}^n \absitransitionlb(\absq_i) \cdot f(\absq_i) \\
        &\iff \sum_{i=j+1}^n \absitransitionlb(\absq_i) \cdot (f(\absq_j) + f(\absq_i)) 
        \\&\leq \sum_{i=j+1}^n \min_{\state \in \absstate} \sum_{\q' \in \absq_i} \transition(\q') \cdot (f(\absq_i) + f(\absq_j) \\
        &\iff \sum_{i=j+1}^n \absitransitionlb(\absq_i) 
        \leq \sum_{i=j+1}^n \min_{\state \in \absstate} \sum_{\q' \in \absq_i} \transition(\q') \tag{\cref{def:imdp_abstraction}} \\
    \end{align*}
\end{remark}

\begin{lemma}
    \label{thm:interval_policy_valid}
    The interval distribution for \emph{lower bounds} fixed in \cref{thm:exist_interval_policy} belongs to an MDP consistent with $\absimdp$, i.e., $\abstransition^{\intervalpolicy} \in \uncertaintyset$, where $\intervalpolicy$ is the lower bound interval distribution.
    The same holds for \emph{upper bounds}, i.e., $\abstransition^{\intervalpolicy} \in \uncertaintyset$, where $\intervalpolicy$ is the upper bound interval distribution.
\end{lemma}

\begin{proof}
    The interval distribution is chosen identically for both upper and lower bounds, differing only in the direction of ordering over abstract successor states.
    It is therefore sufficient to show the validity of the interval distribution $\abstransition^{\intervalpolicy}$ once.
    We need to show that $\forall (\absstate, a)$, $\abstransition^{\intervalpolicy}$ (1) is a valid probability distribution, and (2) lies within the interval bounds for each successor state $\absq \in \absstates$.
    \begin{enumerate}
        \item $\forall \absstate, a \in \absstates \times \actions$, $\abstransition^{\intervalpolicy}$ is a valid probability distribution:
        \begin{align*}
            \sum_{i=1}^{j-1} \absitransitionub(\absq_i) + \sum_{i=j+1}^n \absitransitionlb(\absq_i) + 1-
            \\
            \sum_{i=1}^{j-1} \absitransitionub(\absq_i) - \sum_{i=j+1}^n \absitransitionlb(\absq_i) = 1
        \end{align*}
        \item $\forall \absstate, a \in \absstates \times \actions$, $\abstransition^{\intervalpolicy}$ is within the valid interval bounds.
        \begin{itemize}
            \item For $i<j$ and $i>j$, this follows from \cref{def:imdp_abstraction} which explicitly includes the interval bounds, and
            
            $\abstransition^{\intervalpolicy}(\absq_i) = \absitransitionub(\absq_i)$ 
            and 
            $\abstransition^{\intervalpolicy}(\absq_i) = \absitransitionlb(\absq_i)$.
            \item For $i=j$, we show that $\absitransitionlb(\absq_j) \leq \abstransition^{\intervalpolicy}(\absq_j) \leq \absitransitionub(\absq_j)$.

            For the \emph{upper bound}, since
            $\sum_{k=1}^j \absitransitionub(\absq_k) > 1 - \text{Base}$ but by construction, $\sum_{k=1}^{j-1} \absitransitionub(\absq_k) \leq 1 - \text{Base}$, 
            we have that $\absitransitionub(\absq_j) \geq \abstransition^{\intervalpolicy}(\absq_j)$.

            We prove the \emph{lower bound} by contradiction. 
            Assume $\absitransitionlb(\absq_j) > \abstransition^{\intervalpolicy}(\absq_j)$.
            Then,
            \begin{align*}
                &\absitransitionlb(\absq_j) > 1 - \sum_{i=1}^{j-1} \absitransitionub(\absq_i) - \sum_{i=j+1}^n \absitransitionlb(\absq_i) \\
                \implies &\absitransitionlb(\absq_j) + \absitransitionlb(\absq_j) + \sum_{i=1}^{j-1} \absitransitionub(\absq_i) \\
                &+ \sum_{i=j+1}^n \absitransitionlb(\absq_i) > j
            \end{align*}
            This cannot hold since previously showed that $\sum_{\absq \in \absstates} \abstransition^{\intervalpolicy}(\absq) = 1$.
            Hence, $\absitransitionlb(\absq_j) \leq \abstransition^{\intervalpolicy} (\absq_j)$.
        \end{itemize}
    \end{enumerate}
\end{proof}

\para{Overall proof of \cref{thm:4-abstraction-comparison} - \cref{thm:valueIMDP}.}
Using the results from the above steps, we prove the inequalities of \cref{thm:4-abstraction-comparison} - \cref{thm:valueIMDP} by induction.

\begin{proof}
    The proof is fully analogous in all four cases.
    Therefore, we provide the full proof for $\opt=\max,\wval(\absstate) \leq \val(\state)$, i.e., the \emph{worst case lower bound} and point out differences for the remaining cases.
    
    We show by induction that 
    \[
        \forall k \in \mathbb{N}_0. \wval^{\leq k}(\absstate) \leq \min_{\state \in \absstate} \val^{\leq k}(\state), \forall \absstate \in \absstates, \state \in \absstate
    \]

    \para{Base case.} 
    If $\absstate \in \abstargetset$, then $\state \in \targetset \forall \state \in \absstate$. 
    Therefore
    \[
        \forall \state \in \absstate, \wval^{\leq 0}(\absstate) = 1 = \val^{\leq 0}(\state)
    \]
    \para{Induction hypothesis.}
    \[
        \forall \absstate \in \absstates, \forall k \in \mathbb{N}_0 \colon \wval^{\leq k}(\absstate) \leq \min_{\state \in \absstate} \val^{\leq k}(\state)
    \]
    \para{Induction step.} $\forall \absstate \in \absstates\colon$
        \begin{align*}
        \wval^{\leq k+1}(\absstate) 
        &= \max_{a \in \actions} \min_{\abstransition^{\intervalpolicy} \in \uncertaintyset(\absstate, a)} \sum_{\absq \in \absistates} \abstransition^{\intervalpolicy}(\absq) \cdot \wval^{\leq k}(\absq) \tag{k-step VI \cite{chatterjee2008valueiteration}} \\
        &\leq \max_{a \in \actions} \sum_{\absq \in \absistates}(\min_{\state \in \absstate} \sum_{\q' \in \absq} \transition(\q')) \cdot \wval^{\leq k}(\absq) \tag{\cref{thm:exist_interval_policy}} \\
        &\leq \min_{\state \in \absstate} \max_{a \in \actions} \sum_{\absq \in \absistates} \sum_{\q' \in \absq} \transition(\q') \cdot \wval^{\leq k}(\absq) \\
        &\leq \min_{\state \in \absstate} \max_{a \in \actions} \sum_{\absq \in \absistates} \sum_{\q' \in \absq} \transition(\q') \cdot \min_{\q'' \in \absq} \val^{\leq k}(\q '') \tag{Induction hypothesis} \\
        &\leq \min_{\state \in \absstate} \max_{a \in \actions} \sum_{\absq \in \absistates} \sum_{\q' \in \absq} \transition(\q') \cdot \val^{\leq k}(\q') \tag{\cref{thm:induction-1}} \\
        &= \min_{\state \in \absstate} \max_{a \in \actions} \sum_{\q' \in \states} \transition(\q') \cdot \val^{\leq k}(\q') \tag{\cref{thm:induction-2}} \\
        &= \min_{\state \in \absstate} \val^{\leq k+1}(\state) \tag{k-step VI \cite{chatterjee2008valueiteration}} 
    \end{align*}

    For the \emph{best case lower bound} where $\opt=\min, \bval(\absstate) \leq \val(\state)$, we replace all occurrences of $\max_{a \in \actions}$ with $\min_{a \in \actions}$ and $\wval(\absstate)$ with $\bval(\absstate)$ in the induction. 

    For the remaining cases, the abstract target set definition changes to $\abstargetset = \{\ \absstate \mid \exists \state \in \targetset \land \absstate \}$.
    That changes the base case of induction as follows:
    If $\absstate \in \abstargetset$, then $\exists \state \in \absstate \colon \state \in \targetset$. 
    Therefore, $\exists \state \in \absstate \colon \wval^{\leq 0}(\absstate) = 1 = \val^{\leq 0}(\state)$.

    Then, for the \emph{worst case upper bound} where $\opt=\min, \wval(\absstate) \leq \val(\state)$, we replace all occurrences of $\min$ with $\max$ and conversely all occurrences of $\max$ with $\min$ and exchange all occurrences of $\leq$ with $\geq$.

    Finally, for the \emph{best case upper bound} where $\opt=\min, \bval(\absstate) \geq \val(\state)$, we replace all $\min_{\abstransition^{\intervalpolicy} \in \uncertaintyset}$ with $\max_{\abstransition^{\intervalpolicy} \in \uncertaintyset}$, and change all occurrences of $\leq$ with $\geq$.

    Thus, by convergence of value iteration \cite{chatterjee2008valueiteration}, $\forall \absstate \in \absstates \colon$

    $\opt=\max\colon$
    \[
        \wval(\absstate) = \lim_{k \to \infty} \wval^{\leq}(\absstate) \leq \lim_{k \to \infty} \min_{\state \in \states} \val^{\leq k}(\state) = \min_{\state \in \absstate} \val(\state)
    \]
    \[
        \bval(\absstate) = \lim_{k \to \infty} \bval^{\leq}(\absstate) \geq \lim_{k \to \infty} \max_{\state \in \states} \val^{\leq k}(\state) = \max_{\state \in \absstate} \val(\state)
    \]
    $\opt=\min\colon$
    \[
        \wval(\absstate) = \lim_{k \to \infty} \wval^{\leq}(\absstate) \leq \lim_{k \to \infty} \max_{\state \in \states} \val^{\leq k}(\state) = \max_{\state \in \absstate} \val(\state)
    \]
    \[
        \bval(\absstate) = \lim_{k \to \infty} \bval^{\leq}(\absstate) \leq \lim_{k \to \infty} \min_{\state \in \states} \val^{\leq k}(\state) = \min_{\state \in \absstate} \val(\state)
    \]
\end{proof}

\para{Remarks supporting the induction.}
The following remarks extend on the referenced induction steps in the proof of \cref{thm:4-abstraction-comparison} - \cref{thm:valueIMDP}.

\begin{remark}
    \label{thm:induction-1}
    Let $\opt \in \{\min, \max\}$.

    For the \emph{lower bound}, we have
    \begin{align*}
        \min_{\state \in \absstate} \opt_{a \in \actions} \sum_{\absq \in \absistates} \sum_{\q' \in \absq} \transition(\q') \cdot \min_{\q'' \in \absq} \val^{\leq k}(\q '') 
        \\
        \leq \min_{\state \in \absstate} \opt_{a \in \actions} \sum_{\absq \in \absistates} \sum_{\q' \in \absq} \transition(\q') \cdot \val^{\leq k}(\q')
    \end{align*}

    since
    \[
        \min_{\q'' \in \absq}\val^{\leq k}(\q'') \leq V^{\leq k}(\q'), \forall \q' \in \absq.
    \]

    Analogously, for the \emph{upper bound}, we have
    \begin{align*}
        \max_{\state \in \absstate} \opt_{a \in \actions} \sum_{\absq \in \absistates} \sum_{\q' \in \absq} \transition(\q') \cdot \max_{\q'' \in \absq} \val^{\leq k}(\q '') 
        \\
        \geq \max_{\state \in \absstate} \opt_{a \in \actions} \sum_{\absq \in \absistates} \sum_{\q' \in \absq} \transition(\q') \cdot \val^{\leq k}(\q')
    \end{align*}
    
    since
    \[
        \max_{\q'' \in \absq}\val^{\leq k}(\q'') \geq V^{\leq k}(\q'), \forall \q' \in \absq.
    \]
\end{remark}

\begin{remark}
    \label{thm:induction-2}
    Since $\bigcup_{\absq \in \absstates} q' \in \absq = \states$ by definition, the following holds for all four combinations of $\opt \in \{\min, \max\}$ and $\opt' \in \{\min, \max\}$.
    \begin{align*}
        \opt_{\state \in \absstate} \opt'_{a \in \actions} \sum_{\absq \in \absstates} \sum_{q' \in \absq} \transition(q') \cdot \val^{\leq k}(q')\\
        =
        \opt_{\state \in \absstate} \opt'_{a \in \actions} \sum_{q' \in \states} \transition(q') \cdot \val^{\leq k}(q')
    \end{align*}
\end{remark}

 \subsection{Proof of \cref{thm:4-abstraction-comparison} - \cref{thm:valueSG}} 
\label{sec:appendix_sg_proof}

\begin{proof}
    Given an MDP $\mdp$ and a partition of its state space $\partition$.
    Let $\abssg$ be the abstract stochastic game obtained from $\mdp$ and $\partition$ by applying \cref{def:sg_abstraction} and let $\abssg'$ be the abstract stochastic game obtained from $\mdp$ and $\partition$ by applying \cite[Definition 10]{kattenbelt2010gamebasedabstraction}.

    Correctness of value bounds is proved in \cite[Definition 10]{kattenbelt2010gamebasedabstraction} for $\abssg'$.
    The definition of $\abssg$ is identical in terms of reachability, only adding action labels for the purpose of policy generation. 
    Transition probabilities, and therefore reachability values of the games, are computed identically.
    The notational differences translate as follows:
    \begin{itemize}
        \item Abstract states are elements of $\partition$, called player 1 vertices in \cite{kattenbelt2010gamebasedabstraction} and opponent states in \cref{def:sg_abstraction}.
        In these states, policies choose a concrete lifted distribution.
        \item Agent states (player 2 states in \cite{kattenbelt2010gamebasedabstraction}) contain lifted distributions labeled with actions ($\transition_\partition$ in \cref{def:sg_abstraction}) or rewards ($\bar{\text{Steps}}$ in \cite{kattenbelt2010gamebasedabstraction}).
        \item Lifted distributions are defined idenically ($\mathsf d$ in \cref{def:sg_abstraction}, $\bar \mu$ in \cite{kattenbelt2010gamebasedabstraction}).
    \end{itemize}

    It follows that the results of \cite[Theorem 1]{kattenbelt2010gamebasedabstraction} w.r.t reachability objectives also hold for $\sg$.
\end{proof}

\section{Additional Details on Experiments}
\label{app:exp_details}

\subsection{Additional Evaluation of Experiments}
\label{app:experiment-plots}

\subsubsection*{(RQ1) How does the property used to generate the effect set affect performance?}

This section includes further evaluation details on the $\Effect$ configurations for the causal partitions.

\para{\texttt{C-1} partition.}
We evaluate different thresholds for the \texttt{C-1} partition, by counting the number of best results per threshold compared to other threshold, and use fuzzy values to count how often thresholds actually perform well.
\cref{tab:oneshot-thresholds} shows that percentile-based thresholds often lead to partitions of smaller cardinality, but perform worse than relative thresholds in terms of quality.
While the relative $\tau=0.4$ threshold often achieves the smallest interval bounds, for the fuzzy values, the relative  
$\tau=0.1$ threshold performs slightly better.
Further, we do not find observe significant differences in terms of threshold performance under different combinations of $\Reach$-$\Effect$ definitions, which we visualize in \cref{fig:oneshot-thresholds-ve}.
Overall, the differences between threshold performances are relatively small. 
We conclude that a relative threshold $\tau=0.1$ is the best threshold and use it in all further experiments.

\para{\texttt{C-IT} partition.}
We analyse the same threshold settings as previously and additionally compare this to thresholds derived from fixing a number of iterations beforehand.
As we observe in \cref{fig:iterative-hyperparameters}, the results for all three metrics are very similar for different thresholds and when fixing either 6 or 10 iterations.
We do, however, observe clear differences between setting thresholds and fixing the iterations in \cref{fig:iterative-best-parameters}.
Partitions from a fixed number of iterations perform better than or equal to the threshold-based iterations in terms of policy difference and tightness of value bounds.
This occasionally comes at the price of larger partition cardinality. 
For many benchmarks, however, the size is similar or equal between the two types of iterations.
Threshold-based iterations generally iterate fewer times, leading to results that are similar to the \texttt{C-1} partition.

\para{Semantically interpreting $\Effect$ sets.}

For both partitions, we observe slightly better performance on \enquote{good} compared to \enquote{bad} $\Effect$ sets.
A reason for that could be that what we interpret to be \enquote{good} effects is what the MDP benchmarks were designed for.

\begin{figure}[h]
    \centering
    \includegraphics[width=\linewidth]{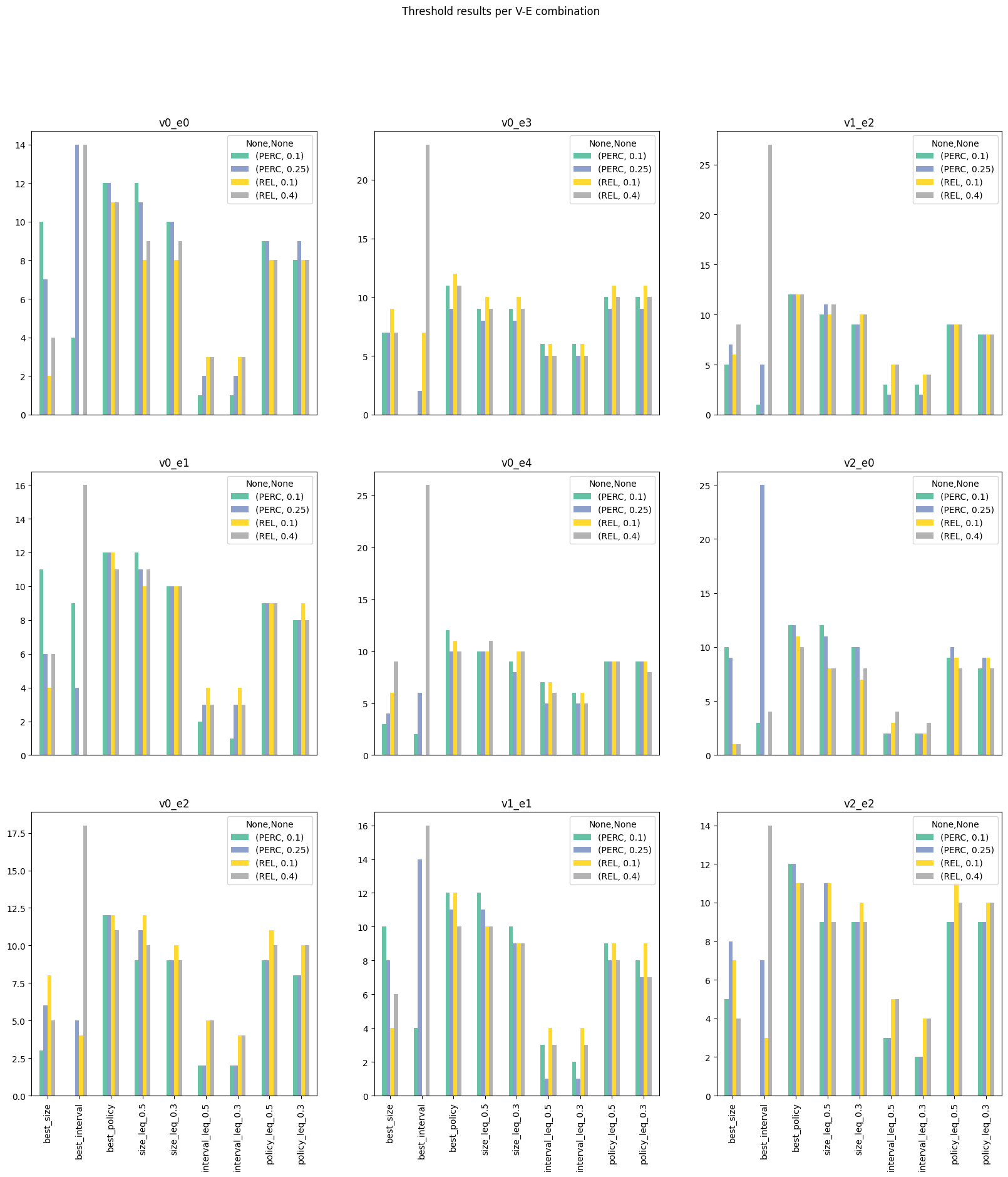}
    \caption{Performance of different thresholds for the \texttt{C-1} partition for different $\Effect$-$\Reach$ configurations.}
    \label{fig:oneshot-thresholds-ve}
\end{figure}

\begin{table}[]
    \centering
    \begin{tabular}{l|c|c|c|c}
         & \multicolumn{2}{c|}{percentile} & \multicolumn{2}{c}{relative} \\
                                    & 0.1 & 0.25 & 0.1 & 0.4 \\
        \hline \hline
        best size                   & \textbf{64} & 62 & 47 & 51 \\
        best interval size          & 23 & 82 & 14 & \textbf{158} \\
        best policy difference      & \textbf{98} & 93 & 95 & 88 \\
        \hline
        size $< 0.5$                & \textbf{95} & \textbf{95} & 89 & 88 \\
        size $< 0.3$                & \textbf{85} & 82 & 84 & 83 \\
        interval size $<0.5$        & 29 & 25 & \textbf{42} & 39 \\
        interval size $<0.3$        & 25 & 24 & \textbf{37} & 34 \\
        policy difference $< 0.5$   & 34 & 35 & \textbf{43} & 39 \\
        policy difference $< 0.3$   & 27 & 27 & \textbf{35} & 31 
    \end{tabular}
    \caption{Performance of thresholds for the \texttt{C-1} partition.}
    \label{tab:oneshot-thresholds}
\end{table}

\begin{figure}[h]
    \centering
    \includegraphics[width=\linewidth]{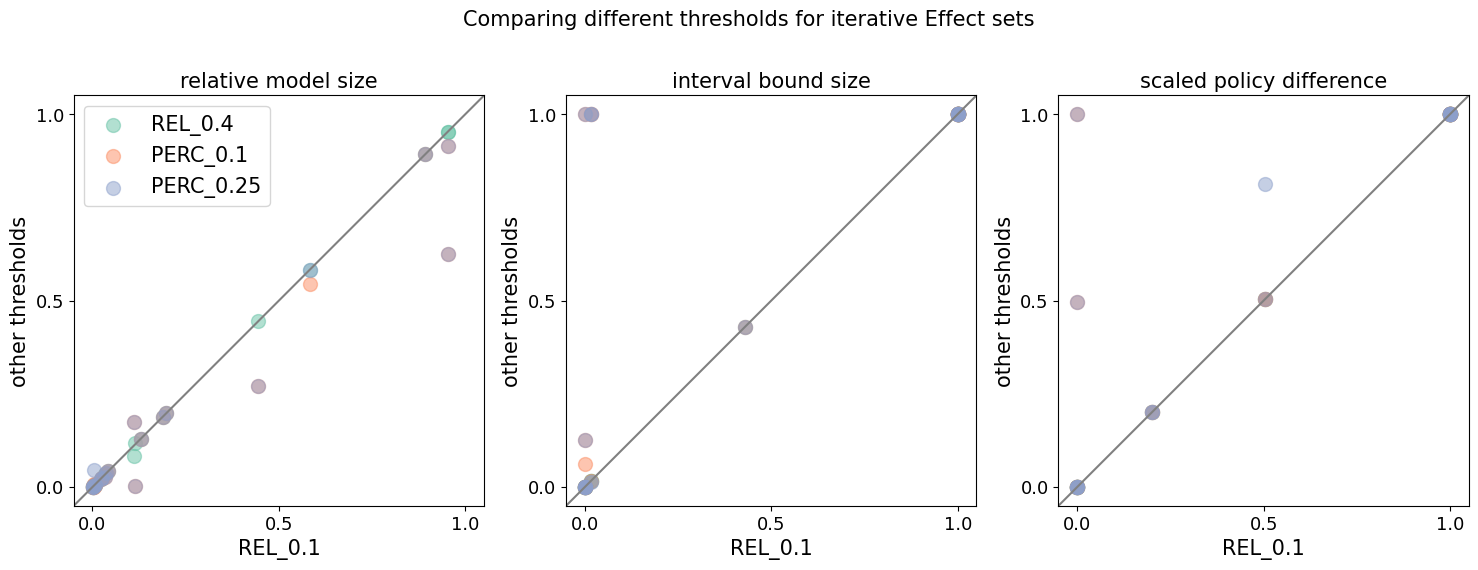} \\
    \includegraphics[width=\linewidth]{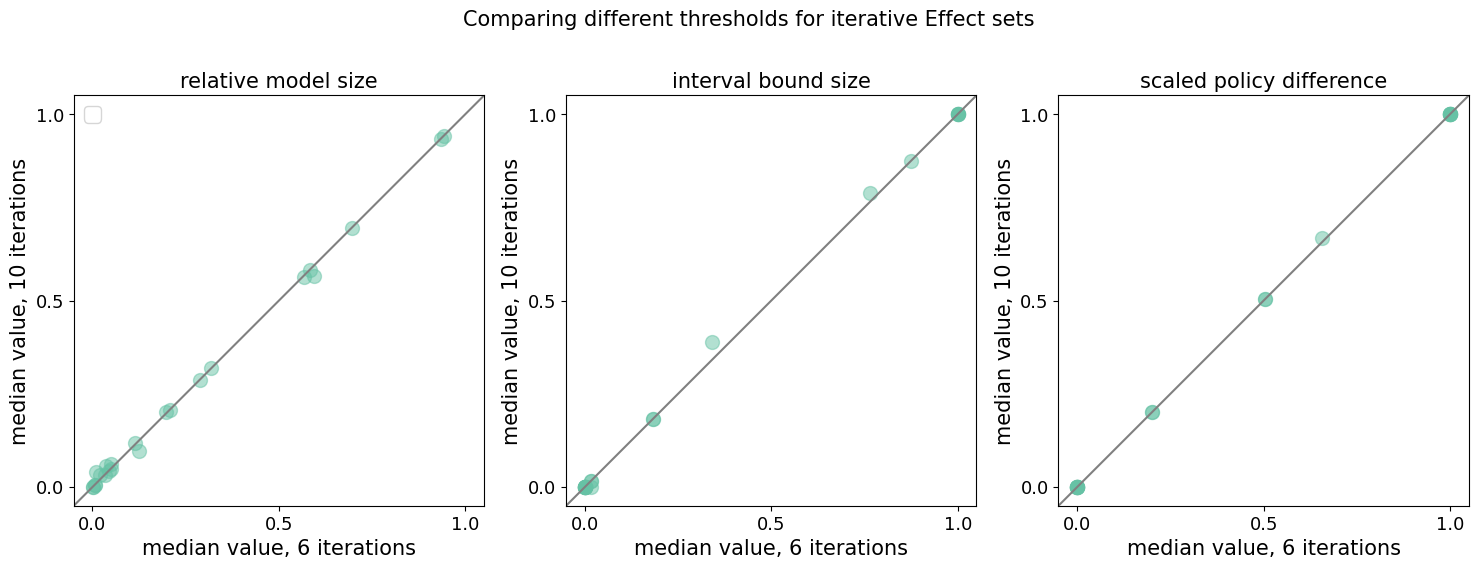}
    \caption{Performance of different iteration types for \texttt{C-IT}.}
    \label{fig:iterative-hyperparameters}
\end{figure}

\begin{figure}[h]
    \centering
    \includegraphics[width=\linewidth]{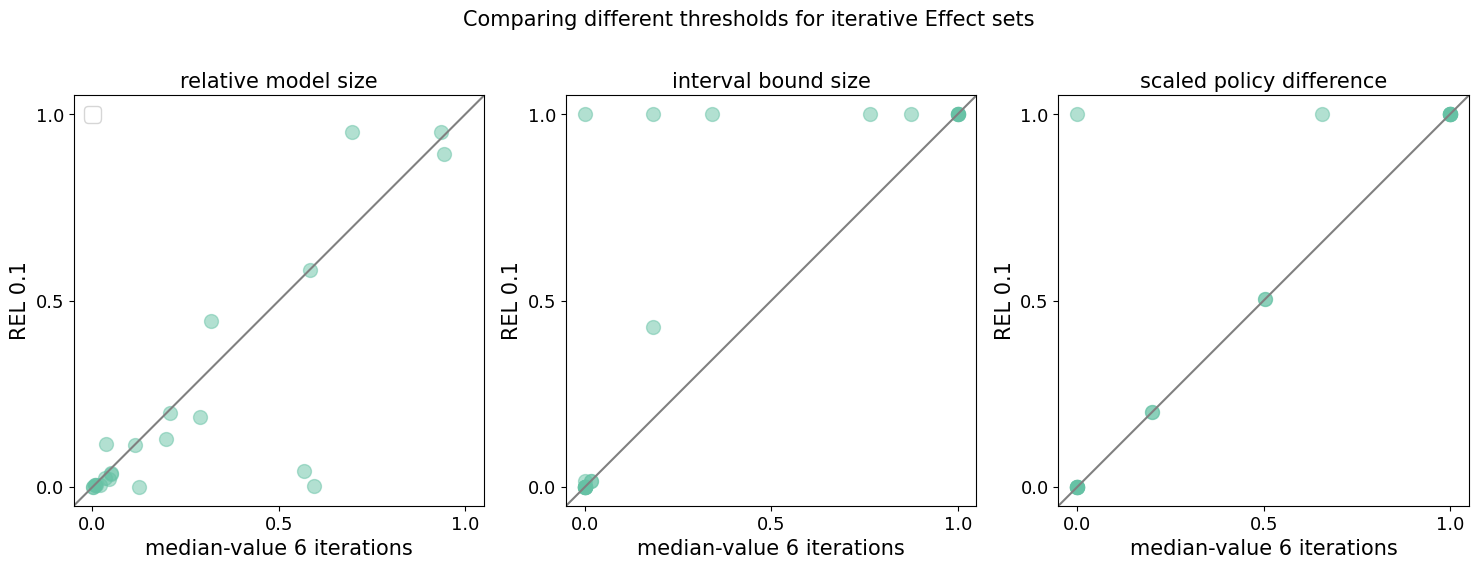}
    \caption{Comparing the best performing threshold-based iteration with the best fixed number of iterations.}
    \label{fig:iterative-best-parameters}
\end{figure}

\begin{figure}[h]
    \centering
    \includegraphics[width=\linewidth]{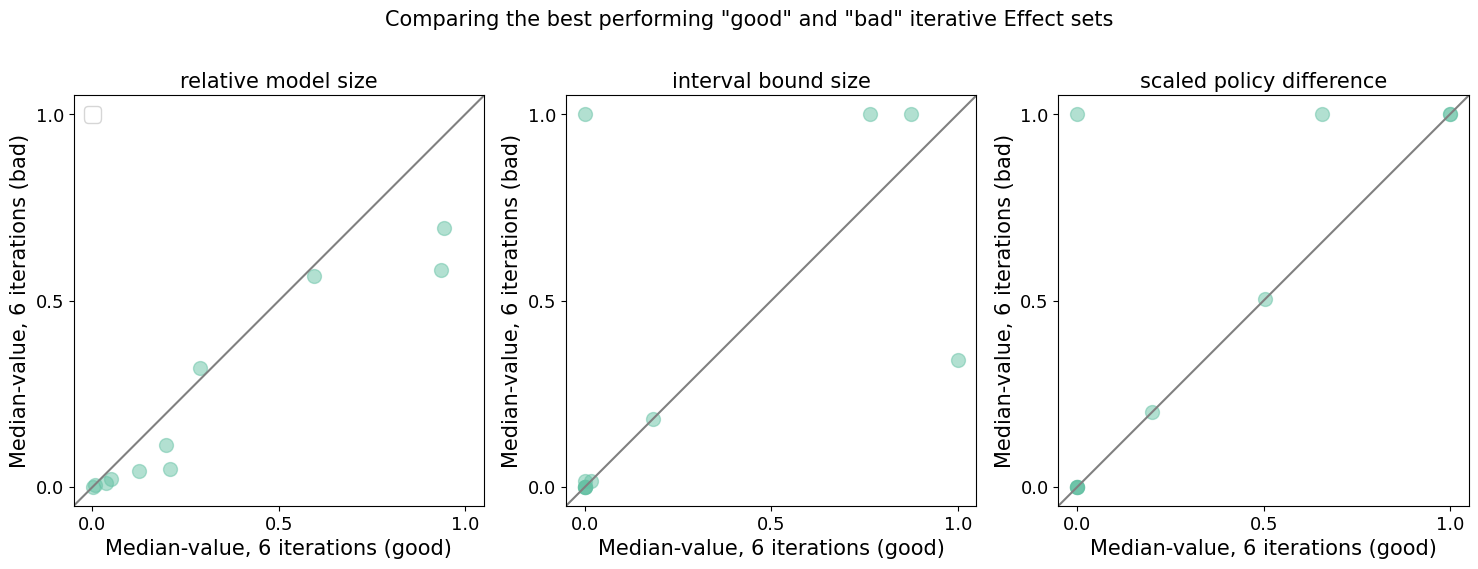} \\
    \includegraphics[width=\linewidth]{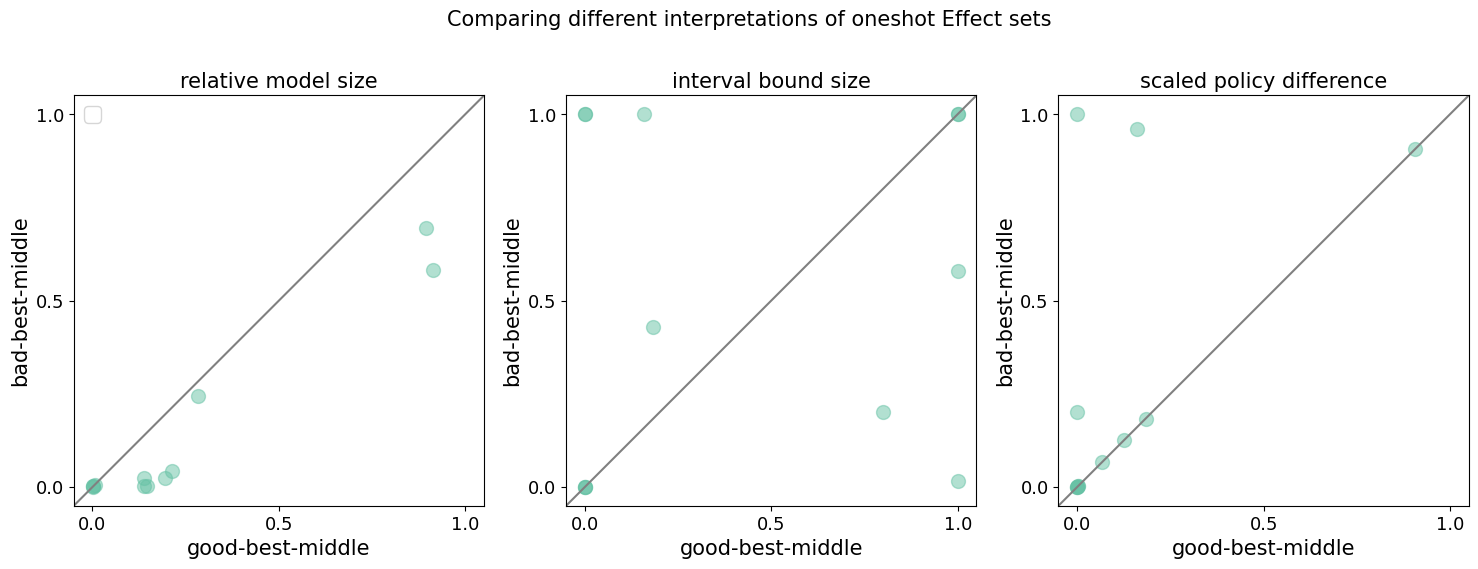}
    \caption{Good vs bad effect sets for \texttt{C-IT} and \texttt{C-1}.}
    \label{fig:hyperparameters-good-bad}
\end{figure}

\subsubsection*{(RQ2) How do different causal partitions compare?}

In 38\% of our benchmarks, the \texttt{C-G} partition does not exclude any features, leaving the original model untouched.
For fair comparison (since policy performance and value bounds receive a perfect score of 0, since the original model is exactly reproduced), we excluded those benchmarks from the analysis in \cref{sec:5-title}; for completeness, we include them in \cref{fig:partition-performance-all}.

\begin{figure*}[h]
    \centering
    \includegraphics[width=.27\linewidth]{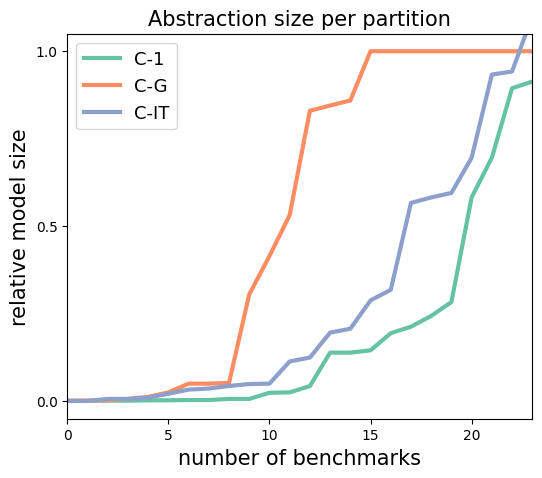} 
    \includegraphics[width=.27\linewidth]{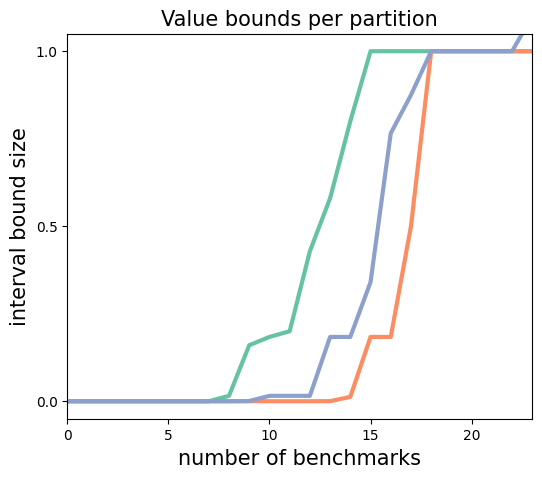} 
    \includegraphics[width=.27\linewidth]{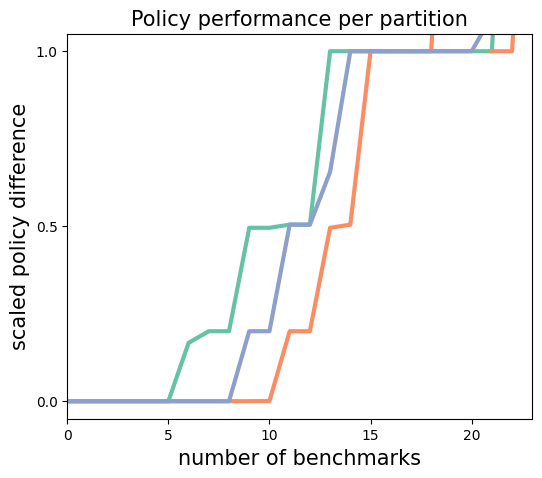}
    \caption{Comparison of partitions, including benchmarks where the causal graph retains the original model.}
    \label{fig:partition-performance-all}
\end{figure*}

\subsubsection*{(RQ3) How do different aggregation methods compare?}

\para{Extending: The SG aggregation works best, then IMDP, then WA.}
\cref{fig:5-abstractions-quantile} shows that the SG aggregation receives small intervals on more benchmarks than the IMDP one.
Whereas the IMDP models receive interval bounds of size 1 on approximatively half the benchmarks, the SG can improve bound tightness in ca. two thirds of our experiments.
In terms of policy performance, all three types of models perform well on most benchmarks, with slightly better results for the SG aggregation method.
Interestingly, the WA abstraction, which does not provide guarantees on the performance of the abstract model, also receives near perfect policies for almost half the experiments, and reasonable performance in ca. 35 out of 44 benchmarks.
It can however, also be arbitrarily wrong.
While not interesting in settings where guarantees are required, it can still be an interesting approach for example in reinforcement learning applications due to its simplicity.

\begin{figure}[h]
    \centering
    \includegraphics[width=.85\linewidth]{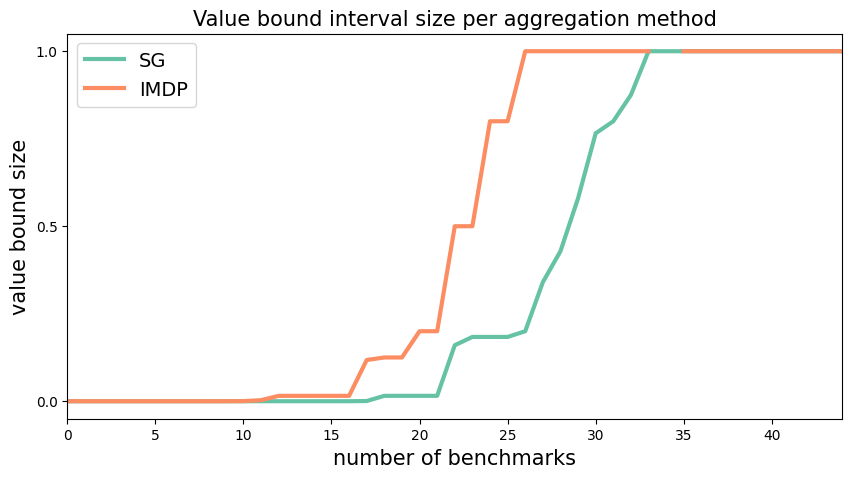} \\
    \includegraphics[width=.85\linewidth]{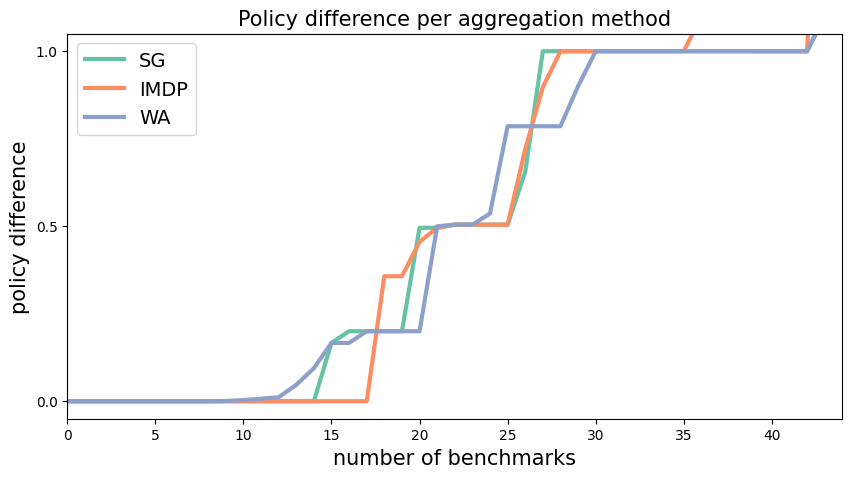}
    \caption{Comparison of the performance of different abstraction methods across benchmarks.}
    \label{fig:5-abstractions-quantile}
\end{figure}

\para{Comparing performance of different partitions on \enquote{good} vs. \enquote{bad} $\Effect$ sets.}
Here, we find differences, especially with respect to policy quality.
Across all abstraction models, \enquote{good} $\Effect$ sets usually result in larger abstract models.
While the tightness of value bounds shows mixed results, there is a visible tendency towards smaller bounds on \enquote{good} effects.
The distinction is clearer for policy performance, where the performance is either the same or better for \enquote{good} effects, with occasional outliers.
We show this in \cref{fig:5-abstractions-good-bad}.

\begin{figure}[h]
    \centering
    \includegraphics[width=\linewidth]{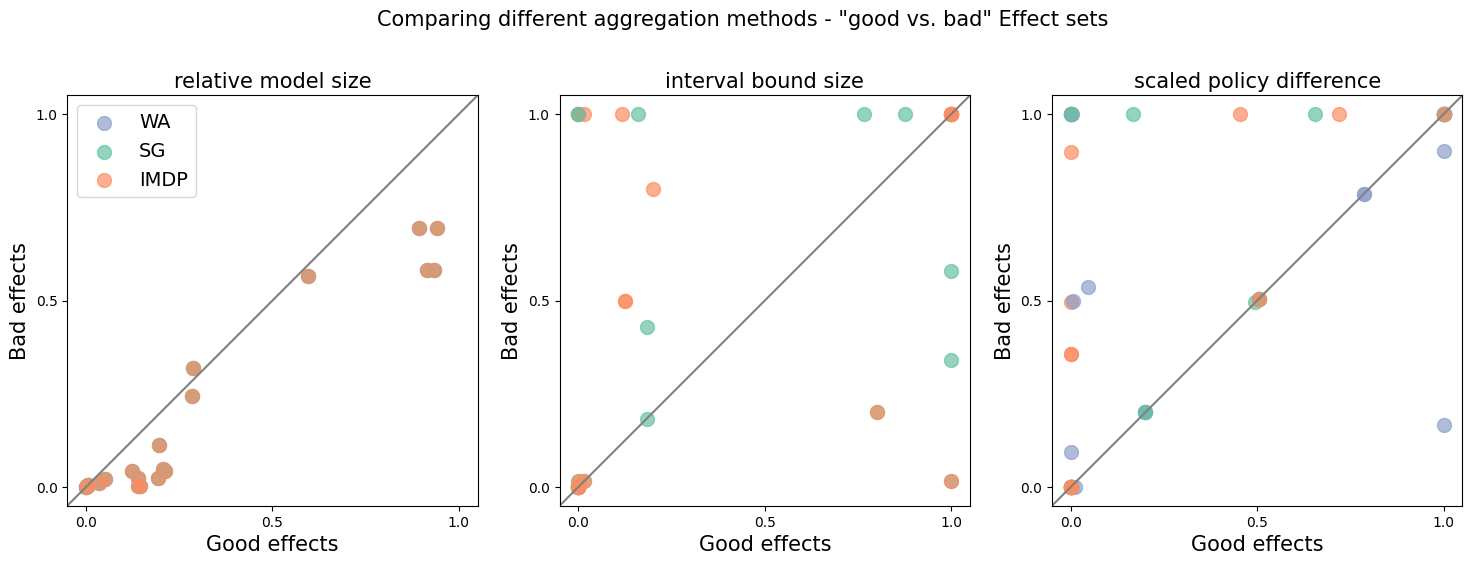}
    \caption{Comparison of the performance of different abstraction methods on \enquote{good} vs. \enquote{bad} $\Effect$ sets}
    \label{fig:5-abstractions-good-bad}
\end{figure}

\subsubsection*{(RQ4) Do the causal partitions generalize to larger model sizes?}

\begin{figure}[h]
    \centering
    \includegraphics[width=\linewidth]{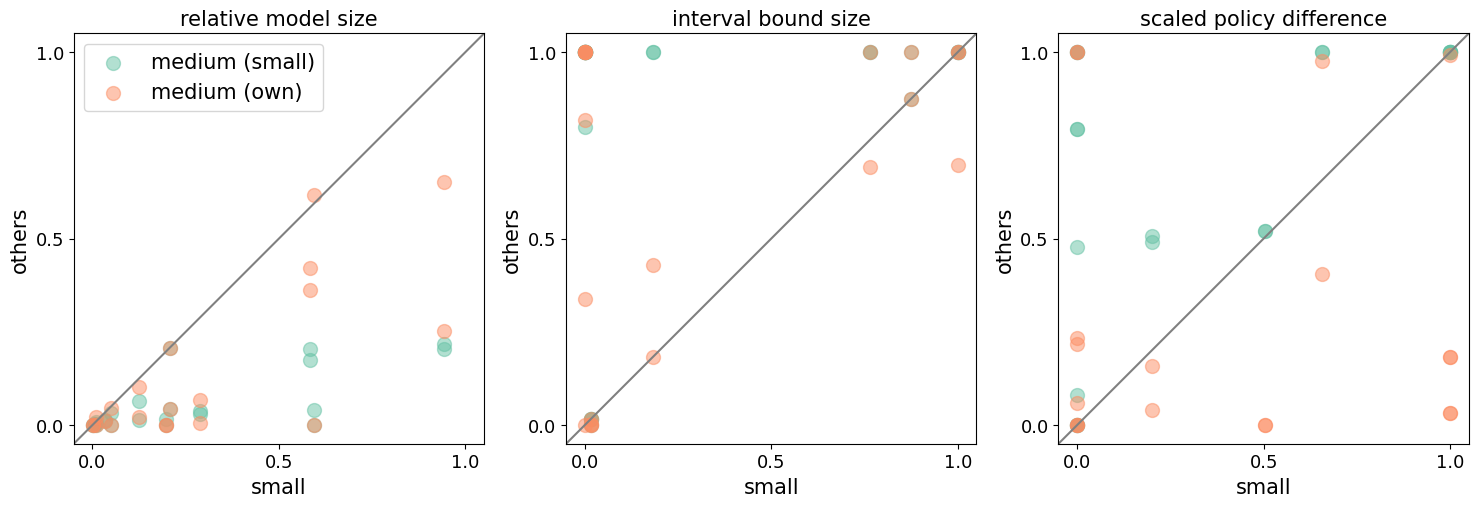}
    \caption{Generalization of causal abstractions. Comparing the performance of small abstractions with medium abstractions based on small or medium sized causes.}
    \label{fig:generalization-small-medium-causes}
\end{figure}

\begin{figure}[h]
    \centering
    \includegraphics[width=\linewidth]{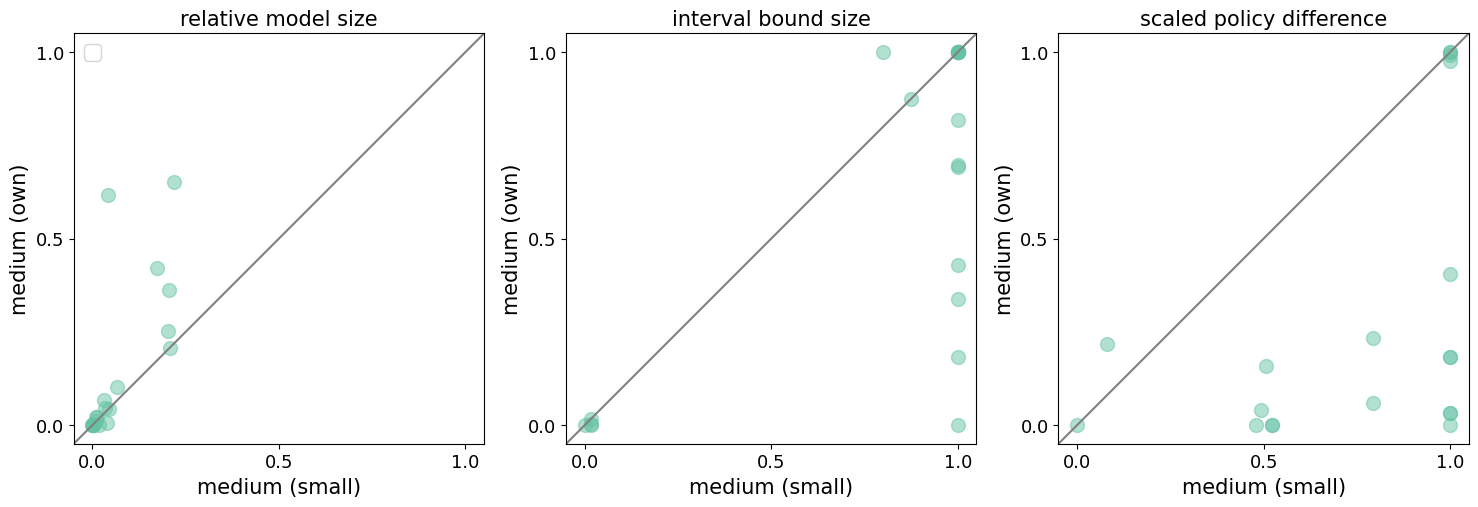}
    \caption{Generalization of causal abstractions. Comparing the performance of medium abstractions based on small or medium sized causes.}
    \label{fig:generalization-medium-models}
\end{figure}

\cref{fig:generalization-small-medium-causes} compares the performance of small abstractions from small causes with medium abstractions from small and medium sized causes.
While medium abstractions from small causes lead to relatively smaller abstractions, those from medium causes still reduce the relative size of states more than small abstractions. 
The interval bound tightness is often worse on medium abstractions than on small abstractions.
For policy quality, the performance of medium abstractions from medium causes is usually better than for small abstractions and also compared to medium abstractions from small models.
\cref{fig:generalization-medium-models} directly compares the performance between medium abstractions from small and medium causes. 
While those abstractions from medium causes are often relatively larger than those from small causes, they often receive much tighter interval bounds and better policy quality.

\subsection{Full overview of benchmarks}

We include a list of benchmarks in \cref{tab:app_benchmarks}.

\subsection{Feature causality time-outs}
\label{app:timeout}

We introduced timeouts for feature cause calculations.
The initial timeout is 4h and then 12h.
We report the settings that timed out after 4h but completed within 12h, and those that timed out even after 12h in \cref{tab:app-fc-timeouts}.

\subsection{Metrics}\label{app:exp-metrics}

For all experiments we report three performance metrics, formally defined below.
We write $\state_0$ for the initial state given by the property (since all properties in our evaluation specify a unique initial state), and $\absstate_0$ is the corresponding abstract state.
Further, we write $\reachprob^{\max}$ for the maximum reachability probability, i.e., $\max_{\policy\in\policies} \reachprob^{\policy}_{\state_0, \mdp}[\lozenge \targetset]$, and dually $\reachprob^{\min}$.
For policy performance, by normalizing the difference with $\reachprob^{\max} - \reachprob^{\min}$, we achieve a relative measure of performance: If, e.g., the optimal value in the original MDP is close to 0, even a policy achieving 0 has a small absolute difference.
Consequently, a policy difference of 0 indicates optimal performance ($\val^{\policy}(\state_0) = \val^{\abspolicy}(\state_0)$), while  a policy difference of 1 indicates $\abspolicy$ performs as bad as the worst possible policy, i.e., $\val^{\abspolicy}(\state_0) = \nopt_{\policy\in\policies} \reachprob^{\policy}_{\state, \mdp}[\lozenge \targetset]$.

\begin{align*}
    \mathrm{relative\_size} &= \abs{\absstates} / \abs{\states} \\
    \mathrm{value\_bound} &= \abs{\wval(\absstate_0) - \bval(\absstate_0)} \\
    \mathrm{policy\_difference} &= \abs{\val^{\policy}(\state_0) - \val^{\abspolicy}(\state_0)} / (\reachprob^{\max} - \reachprob^{\min})
\end{align*}

\FloatBarrier
\begin{table*}[]
    \centering
    \resizebox{\linewidth}{!}{
    \begin{tabular}{l|p{4cm}|l|c|c|l}
        Name & Parameters & Property & States & Transitions & Type of model \\
        \hline
        \hline
        avoid-mdp & N=10,slippery=0.2 & Pmax=? ["notbad" U "goal"] & 46781 & 415105 & grid world\\
            & N=12,slippery=0.2 && 84673 & 368825 &\\
            & N=15,slippery=0.2 && 272701 & 1350803 &\\
        \hline
        consensus & N=4,K=0 & Pmin=? [F "finished"\&"all\_coins\_equal\_1"] & 2504 & 8556 & communication protocol \\
            & N=4,K=2 && 22656 & 75232 &\\
            & N=4,K=10 && 43136 & 144352 & \\
        \hline
        consensus & N=4,K=0 & Pmax=? [F "finished"\&!"agree"] & 2504 & 8556 & communication protocol \\
            & N=4,K=2 && 22656 & 75232 &\\
            & N=4,K=10 && 43136 & 144352 & \\
        \hline
        csma & N=2,K=2 & Pmax=? [!"collision\_max\_backoff" U "all\_delivered"] & 1038 & 1282 & network protocol \\
            & N=2,K=3 && 36850 & 55862 &\\
            & N=2,K=4 && 761962 & 1327068 & \\
        \hline
        evade-mdp & N=4,slippery=0.7 & Pmax=? ["notbad" U "goal"] & 4463 & 41567 & grid world \\
            & N=6,slippery=0.7 && 63899 & 698635 &\\
            & N=8,slippery=0.7 && 398271 & 4601231 &\\
        \hline
        firewire\_dl & delay=3,deadline=100 & Pmin=? [F "done"] & 3662 & 4898 & internet protocol\\
            & delay=3,deadline=200 && 14824 & 17607 &\\
            & delay=3,deadline=500 && 113869 & 133522 &\\
        \hline
        mer & n=1,x=0.1 & Pmax=? [F "err\_G"] & 4356 & 17201 & resource scheduler \\
            & n=10,x=0.1 && 24282 & 95069 &\\
            & n=100,x=0.1 && 223542 & 873749 &\\
        \hline
        random-grid & N=20,pover=0.3,pstay=0.3 & Pmax=? ["first\_condition" U "second\_condition"] & 7980 & 124363 & grid world \\
            & N=40,pover=0.3,pstay=0.3 && 63960 & 1071943 &\\
            & N=80,pover=0.3,pstay=0.3 && 5111920 & 8892703 &\\
        \hline
        refuel-mdp & N=10,ENERGY=40 & Pmax=? [F "goal"] & 3856 & 24769 & grid world\\
            & N=40,ENERGY=40 && 52506 & 394218 &\\
            & N=80,ENERGY=40 && 161070 & 619178 &\\
        \hline
        sensor & N=2 & Pmax=? [F "condition"] & 7222 & 23786 & network protocol \\
            & N=3 && 66051 & 263526 &\\
            & N=4 && 462056 & 2156716 &\\
        \hline
        zeroconf & N=1000,K=1,reset=false & Pmax=? [F "correct"] & 31402 & 70643 & internet protocol \\
            & N=1000,K=2,reset=false && 88858 & 203550 &\\
            & N=1000,K=4,reset=false && 306585 & 703189 &\\
        \hline
        zeroconf & N=1000,K=1,reset=false & Pmin=? [F "correct"] & 31402 & 70643 & internet protocol \\
            & N=1000,K=2,reset=false && 88858 & 203550 &\\
            & N=1000,K=4,reset=false && 306585 & 703189 &\\
        \hline
    \end{tabular}
    }
    \caption{Full overview of benchmarks. \\
    Note that configurations of avoid-mdp with $N<10$ do not achieve Pmax$>$0, hence we cannot use them. Therefore, the smallest avoid-mdp configuration has more than $>10^4$ states.
    Zeroconf with reset=false does not allow for parameter settings that result in $<10^5$ states. Therefore we also use a larger model as the smallest here.
    }
    \label{tab:app_benchmarks}
\end{table*}

\begin{table*}[]
    \centering
    \begin{tabular}{l|c|c|c|c|c|c|c|c|c|c|c|c|c|c}
        & \rotatebox{90}{avoid-mdp (s)} & \rotatebox{90}{avoid-mdp (m)} & \rotatebox{90}{consensus-Pmax (m)} & \rotatebox{90}{consensus-Pmin (s)} & \rotatebox{90}{consensus-Pmin (m)} & \rotatebox{90}{evade-mdp (m)} & \rotatebox{90}{firewire\_dl (m)} & \rotatebox{90}{mer (s)} & \rotatebox{90}{mer (m)} & \rotatebox{90}{refuel-mdp (m)} & \rotatebox{90}{sensor (s)} & \rotatebox{90}{sensor (m)} & \rotatebox{90}{zeroconf-Pmax (m)} & \rotatebox{90}{zeroconf-Pmin (m)}\\
        \hline\hline
        C-1-PERC-0.1-v0e2 & & &\cellcolor{blue!25}& &\cellcolor{blue!25}&\cellcolor{green!25}& & &\cellcolor{blue!25}&\cellcolor{blue!25}& &\cellcolor{blue!25}&\cellcolor{blue!25}& \\
        C-1-PERC-0.1-v0e3 & & &\cellcolor{green!25}& & & &\cellcolor{blue!25}& &\cellcolor{blue!25}&\cellcolor{blue!25}& & \cellcolor{blue!25}& & \\
        C-1-PERC-0.1-v0e4 & & &\cellcolor{green!25}& & & & & & & & &\cellcolor{blue!25}&\cellcolor{blue!25}& \\
        C-1-PERC-0.1-v1e1 & &\cellcolor{green!25}&\cellcolor{green!25}& & & & & & & & & & & \\
        C-1-PERC-0.1-v1e2 & & &\cellcolor{green!25}& & & & & & & & &\cellcolor{blue!25}&\cellcolor{blue!25}& \\
        C-1-PERC-0.1-v2e2 & & & & & & & & & &\cellcolor{blue!25}& &\cellcolor{blue!25}& & \\
        C-1-PERC-0.25-v0e0 & & & & & &\cellcolor{green!25}& &\cellcolor{blue!25}& & & &\cellcolor{blue!25}& & \\
        C-1-PERC-0.25-v0e1 &\cellcolor{green!25}& & & & & &\cellcolor{blue!25}& & & & &\cellcolor{blue!25}& &\cellcolor{blue!25}\\
        C-1-PERC-0.25-v0e2 & & &\cellcolor{blue!25}&\cellcolor{green!25}&\cellcolor{green!25}&\cellcolor{blue!25}& & &\cellcolor{blue!25}&\cellcolor{blue!25}& & &\cellcolor{blue!25}&\cellcolor{blue!25}\\
        C-1-PERC-0.25-v0e3 & & & & & & &\cellcolor{blue!25}& &\cellcolor{blue!25}&\cellcolor{blue!25}&\cellcolor{blue!25}&\cellcolor{blue!25}& & \\
        C-1-PERC-0.25-v0e4 & & &\cellcolor{green!25}& & &\cellcolor{green!25}& & & &\cellcolor{blue!25}&\cellcolor{blue!25}&\cellcolor{blue!25}&\cellcolor{blue!25}& \\
        C-1-PERC-0.25-v1e1 &\cellcolor{blue!25}&\cellcolor{blue!25}& & & &\cellcolor{green!25}& & & &\cellcolor{blue!25}& & &\cellcolor{blue!25}& \\
        C-1-PERC-0.25-v1e2 & & &\cellcolor{green!25}& & &\cellcolor{green!25}& & & &\cellcolor{blue!25}& & &\cellcolor{blue!25}& \\
        C-1-PERC-0.25-v2e0 & & & & & & & & &\cellcolor{blue!25}& & &\cellcolor{blue!25}& &\cellcolor{blue!25}\\
        C-1-PERC-0.25-v2e2 & & & & & & & & &\cellcolor{blue!25}& & & & & \\
        C-1-REL-0.1-v0e0 & & & & & &\cellcolor{blue!25}&\cellcolor{blue!25}&\cellcolor{blue!25}&\cellcolor{blue!25}& & &\cellcolor{green!25}& & \\
        C-1-REL-0.1-v0e1 & & & & & & &\cellcolor{blue!25}& &\cellcolor{blue!25}& & & & & \\
        C-1-REL-0.1-v0e2 & & &\cellcolor{blue!25}& & & & & &\cellcolor{blue!25}&\cellcolor{blue!25}& &\cellcolor{blue!25}& & \\
        C-1-REL-0.1-v0e3 & & &\cellcolor{green!25}& & & &\cellcolor{blue!25}& &\cellcolor{blue!25}&\cellcolor{blue!25}& &\cellcolor{blue!25}& & \\
        C-1-REL-0.1-v0e4 & & & & & & &\cellcolor{blue!25}& & & & & & & \\
        C-1-REL-0.1-v1e1 & & & & & & & & &\cellcolor{blue!25}& & & & & \\
        C-1-REL-0.1-v2e0 & & & & & &\cellcolor{blue!25}& & &\cellcolor{blue!25}& & &\cellcolor{green!25}& & \\
        C-1-REL-0.1-v2e2 & & & & & & & & &\cellcolor{blue!25}&\cellcolor{blue!25}& &\cellcolor{blue!25}& & \\
        C-1-REL-0.4-v0e0 & & & & & & &\cellcolor{blue!25}&\cellcolor{blue!25}&\cellcolor{blue!25}& & & & & \\
        C-1-REL-0.4-v0e1 & & & & & & &\cellcolor{blue!25}& &\cellcolor{blue!25}& & & & & \\
        C-1-REL-0.4-v0e2 & & &\cellcolor{blue!25}& & & & & &\cellcolor{blue!25}&\cellcolor{blue!25}& &\cellcolor{blue!25}& & \\
        C-1-REL-0.4-v0e3 & & &\cellcolor{green!25}& & & &\cellcolor{blue!25}& &\cellcolor{blue!25}&\cellcolor{blue!25}& &\cellcolor{blue!25}& & \\
        C-1-REL-0.4-v0e4 & & & & & & &\cellcolor{blue!25}& & & & & & & \\
        C-1-REL-0.4-v1e1 &\cellcolor{blue!25}&\cellcolor{blue!25}& & & & & & &\cellcolor{blue!25}& & & & & \\
        C-1-REL-0.4-v2e0 &\cellcolor{blue!25}&\cellcolor{blue!25}& & & & & & &\cellcolor{blue!25}& & & & & \\
        C-1-REL-0.4-v2e2 & & &\cellcolor{green!25}& & & & & &\cellcolor{blue!25}&\cellcolor{blue!25}& &\cellcolor{blue!25}& & \\
        \hline
        C-IT-from0-iter-6 & & &\cellcolor{green!25}& & &\cellcolor{blue!25}&\cellcolor{blue!25}&\cellcolor{green!25}&\cellcolor{blue!25}&\cellcolor{blue!25}&\cellcolor{blue!25}&\cellcolor{blue!25}&\cellcolor{green!25}& \\
        C-IT-from0-iter-10 & & &\cellcolor{green!25}& & &\cellcolor{blue!25}&\cellcolor{blue!25}&\cellcolor{green!25}&\cellcolor{blue!25}&\cellcolor{blue!25}&\cellcolor{blue!25}&\cellcolor{blue!25}&\cellcolor{green!25}& \\
        C-IT-from1-iter-6 & &\cellcolor{green!25}& & & & &\cellcolor{blue!25}&\cellcolor{green!25}&\cellcolor{blue!25}&\cellcolor{blue!25}& &\cellcolor{blue!25}&\cellcolor{blue!25}&\cellcolor{blue!25}\\
        C-IT-from1-iter-10 & &\cellcolor{green!25}&\cellcolor{green!25}& & & &\cellcolor{blue!25}&\cellcolor{green!25}&\cellcolor{blue!25}&\cellcolor{blue!25}& &\cellcolor{blue!25}&\cellcolor{blue!25}&\cellcolor{blue!25}\\
        C-IT-from0-PERC-0.1 & & & & & & &\cellcolor{blue!25}&\cellcolor{green!25}&\cellcolor{green!25}&\cellcolor{green!25}&\cellcolor{green!25}&\cellcolor{green!25}&\cellcolor{green!25}&\cellcolor{green!25}\\
        C-IT-from0-PERC-0.25 & & & & & & &\cellcolor{blue!25}&\cellcolor{green!25}&\cellcolor{green!25}&\cellcolor{blue!25}&\cellcolor{blue!25}&\cellcolor{blue!25}&\cellcolor{blue!25}&\cellcolor{green!25}\\
        C-IT-from0-REL-0.1 & & & & & & &\cellcolor{blue!25}&\cellcolor{green!25}&\cellcolor{green!25}&\cellcolor{green!25}&\cellcolor{green!25}&\cellcolor{green!25}&\cellcolor{green!25}&\cellcolor{green!25}\\
        C-IT-from0-REL-0.4 & & & & & & &\cellcolor{blue!25}&\cellcolor{green!25}&\cellcolor{green!25}&\cellcolor{green!25}&\cellcolor{green!25}&\cellcolor{green!25}&\cellcolor{green!25}&\cellcolor{green!25}\\
        C-IT-from1-PERC-0.1 & & & & & & &\cellcolor{blue!25}&\cellcolor{green!25}&\cellcolor{green!25}&\cellcolor{green!25}& &\cellcolor{blue!25}&\cellcolor{blue!25}&\cellcolor{blue!25}\\
        C-IT-from1-PERC-0.25 & & & & & & &\cellcolor{blue!25}&\cellcolor{green!25}&\cellcolor{green!25}&\cellcolor{green!25}& &\cellcolor{blue!25}&\cellcolor{blue!25}&\cellcolor{blue!25}\\
        C-IT-from1-REL-0.1 & & & & & & &\cellcolor{blue!25}&\cellcolor{green!25}&\cellcolor{green!25}&\cellcolor{green!25}& &\cellcolor{blue!25}&\cellcolor{blue!25}&\cellcolor{green!25}\\
        C-IT-from1-REL-0.4 & & &\cellcolor{green!25}& & & &\cellcolor{blue!25}&\cellcolor{green!25}&\cellcolor{green!25}& \cellcolor{green!25}& &\cellcolor{blue!25}&\cellcolor{blue!25}&\cellcolor{green!25}\\
    \end{tabular}
    \caption{Timeouts when running feature causality. Cells colored green finished between 4-12 hours. Cells colored blue did not finish within 12 hours. Configurations and benchmarks without any timeouts are not included in the table.}
    \label{tab:app-fc-timeouts}
\end{table*}

\section{Artifact}
\label{app:artifact}

Code for this paper is publicly accessible on GitHub at \hyperref[https://github.com/ai-fm/causalabstractions.git]{https://github.com/ai-fm/causalabstractions.git}.
}{}

\end{document}